%% file: neurips_2026.tex
\documentclass{article}

\PassOptionsToPackage{numbers,compress}{natbib}

\usepackage[main,final]{neurips_2026}

\usepackage[utf8]{inputenc} 
\usepackage[T1]{fontenc}    
\usepackage{url}            
\usepackage{booktabs}       
\usepackage{graphicx}       
\usepackage{multirow}       
\usepackage{amsmath}        
\usepackage{amssymb}        
\usepackage{amsfonts}       
\usepackage{amsthm}         
\usepackage{mathtools}      
\usepackage{bm}             
\usepackage{algorithm}      
\usepackage{algpseudocode}  
\usepackage{array}          
\usepackage{siunitx}        
\usepackage{enumitem}       
\usepackage{nicefrac}       
\usepackage{microtype}      
\usepackage{capt-of}        
\usepackage{float}          
\usepackage{wrapfig}        
\usepackage{xcolor}         
\usepackage{placeins}       
\usepackage{hyperref}       
\usepackage{bookmark}       

\hypersetup{
  pdftitle={Spectral Transformer Neural Processes},
  pdfauthor={Xianhe Chen, Hao Chen, Yingzhen Li},
  hidelinks
}

\newtheorem{proposition}{Proposition}[section]
\newtheorem{corollary}[proposition]{Corollary}
\theoremstyle{definition}
\newtheorem{definition}[proposition]{Definition}
\theoremstyle{remark}

\DeclareMathOperator{\diag}{diag}

\title{Spectral Transformer Neural Processes}

\author{%
  Xianhe Chen\thanks{Equal contribution.} \\
  University of Cambridge \\
  \And
  Hao Chen\footnotemark[1] \\
  Tencent \\
  \And
  Yingzhen Li\thanks{Corresponding author: \texttt{yingzhen.li@imperial.ac.uk}.} \\
  Imperial College London \\
}

\begin{document}

\maketitle

\begin{abstract}
  Time series, spatial data, and images are natural applications of Neural
  Processes. However, when such data exhibit strong \textit{periodicity} and
  \textit{quasi-periodicity}, existing methods often suffer from
  underfitting and generalise poorly beyond the training distribution. In this
  work, we propose \textit{Spectral Transformer Neural Processes} (STNPs), a
  frequency-aware extension of Transformer Neural Processes (TNPs). STNPs
  introduce a \textit{Spectral Aggregator} that estimates an empirical context spectrum,
  compresses it into a spectral mixture, samples task-adaptive \textit{spectral
  features}, and concatenates them with time-domain embeddings, thereby injecting a
  spectral-mixture-kernel bias into TNPs. This design reshapes the
  similarity geometry, allowing inputs that are distant in Euclidean space to
  remain close in an induced periodic manifold while enhancing time--frequency interactions. Extensive experiments on
  synthetic regression tasks, real-world time-series datasets, and an image
  dataset demonstrate that STNPs consistently improve predictive performance
  over existing baselines, extending Neural Processes beyond translation
  equivariance towards effective modelling of periodicity and
  quasi-periodicity.
\end{abstract}

\section{Introduction}

Neural Processes (NPs)~\cite{Garnelo2018ConditionalNP,Garnelo2018NeuralP}
map a context set to a predictive stochastic process, enabling data-efficient
learning with uncertainty quantification. Yet in many applications, such as
time series, spatial data, and images, the underlying data exhibit
increasingly challenging structures: \textit{translation equivariance}, \textit{periodicity}, and
\textit{quasi-periodicity}. These structures respectively require consistent responses to input
shifts, recognition of repeating patterns across scales or frequencies, and recognition of
approximately repeating patterns with drifting period, phase, amplitude, or shape.

When data exhibit periodicity or quasi-periodicity, most existing NP models
suffer a substantial drop in performance and may even collapse to mean
predictions. Prior work has mainly improved translation equivariance: ConvCNPs
\cite{Gordon2019ConvolutionalCN} use convolutional neural networks, while
TETNPs \cite{Ashman2024TranslationET} incorporate relative positional offsets
\(\Delta x\) into TNP attention \cite{Nguyen2022TransformerNP}. Although
effective for translational structure, these methods do not explicitly capture
periodic or quasi-periodic patterns. We argue that the bottleneck 
lies in the representation geometry through which context information is aggregated. In TNPs, 
target-to-context attention compares inputs in learned time-domain embedding and hidden spaces. 
If these embeddings do not encode frequency-domain structure, distant but phase-aligned observations 
remain difficult to relate. Standard ReLU-MLP embeddings do not 
encode periodicity by construction. This motivates a context-conditioned spectral 
representation that exposes task-adaptive frequency information to the downstream neural process.

We propose STNPs, which augment TNPs with
task-adaptive \textit{spectral features} inferred from the context set.
Motivated by Bochner's theorem \cite{Cooper1957HarmonicAA},
STNPs use a \textit{Spectral Aggregator} to project context points into
the frequency domain and summarise their energy spectrum. The resulting
high-dimensional discrete spectrum is compressed and sampled into compact
spectral features, which are then concatenated with the original embeddings.
This allows attention to compare inputs in an augmented time--frequency
representation. From a Gaussian-process perspective, STNPs amortise a
task-specific spectral mixture (SM) kernel \cite{Wilson2013GaussianPK} from the context set: the Spectral
Aggregator infers frequency-domain structure, while the query--key dot product
converts the resulting spectral features into a covariance-like similarity
within attention. This implicitly equips TNPs with an adaptive spectral
  mixture kernel, combining the multi-frequency stationary and quasi-periodic inductive biases of spectral GPs with the
flexibility of TNPs. Our key contributions are summarised as follows:
\begin{itemize}[leftmargin=1.4em,itemsep=2pt,topsep=2pt]
\item We propose STNPs, which
augment TNPs with a Spectral Aggregator that extracts contextual
information in the frequency domain with a compressive feature-sampling mechanism to transform this information into spectral features.
\item We show that the proposed raw spectral features are conditionally
translation equivariant and their inner product is conditionally translation
invariant. In expectation, this inner product defines an SM kernel. These
feature-level results motivate a context-adaptive spectral inductive bias for
TNPs while preserving flexible learned attention.
\item We demonstrate that STNPs outperform existing baselines on synthetic
regression, forecasting and imputation across seven real-world time-series
datasets, and image completion. The official code is available at
\url{https://github.com/ChanXianHe/stnp}.
\end{itemize}

\section{Background and Formal Problem Statement}

We work in the standard episodic meta-learning setting. Let
\(\mathcal{X}=\mathbb{R}^{d_x}\) and \(\mathcal{Y}=\mathbb{R}^{d_y}\) denote
the input and output spaces, and let \(\mathcal{F}\) be a family of functions
\(f:\mathcal{X}\to\mathcal{Y}\). We assume a distribution \(p(f)\) over tasks,
equivalently over functions in \(\mathcal{F}\). For a sampled task \(f\), we
observe a finite dataset $\mathcal{D}=\{(\mathbf{x}_i,\mathbf{y}_i)\}_{i=1}^{N}$, where all input--output pairs are generated by the same underlying function.
Meta-training proceeds by sampling many such episodes from different functions
drawn from the shared task distribution. The goal is to learn a predictor that
can infer the behaviour of a new task from only a small set of observed context
points and then generalise to unseen target inputs from that same task. For a given episode, the dataset \(\mathcal{D}\) is split into a context set
\(\mathcal{D}_C\) and a target set \(\mathcal{D}_T\):
\[
\mathcal{D}_C=\{(\mathbf{x}_i,\mathbf{y}_i)\}_{i=1}^{M},
\qquad
\mathcal{D}_T=\{(\mathbf{x}_i,\mathbf{y}_i)\}_{i=M+1}^{N},
\qquad
M\leq N.
\]
We use \((\mathbf{x}_C,\mathbf{y}_C)\) and \((\mathbf{x}_T,\mathbf{y}_T)\) to
denote the inputs and outputs contained in \(\mathcal{D}_C\) and
\(\mathcal{D}_T\), respectively, and write a single episode as
\((\mathcal{D}_C,\mathcal{D}_T)\). The meta-learning objective is to optimise
the model parameters \(\theta\) so that the conditional predictor assigns high
probability to the target outputs given the context set
\(
\max_{\theta}\;
\mathbb{E}_{f\sim p(f)}
\left[
\mathbb{E}_{(\mathcal{D}_C,\mathcal{D}_T)\sim p(\cdot\mid f)}
\log p_{\theta}(\mathbf{y}_T\mid \mathbf{x}_T,\mathcal{D}_C)
\right].
\)
At meta-test time, the learned model is evaluated on unseen functions drawn
from the same task distribution, using only the context observations available
within each new episode.

\subsection{Neural Processes and Transformer Neural Processes}

Neural Processes (NPs) model the conditional distribution over target outputs
\(\mathbf{y}_T\) given target inputs \(\mathbf{x}_T\) and a context set
\(\mathcal{D}_C\), i.e.,
\(p_\theta(\mathbf{y}_T \mid \mathbf{x}_T, \mathcal{D}_C)\).
In deterministic NP and CNP variants, the context is encoded into a
permutation-invariant global representation
\(
r = \mathrm{Agg}\{\phi(\mathbf{x}_i,\mathbf{y}_i)\}_{i=1}^M,
\)
where \(\phi\) is a pointwise encoder and \(\mathrm{Agg}\) is typically
mean or sum pooling. The predictive distribution factorises as
\(
p_\theta(\mathbf{y}_T \mid \mathbf{x}_T, \mathcal{D}_C)
= \prod_{i=M+1}^{N} p_\theta(\mathbf{y}_i \mid \mathbf{x}_i, r).
\)
Attention-based variants replace the shared summary \(r\) with
target-dependent representations \(r_i=r(\mathbf{x}_i,\mathcal{D}_C)\), yielding
\(
p_\theta(\mathbf{y}_T \mid \mathbf{x}_T, \mathcal{D}_C)
= \prod_{i=M+1}^{N} p_\theta(\mathbf{y}_i \mid \mathbf{x}_i, r_i).
\)
Transformer Neural Processes (TNPs) instantiate \(r_i\) using a masked
Transformer encoder~\cite{Vaswani2017AttentionIA}; throughout this paper, TNP
refers to TNP-Diagonal.
Given the token sequence
\(
((\mathbf{x}_1,\mathbf{y}_1),\ldots,(\mathbf{x}_M,\mathbf{y}_M),(\mathbf{x}_{M+1},\mathbf{0}),\ldots,(\mathbf{x}_N,\mathbf{0})),
\)
where the first \(M\) tokens are context points and the remaining tokens are
target queries, an input MLP maps each token to an embedding \(e_i^{(0)}\).
A stack of masked self-attention layers produces final representations
\(e_i^{(L)}\), where each target token attends to all context tokens but not to
other target tokens. For each target \(i>M\), a prediction head outputs the
parameters of a diagonal Gaussian,
\[
(\bm{\mu}_i,\bm{\sigma}_i)=g_{\mathrm{pred}}(e_i^{(L)}), \qquad
p_\theta(\mathbf{y}_i \mid \mathbf{x}_i,\mathcal{D}_C)=\mathcal{N}(\bm{\mu}_i,\mathrm{diag}(\bm{\sigma}_i^2)).
\]
The model is trained by maximising the predictive likelihood over target points,
equivalently minimising
\(
\mathcal{L}_{\mathrm{TNP}}(\theta)
= - \mathbb{E}_{\mathcal{D}}\!\left[\sum_{i=M+1}^{N}
\log p_\theta(\mathbf{y}_i \mid \mathbf{x}_i,\mathcal{D}_C)\right].
\)

\section{Methods}
We interpret target-to-context attention in TNPs as kernel regression~\cite{nadaraya1964estimating} in the
learned embedding and hidden spaces. If the embedding space fails to encode
structural properties, attention in the hidden space may degenerate, making it
difficult to relate distant but periodically correlated observations. Standard
TNPs use shallow ReLU-MLP embeddings, which are continuous piecewise-linear maps
and do not encode periodicity by construction. A recent alternative, Fourier
Analysis Networks (FANs)~\cite{Dong2024FANFA}, replaces part of the ReLU
activations with sinusoidal functions, effectively transforming the embedding
into
\(
[\sin(W_p x),\cos(W_p x),\phi_{\mathrm{MLP}}(x,y)].
\)
Frequencies in FANs are learned as fixed model parameters rather than inferred
from the current context set. This is inadequate in meta-learning settings,
where the period, phase, amplitude, and dominant frequencies may vary across
tasks. Similarly, standard random Fourier features~\cite{Rahimi2007RandomFF} approximate a pre-specified
stationary kernel, which is typically unknown in our setting, and do not infer
task-specific spectra. These observations highlight the difficulty of modelling periodic functions
with TNPs, and more broadly with neural networks, in meta-learning settings.
They also suggest that modifying the embedding space is a promising route for
introducing task-adaptive periodic inductive bias.

By Bochner's theorem, stationary positive-definite kernels correspond to
non-negative spectral measures. This suggests that the frequency-domain
structure of the context set can provide useful information about the underlying
task. Motivated by this connection, we propose a \textit{Spectral Aggregator}
that extracts a task-specific spectral representation from the context
observations. The Spectral Aggregator maps the context set to the parameter
collection of a \(Q\)-component spectral mixture,
\(\mathcal{A}(\mathcal{D}_C)=\theta_{\mathcal{D}_C}:=
\{(w_q,\mu_q,\sigma_q^2)\}_{q=1}^{Q}\).
Given \(\theta_{\mathcal{D}_C}\), we construct spectral features
\(\phi_S(\cdot;\theta_{\mathcal{D}_C}):\mathcal{X}\to\mathbb{R}^{2QD_0}\) using frequencies sampled from the inferred spectral
mixture. The procedure consists of four stages: empirical spectrum estimation,
latent spectral decomposition, probabilistic compression, and spectral feature
construction. In expectation, the inner product of these features induces a
spectral-mixture kernel, thereby injecting a task-adaptive SM-kernel inductive
bias into TNPs while preserving the fast amortised prediction capability of
neural processes. The overall workflow of STNPs is illustrated in
Figure~\ref{fig:method-procedure}.

\begin{figure*}[t]
\centering
\includegraphics[width=\textwidth]{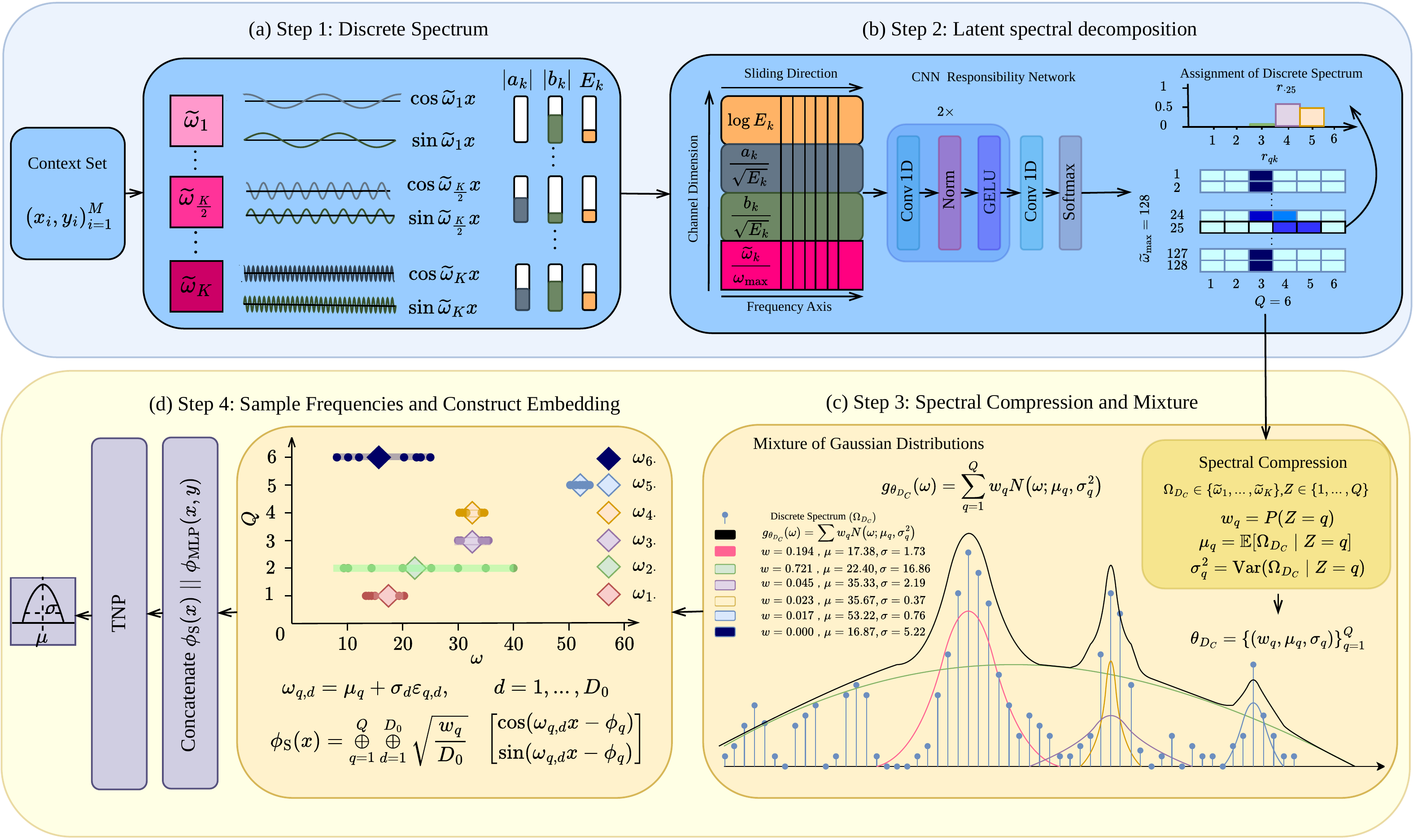}
\caption{Overview of STNPs. (a) An empirical spectrum is estimated from the
context set on a fixed frequency grid via cosine/sine projections of centred
responses. (b) A CNN responsibility network softly assigns each grid frequency
to one of \(Q\) mixture components. (c) The discrete spectrum is compressed
into a Gaussian mixture with weights, means, and variances induced by the
responsibilities. (d) Task-adaptive frequencies sampled from this mixture form
spectral embeddings, which are concatenated with the MLP token representation
and passed to the TNP predictor.}
\label{fig:method-procedure}
\vspace{-0.6em}
\end{figure*}

\subsection{Spectral Aggregator and Spectral Features}

\label{sec:smfe}

\paragraph{Step 1: Empirical discrete spectral distribution.}
Let
\(\{\tilde{\omega}_k\}_{k=1}^K\) be a fixed frequency grid and let \(Q\) be the
number of spectral mixture components. For notational simplicity, we suppress
conditioning on the context set \(\mathcal{D}_C\); for example,
\(\mathbb{P}(\cdot|\mathcal{D}_C)\) is denoted by \(\mathbb{P}(\cdot)\).
We present the scalar-input, scalar-output case here\footnote{For higher-dimensional \(x\) and \(y\), we present the method and discuss these cases in Appendix~\ref{app:multidim}.}.
Spectral energy is estimated with stabiliser \(\varepsilon>0\) on the fixed
frequency grid. The responses are centred by subtracting the context mean
\(y_i^{(s)}=y_i-\bar{y}\), and the resulting centred responses are projected
onto the cosine-sine basis associated with each grid frequency
$\tilde{\omega}_k$:
\begin{equation}
a_k = \frac{1}{M}\sum_{i=1}^{M} y_i^{(s)} \cos(\tilde{\omega}_k x_i),
\qquad
b_k = \frac{1}{M}\sum_{i=1}^{M} y_i^{(s)} \sin(\tilde{\omega}_k x_i),
\qquad
E_k = a_k^2 + b_k^2 + \varepsilon.
\end{equation}
Normalising \(p_k = E_k / \sum_{j=1}^{K} E_j\) yields a discrete
random variable
\(\Omega_{\mathcal{D}_C}\in\{\tilde{\omega}_1,\dots,\tilde{\omega}_K\}\)
with
\(\mathbb{P}(\Omega_{\mathcal{D}_C}=\tilde{\omega}_k)=p_k\), and the
conditional spectral measure
\(\hat{\nu}_{\mathcal{D}_C}=\sum_{k=1}^{K}p_k\delta_{\tilde{\omega}_k}\).

\paragraph{Step 2: Latent spectral decomposition.}
The empirical measure $\hat{\nu}_{\mathcal{D}_C}$ is high-dimensional, spiky, and tied to the fixed frequency grid.
Used directly as features, the fixed frequencies would force the downstream model to operate on a 
$K$-dimensional representation that changes discontinuously with small shifts in context, and whose resolution is capped by the grid spacing. 
To obtain a compact and continuous spectral representation, we introduce a latent mixture variable
$Z \in \{1,\dots,Q\},$
which indicates the spectral component to which a grid frequency belongs. 

We form a spectral summary sequence along the frequency axis
\begin{equation}
S_k=
\left[
\log E_k,\;
\frac{a_k}{\sqrt{E_k}},\;
\frac{b_k}{\sqrt{E_k}},\;
\frac{\tilde{\omega}_k}{\omega_{\max}}
\right]
\in \mathbb R^{C_{\mathrm{in}}},
\end{equation}
where the last three entries encode the phase direction and normalised frequency coordinate. 
We feed $
\{S_k\}_{k=1}^K \in \mathbb R^{K\times C_{\mathrm{in}}}
$ to a lightweight convolutional network
$
h_\phi:\mathbb R^{K\times C_{\mathrm{in}}}\to \mathbb R^{K\times Q},
$
which operates along the frequency axis and outputs component logits \(\ell_{qk}\) for each grid frequency. The use of a convolutional network imposes a useful inductive bias, with mixture assignments determined by local spectral geometry such as peak shape, bandwidth, and neighbourhood continuity across adjacent frequencies rather than by pointwise energy values alone. The responsibilities are then defined as a learned conditional categorical distribution over mixture components:
\(
r_{qk}
:=
\mathbb P_\phi(Z=q \mid \Omega_{\mathcal{D}_C}=\tilde{\omega}_k)
=
\frac{\exp(\ell_{qk})}{\sum_{q'=1}^{Q}\exp(\ell_{q'k})},
\sum_{q=1}^{Q} r_{qk}=1.
\)
This step can be interpreted as a smooth amortised decomposition of the empirical discrete 
spectrum into a small number of spectral components. 

\paragraph{Step 3: Probabilistic spectral compression.}
Given the discrete spectral distribution \(p_k\) and the responsibilities
\(r_{qk}\), the induced joint distribution and the mixture weight of component $q$ are
\begin{equation}
\mathbb{P}(Z=q,\Omega_{\mathcal{D}_C}=\tilde{\omega}_k)=r_{qk}p_k,\qquad w_q
=
\mathbb{P}(Z=q)
=
\sum_{k=1}^{K} r_{qk} p_k.
\end{equation}
This yields the component-wise posterior over the frequency
grid $\mathbb{P}(\Omega_{\mathcal{D}_C}=\tilde{\omega}_k \mid Z=q)
=
\frac{r_{qk}p_k}{w_q}.$
We then define the Gaussian parameters by the conditional moments of
\(\Omega_{\mathcal{D}_C}\):
\begin{equation}
\mu_q
=
\mathbb{E}[\Omega_{\mathcal{D}_C} \mid Z=q]
=
\frac{\sum_{k=1}^{K} r_{qk}p_k \tilde{\omega}_k}{w_q},\quad\sigma_q^2
=
\mathrm{Var}(\Omega_{\mathcal{D}_C} \mid Z=q)
=
\frac{\sum_{k=1}^{K} r_{qk}p_k(\tilde{\omega}_k-\mu_q)^2}{w_q}.
\end{equation}
This yields a compact continuous approximation 
\(
g_{\theta_{\mathcal{D}_C}}(\omega)
=
\sum_{q=1}^{Q} w_q \mathcal N(\omega;\mu_q,\sigma_q^2) \text{ with }
\theta_{\mathcal{D}_C}=\{(w_q,\mu_q,\sigma_q^2)\}_{q=1}^{Q}
\)
to the discrete conditional spectrum \(\hat{\nu}_{\mathcal{D}_C}\).

\paragraph{Step 4: Sampling and embedding.}
For each component, draw \(D_0\) frequencies via reparameterisation
\(\omega_{q,d}=\mu_q+\sigma_q\epsilon_{q,d}\),
\(\epsilon_{q,d}\sim\mathcal N(0,1)\). The spectral embedding is
\begin{equation}
\phi_S(x)
=
\bigoplus_{q=1}^{Q}
\bigoplus_{d=1}^{D_0}
\sqrt{\frac{w_q}{D_0}}
\begin{bmatrix}
\cos(\omega_{q,d}x-\phi_q)\\
\sin(\omega_{q,d}x-\phi_q)
\end{bmatrix}
\in \mathbb{R}^{2QD_0},
\end{equation}
where \(\phi_q=0\) if phase shifts are disabled, otherwise estimated via a
component-wise circular mean (Appendix~\ref{app:optional-phase-estimation}).
The factor \(\sqrt{w_q/D_0}\) scales each component by its inferred mass.

\paragraph{Concatenation.}
We concatenate \(\phi_S(x)\) with a non-spectral MLP branch and project to
model dimension:
\(
e(x,y)=\mathrm{Proj}\!\left([\phi_S(x)\;\|\;\phi_{\mathrm{MLP}}(x,y)]\right).
\)
The result is passed to the TNP backbone.

\subsection{Properties}
Our aim is to equip TNPs with context-adaptive periodic and quasi-periodic
inductive biases while retaining flexible learned attention.
Proposition~\ref{prop:translation-equivariance},
Corollary~\ref{cor:translation-invariance}, and
Proposition~\ref{prop:stationary-kernel} characterise only the raw features
\(\phi_{\mathrm S}\) and their inner product. The expected SM kernel motivates
the spectral inductive bias, while MLP embeddings and learned attention mappings
retain model flexibility. Proposition~\ref{prop:process-consistency}
separately establishes model-level process consistency under target-wise prediction.
Proofs are provided in Appendix~\ref{app:sec3-proofs}.

\begin{proposition}[Translation equivariance of the spectral features]
\label{prop:translation-equivariance}
Fix the inferred spectral-mixture parameters
\(
\theta_{\mathcal{D}_C}=\{(w_q,\mu_q,\sigma_q^2,\phi_q)\}_{q=1}^Q,
\)
and the sampled frequencies
\(
\omega_{q,d}=\mu_q+\sigma_q\epsilon_{q,d}.
\)
Then, for any translation \(\Delta\in\mathbb R\), there exists a
block-diagonal orthogonal matrix
\[
\rho_\Delta(\theta_{\mathcal{D}_C})
=
\bigoplus_{q=1}^{Q}\bigoplus_{d=1}^{D_0}
R(\omega_{q,d}\Delta),
\qquad
R(\alpha)=
\begin{bmatrix}
\cos\alpha & -\sin\alpha\\
\sin\alpha & \cos\alpha
\end{bmatrix},
\]
such that
\(
\phi_{\mathrm{S}}(x+\Delta)
=
\rho_\Delta(\theta_{\mathcal{D}_C})\,\phi_{\mathrm{S}}(x).
\)
Hence, conditioned on \(\theta_{\mathcal{D}_C}\) and the sampled frequencies, the
spectral features are translation equivariant
(Appendix~\ref{app:translation-definitions}).
\end{proposition}

\begin{corollary}[Translation invariance of the induced inner product]
\label{cor:translation-invariance}
Under the assumptions of
Proposition~\ref{prop:translation-equivariance}, define
\(
\hat{k}_{\theta_{\mathcal{D}_C}}(x,x')
=
\phi_{\mathrm{S}}(x)^\top \phi_{\mathrm{S}}(x').
\)
Then, for any \(\Delta\in\mathbb R\),
\(
\hat{k}_{\theta_{\mathcal{D}_C}}(x+\Delta,x'+\Delta)
=
\hat{k}_{\theta_{\mathcal{D}_C}}(x,x').
\)
Hence the induced kernel estimator is translation invariant
(Appendix~\ref{app:translation-definitions}).
\end{corollary}

\begin{proposition}[Stationarity of the conditional kernel]
\label{prop:stationary-kernel}
Conditioned on the inferred mixture parameters \(\theta_{\mathcal{D}_C}\), let
\(
k_{\theta_{\mathcal{D}_C}}(x,x')
:=
\mathbb E\!\left[
\phi_{\mathrm{S}}(x)^\top \phi_{\mathrm{S}}(x')
\mid \theta_{\mathcal{D}_C}
\right],
\)
where the expectation is taken over the sampled frequencies
\(\omega_{q,d}\sim\mathcal N(\mu_q,\sigma_q^2)\). Then for
\(
\tau=x-x'
\)
we have
\(
k_{\theta_{\mathcal{D}_C}}(\tau)
=
\sum_{q=1}^{Q}
w_q
\exp\!\left(-\frac{1}{2}\sigma_q^2\tau^2\right)
\cos(\mu_q\tau).
\)
Therefore, the conditional kernel induced by the spectral branch is
stationary and is indeed a spectral-mixture kernel \cite{Wilson2013GaussianPK}.
\end{proposition}

\begin{proposition}[STNPs define valid conditional stochastic processes]
\label{prop:process-consistency}
The finite-dimensional predictive distributions of STNPs satisfy the
conditions of the Kolmogorov extension theorem stated in
Appendix~\ref{app:kolmogorov-extension} and therefore define valid
conditional stochastic processes.
\end{proposition}

\section{Experiments}

We evaluate STNPs on four synthetic function families, one image-completion benchmark, and seven real-world
time-series datasets. We compare against a
broad set of neural-process baselines, including NP, CNP, and Bootstrapping NP~\cite{Lee2020BootstrappingNP}
(BNP), together with their attentive variants ANP \cite{Kim2019AttentiveNP}, CANP, and BANP, as well as
TNP, TETNP, ConvCNP, and SConvCNP. The central
question is whether STNPs can better model periodic and quasi-periodic
structure.

\begin{figure*}[t]
\centering
\includegraphics[width=\textwidth]{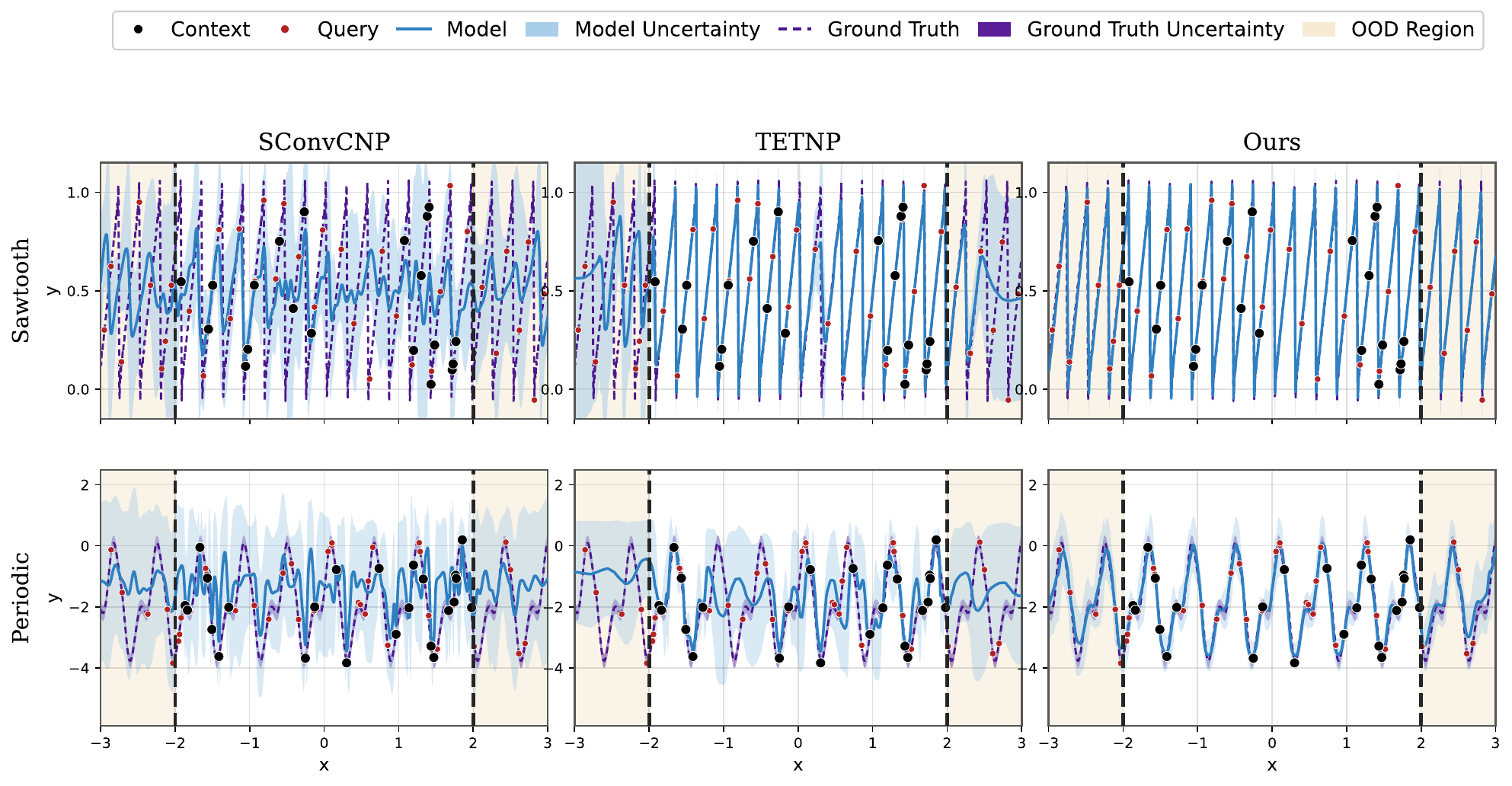}
{\vspace{-0.9em}
\setlength{\abovecaptionskip}{0pt}
\setlength{\belowcaptionskip}{0pt}
\caption{Example predictions on synthetic sawtooth (top) and periodic (bottom) tasks.
  Columns compare SConvCNP, TETNP, and STNP (ours). Black and red dots denote context and query points respectively. The ground truth is shown in purple, and the blue curve denotes the predictive mean; the bands represent $\pm 2$ standard deviations. The shaded regions indicate out-of-distribution ranges.}
\label{fig:synthetic-regression-qualitative}}

\vspace{0.7em}

\begingroup
\footnotesize
\setlength{\tabcolsep}{2.3pt}
\newcommand{\meanstd}[2]{#1$_{\pm #2}$}%
\captionof{table}{Comparison of STNP with the baselines on log-likelihood of the target points on various synthetic regression tasks. We train each method with 5 different seeds and report the mean and standard deviation. For the Periodic task, the reported results use \(m_{\min}=20\).}
\label{tab:regression_results}
\vspace{0.25em}
\resizebox{\textwidth}{!}{%
\begin{tabular}{@{}lccccccccccc@{}}
    \toprule
    \textbf{Task} & \textbf{NP} & \textbf{CNP} & \textbf{ANP} & \textbf{CANP} & \textbf{BNP} & \textbf{BANP} & \textbf{ConvCNP} & \textbf{SConvCNP} & \textbf{TNP} & \textbf{TETNP} & \textbf{Ours} \\
    \midrule
    RBF & \meanstd{0.27}{0.01} & \meanstd{0.26}{0.02} & \meanstd{0.81}{0.00} & \meanstd{0.79}{0.00} & \meanstd{0.38}{0.02} & \meanstd{0.82}{0.01} & \meanstd{1.31}{0.01} & \meanstd{1.37}{0.01} & \meanstd{1.39}{0.00} & \meanstd{1.39}{0.01} & \textbf{\meanstd{1.41}{0.01}} \\
    Mat\'ern & \meanstd{0.11}{0.02} & \meanstd{0.09}{0.01} & \meanstd{0.64}{0.00} & \meanstd{0.62}{0.01} & \meanstd{0.21}{0.03} & \meanstd{0.53}{0.00} & \meanstd{0.97}{0.03} & \meanstd{1.02}{0.02} & \meanstd{1.03}{0.00} & \textbf{\meanstd{1.04}{0.01}} & \textbf{\meanstd{1.04}{0.01}} \\
    Sawtooth & \meanstd{-0.16}{0.01} & \meanstd{-0.16}{0.01} & \meanstd{1.04}{0.01} & \meanstd{1.04}{0.01} & \meanstd{-0.12}{0.02} & \meanstd{1.09}{0.02} & \meanstd{1.94}{0.01} & \meanstd{2.50}{0.02} & \meanstd{-0.16}{0.01} & \meanstd{2.70}{0.03} & \textbf{\meanstd{2.92}{0.01}} \\
    Periodic & \meanstd{-0.36}{0.03} & \meanstd{-1.10}{0.01} & \meanstd{-0.45}{0.01} & \meanstd{-0.19}{0.01} & \meanstd{-0.36}{0.02} & \meanstd{-0.40}{0.02} & \meanstd{-0.01}{0.05} & \meanstd{0.44}{0.02} & \meanstd{-0.35}{0.01} & \meanstd{0.20}{0.02} & \textbf{\meanstd{1.12}{0.02}} \\
    \bottomrule
\end{tabular}}
\endgroup
\end{figure*}

\subsection{Synthetic Regression}
\label{sec:synthetic-regression}
Following the standard benchmark \cite{Nguyen2022TransformerNP,Gordon2019ConvolutionalCN,Wang2024RnyiNP,Lee2020BootstrappingNP}, we evaluate STNPs on 1D regression tasks where functions are drawn from Gaussian process priors with RBF, 
Mat\'ern, and periodic kernels, together with a sawtooth function family. To ensure the model learns from a diverse set of functions, the GP hyperparameters are randomised 
during the generation process. Additional experimental details are provided in Appendix~\ref{app:synthetic-regression-details}.
Table~\ref{tab:regression_results} 
reports a performance comparison of our proposed method against various baseline models.
Figure~\ref{fig:synthetic-regression-qualitative} compares the predictive performance of SConvCNP, TETNP, and STNP. 
Compared with baselines, STNP fits the observed data more accurately,
extrapolates the underlying function more reliably in OOD regions, and remains
stable in context-sparse areas such as around \(x \approx -1\) (bottom), suggesting
stronger interpolation and extrapolation robustness.
This likely reflects the Spectral Aggregator's ability to extract
task-adaptive frequency information from sparse context points and inject it
into the TNP representation, enabling correlations between distant but
phase-aligned inputs beyond local time-domain proximity.

\subsection{Image Completion}
\begin{wrapfigure}{r}{0.48\textwidth}
\vspace{-1.0em}
\centering
\includegraphics[width=\linewidth]{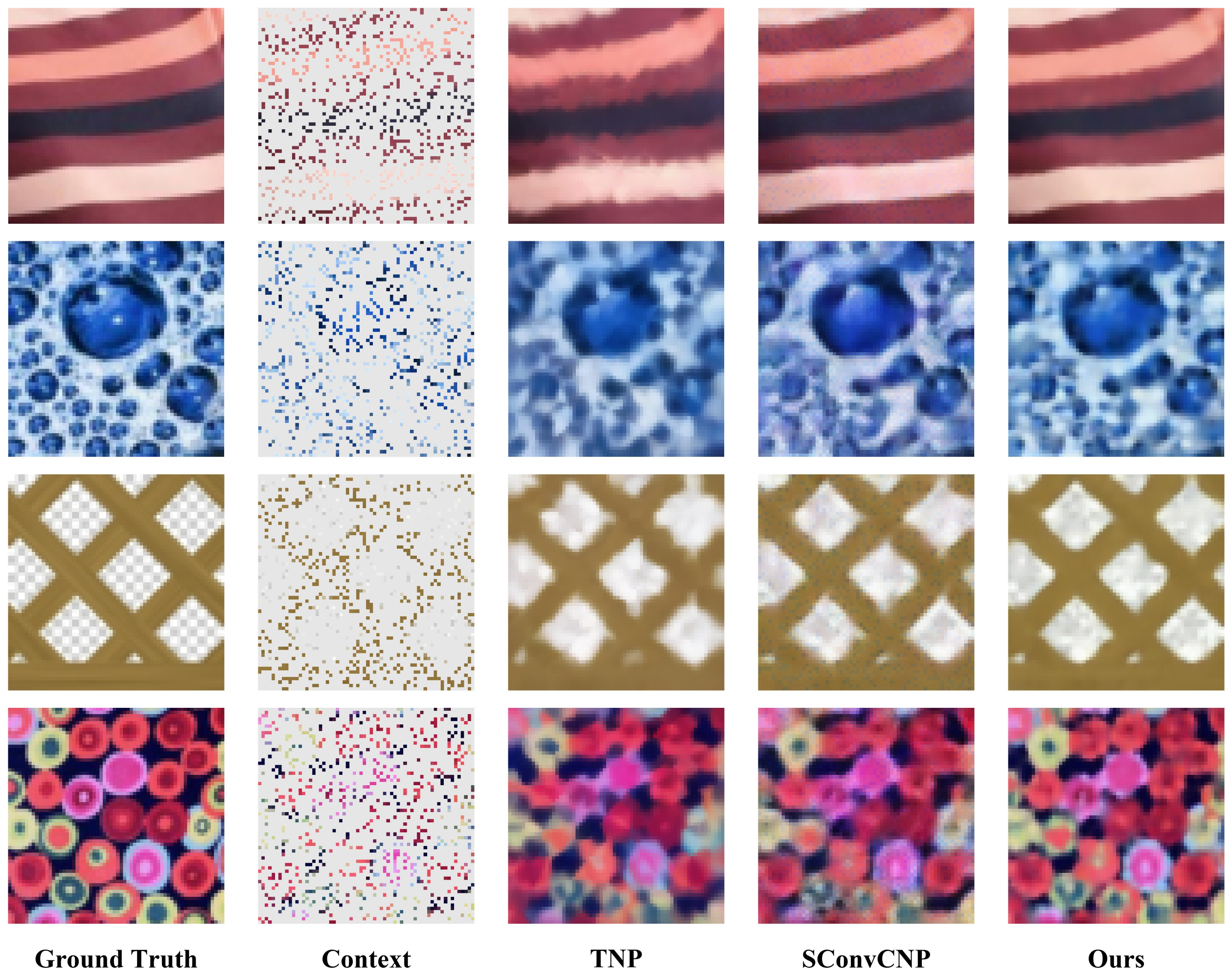}
\vspace{-0.6em}
\caption{Illustrative model outputs for image completion on the DTD dataset. Grey pixels indicate query regions, while the remaining pixels serve as context observations.}
\label{fig:dtd-image-completion-qualitative}
\vspace{-1.0em}
\end{wrapfigure}

Following SConvCNP, we formulate image completion as spatial regression from
2D pixel coordinates to intensity values. We use images from the Describable Textures Dataset (DTD)
\cite{Cimpoi2013DescribingTI}, and construct each task from a processed
\(64\times64\) subsampled crop of an original image. For each batch, the number of context pixels is sampled uniformly as
\(M \sim \mathcal{U}[5,1024)\) and is shared across all tasks in
the batch, while all remaining pixels are treated as query points. Additional experimental
details are provided in Appendix~\ref{app:dtd_configs}. Following
\cite{Mohseni2024SpectralCC}, TETNP cannot be trained on this benchmark due to
out-of-memory failures even with 8 A100 GPUs. Table~\ref{tab:image-completion-results}
summarises the quantitative results (PSNR and SSIM are
computed over the full \(64\times64\) image), and Figure~\ref{fig:dtd-image-completion-qualitative}
shows representative qualitative completions. STNP yields
sharper and more structurally faithful completions than TNP and SConvCNP,
particularly for repeated textures such as stripes, bubbles, lattices, and
circular motifs. Whereas TNP over-smooths local structure and SConvCNP mainly
recovers coarse spatial continuity, STNP better preserves global periodic layouts.

\begin{table}[H]
\centering
\caption{Comparison of predictive performance on image-completion tasks
constructed from the DTD dataset. Higher log-likelihood, PSNR, and SSIM,
and lower RMSE indicate better performance.}
\label{tab:image-completion-results}
\vspace{0.25em}
{\scriptsize
\setlength{\tabcolsep}{3pt}
\newcommand{\pmerr}[1]{{\scriptsize$\pm #1$}}%
\newcommand{\bpmerr}[1]{{\scriptsize\boldmath$\pm #1$}}%
\resizebox{\textwidth}{!}{%
\begin{tabular}{l*{6}{r@{\;}l}}
\toprule
Metric & \multicolumn{2}{c}{CNP} & \multicolumn{2}{c}{AttCNP} & \multicolumn{2}{c}{TNP} & \multicolumn{2}{c}{ConvCNP} & \multicolumn{2}{c}{SConvCNP} & \multicolumn{2}{c}{\textbf{Ours}} \\
\midrule
Log-likelihood $\uparrow$ & 0.67 & \pmerr{0.01} & 1.39 & \pmerr{0.02} & 1.39 & \pmerr{0.05} & 1.48 & \pmerr{0.00} & \textbf{1.50} & \bpmerr{0.02} & \textbf{1.51} & \bpmerr{0.01} \\
RMSE $\downarrow$ & 0.14 & \pmerr{0.00} & \textbf{0.08} & \bpmerr{0.00} & 0.09 & \pmerr{0.00} & \textbf{0.08} & \bpmerr{0.00} & \textbf{0.08} & \bpmerr{0.00} & \textbf{0.08} & \bpmerr{0.00} \\
PSNR $\uparrow$ & 18.7 & \pmerr{0.31} & 22.8 & \pmerr{0.18} & 22.9 & \pmerr{0.11} & 23.4 & \pmerr{0.14} & 23.6 & \pmerr{0.16} & \textbf{24.0} & \bpmerr{0.15} \\
SSIM $\uparrow$ & 0.35 & \pmerr{0.05} & 0.56 & \pmerr{0.02} & 0.57 & \pmerr{0.01} & 0.61 & \pmerr{0.01} & 0.62 & \pmerr{0.01} & \textbf{0.64} & \bpmerr{0.01} \\
\bottomrule
\end{tabular}}}
\end{table}

\subsection{Real-World Time Series and Multivariate Regression}
We consider block forecasting on a six-dataset general benchmark and Chichester
Harbour, and imputation on California Traffic Flow. The latter two tasks use a meta-learning setting. For block forecasting,
we train all compared methods with an \(L^2\) objective and mainly evaluate
point-forecast accuracy; for imputation, we optimise target-point NLL to assess
both accuracy and uncertainties.

\textbf{General Benchmark.} We evaluate long-horizon block forecasting on the
six-dataset benchmark of \cite{Wu2021AutoformerDT}: \(\texttt{Electricity}\),
\(\texttt{ETTh1}\) \cite{Zhou2020InformerBE}, \(\texttt{Exchange~Rate}\)
\cite{Lai2017ModelingLA}, \(\texttt{National~Illness}\), \(\texttt{Traffic}\),
and \(\texttt{Weather}\). Together, these datasets test periodic and
quasi-periodic recurrence as well as performance under non-stationary dynamics.
Following standard forecasting protocols
\cite{Zhang2023MultiresolutionTT, Zhou2022FEDformerFE}, models use
a look-back window of length \(L\) to predict the next \(T\) timestamps, with
\(L=32,T=24\) for National Illness and \(L=T=96\) otherwise. All datasets are
chronologically split following the benchmark settings: \(6:2:2\) for ETTh1
and \(7:1:2\) for the remaining datasets. Table~\ref{tab:mtsf-results} shows that STNP outperforms the other
NP-family models across the benchmark. TNP remains weak,
echoing its failures on periodic-kernel and sawtooth tasks. Figure~\ref{fig:forecast-qualitative} shows that STNP tracks
recurring patterns more accurately than baselines. Details are in
Appendix~\ref{app:realworld-timeseries-details}.

\begin{figure}[!b]
\centering
\captionof{table}{General benchmark performance. Each dataset reports MAE/MSE;
the last column gives relative error reduction from TNP to STNP. All results
are averaged over five seeds.}
\label{tab:mtsf-results}
\vspace{0.05em}
\begingroup
\tiny
\renewcommand{\arraystretch}{1.00}
\resizebox{\textwidth}{!}{%
\begin{tabular}{llcccccccccccc}
\toprule
\multirow{2}{*}{Dataset} & \multirow{2}{*}{Metric} & \multicolumn{11}{c}{Models} & \multirow{2}{*}{Imp. (\%)} \\
\cmidrule(lr){3-13}
 &  & ANP & CANP & BNP & BANP & ConvCNP & SConvCNP & TNP & TETNP & Autoformer & Informer & \textbf{Ours} & \\
\midrule
\multirow{2}{*}{\textbf{Electricity}} & MAE & 0.364 & 0.366 & 0.399 & 0.367 & 0.360 & 0.364 & 0.830 & 0.367 & 0.317 & 0.368 & \textbf{0.281} & \textbf{66.1\%} \\
 & MSE & 0.286 & 0.285 & 0.332 & 0.289 & 0.287 & 0.284 & 1.012 & 0.283 & 0.201 & 0.274 & \textbf{0.179} & \textbf{82.3\%} \\
\midrule
\multirow{2}{*}{\textbf{ETTh1}} & MAE & 0.820 & 0.798 & 0.794 & 0.757 & 0.837 & 0.827 & 0.800 & 0.639 & \textbf{0.459} & 0.713 & 0.487 & \textbf{39.1\%} \\
 & MSE & 1.182 & 1.106 & 1.067 & 1.002 & 1.139 & 1.105 & 1.106 & 0.723 & \textbf{0.449} & 0.865 & 0.539 & \textbf{51.3\%} \\
\midrule
\multirow{2}{*}{\textbf{Exchange Rate}} & MAE & 0.840 & 0.853 & 0.845 & 0.890 & 0.835 & 0.846 & 1.443 & 0.783 & 0.323 & 0.752 & \textbf{0.208} & \textbf{85.6\%} \\
 & MSE & 0.991 & 1.019 & 1.103 & 1.110 & 1.083 & 1.048 & 3.036 & 0.898 & 0.197 & 0.847 & \textbf{0.087} & \textbf{97.1\%} \\
\midrule
\multirow{2}{*}{\textbf{National Illness}} & MAE & 1.395 & 1.396 & 1.337 & 1.351 & 1.272 & 1.281 & 1.827 & 1.528 & 1.287 & 1.677 & \textbf{1.040} & \textbf{43.1\%} \\
 & MSE & 3.885 & 4.314 & 4.021 & 4.020 & 3.955 & 3.733 & 6.593 & 5.308 & 3.483 & 5.764 & \textbf{3.080} & \textbf{53.3\%} \\
\midrule
\multirow{2}{*}{\textbf{Traffic}} & MAE & 0.372 & 0.366 & 0.393 & 0.365 & 0.370 & 0.332 & 0.800 & 0.328 & 0.388 & 0.391 & \textbf{0.318} & \textbf{60.3\%} \\
 & MSE & 0.670 & 0.661 & 0.702 & 0.662 & 0.690 & 0.618 & 1.462 & 0.627 & 0.613 & 0.719 & \textbf{0.582} & \textbf{60.2\%} \\
\midrule
\multirow{2}{*}{\textbf{Weather}} & MAE & 0.516 & 0.418 & 0.599 & 0.464 & 0.510 & 0.525 & 0.610 & 0.257 & 0.336 & 0.384 & \textbf{0.227} & \textbf{62.8\%} \\
 & MSE & 0.518 & 0.344 & 0.686 & 0.406 & 0.509 & 0.593 & 0.641 & 0.202 & 0.266 & 0.300 & \textbf{0.185} & \textbf{71.1\%} \\
\bottomrule
\end{tabular}
}
\endgroup

\vspace{0.15em}

\includegraphics[width=\textwidth]{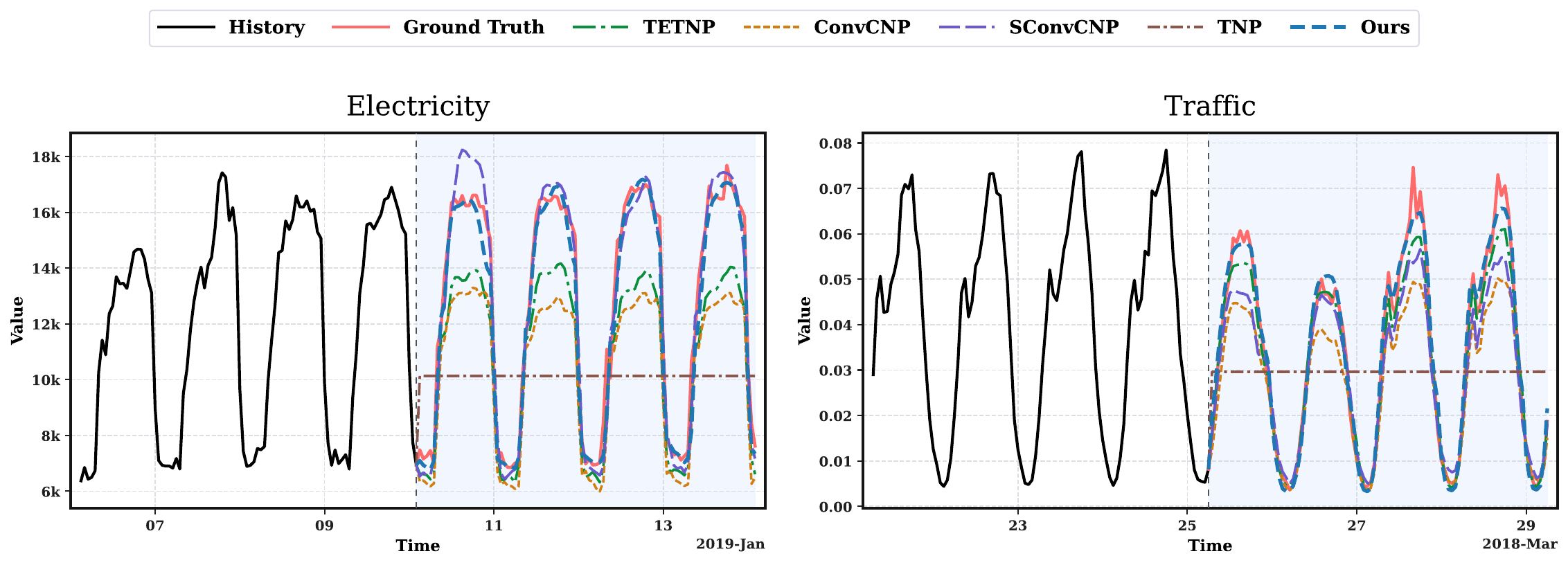}
\vspace{-0.75em}
\captionof{figure}{Representative forecasting windows on the Electricity and
Traffic datasets. The red curve denotes the ground truth and the coloured curve
denotes the predictive mean.}
\label{fig:forecast-qualitative}
\end{figure}

\textbf{California Traffic Flow.} This imputation task uses the 2020 portion of
California Traffic Flow \cite{Liu2023LargeSTAB}, where COVID-19 disruptions
induce stronger variability. We partition the data into non-overlapping 14-day
windows, sample \(M \sim \mathcal{U}[5,25)\) context points from \(N=50\)
observations in each window, and use the remaining \(N-M\) observations as
targets. 
Table~\ref{tab:california-traffic-results} shows that STNP attains the best
log-likelihood and RMSE, and Figure~\ref{fig:traffic-hard-gap-cases} shows that
it provides better uncertainty estimates and follows sharp traffic-flow peaks
more closely than competing baselines.

\begin{figure}[!t]
\centering
\captionof{table}{Comparison of predictive performance across methods on tasks
constructed from California traffic flow measurements. Lower RMSE and higher
log-likelihood indicate better performance.}
\label{tab:california-traffic-results}
\vspace{0.25em}
{\small
\setlength{\tabcolsep}{3pt}
\newcommand{\pmerr}[1]{{\scriptsize$\pm #1$}}%
\newcommand{\bpmerr}[1]{{\scriptsize\boldmath$\pm #1$}}%
\resizebox{\columnwidth}{!}{%
\begin{tabular}{l*{7}{r@{\;}l}}
\toprule
Metric & \multicolumn{2}{c}{CNP} & \multicolumn{2}{c}{AttCNP} & \multicolumn{2}{c}{TNP} & \multicolumn{2}{c}{TETNP} & \multicolumn{2}{c}{ConvCNP} & \multicolumn{2}{c}{SConvCNP} & \multicolumn{2}{c}{\textbf{Ours}} \\
\midrule
Log-likelihood $\uparrow$ & 1.73 & \pmerr{0.10} & 1.80 & \pmerr{0.01} & 1.93 & \pmerr{0.06} & 1.72 & \pmerr{0.02} & 1.98 & \pmerr{0.02} & \textbf{2.05} & \bpmerr{0.02} & \textbf{2.11} & \bpmerr{0.01} \\
RMSE $\downarrow$ & 0.05 & \pmerr{0.00} & 0.05 & \pmerr{0.00} & \textbf{0.04} & \bpmerr{0.00} & 0.05 & \pmerr{0.00} & \textbf{0.04} & \bpmerr{0.00} & \textbf{0.04} & \bpmerr{0.00} & \textbf{0.03} & \bpmerr{0.01} \\
\bottomrule
\end{tabular}}}
\vspace{0.45em}
\includegraphics[width=\textwidth, trim=0 6 0 0, clip]{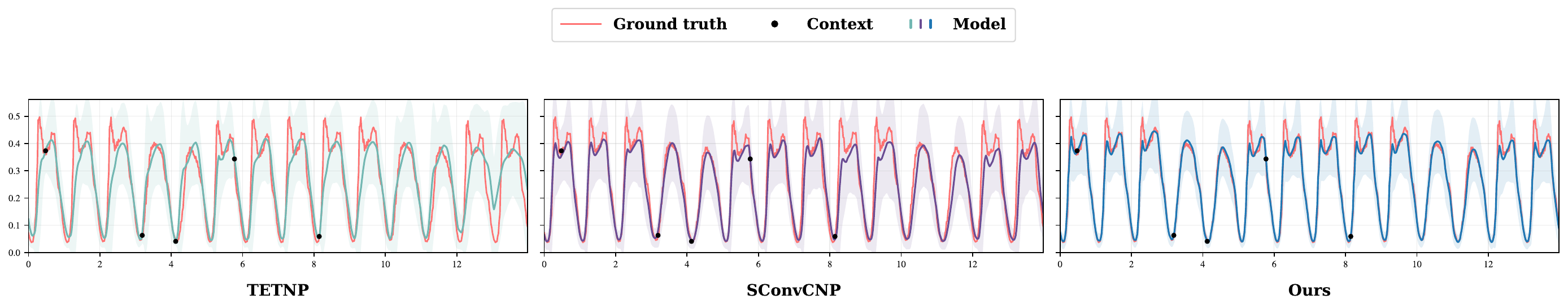}
\vspace{-1.0em}
\begingroup
\setlength{\abovecaptionskip}{0pt}
\captionof{figure}{Representative California traffic flow imputation episodes. The red
curve denotes the ground truth, the coloured curve denotes the predictive mean,
and the band denotes uncertainty.}
\label{fig:traffic-hard-gap-cases}
\endgroup
\end{figure}

\textbf{Chichester Harbour.}\label{sec:climate-data} We evaluate short-horizon
forecasting \cite{mussati2024neural} on Chimet station observations from
Chichester Harbour, UK. The data cover every June from 2009 to 2025, contain
144,474 observations, and record five weather attributes: wind speed, air
temperature, tidal depth, average wave height, and barometric pressure. We treat
each June as an independent seasonal realisation and split the data
chronologically into training (2009--2020), validation (2021--2022), and testing
(2023--2025) sets. Given a 48-hour context, models forecast all five variables
over the next 12 hours and are trained with an \(L^2\) loss. We retain the
Gaussian output parameterisation and visualise the output-head scale in
Figure~\ref{fig:chimet-depth-avwht-forecast} as predictive bands; because the
training objective is \(L^2\) rather than a proper probabilistic scoring rule,
these bands are interpreted qualitatively rather than as calibrated
uncertainties. Table~\ref{tab:chimet-forecast-results} and
Figure~\ref{fig:chimet-depth-avwht-forecast} show that STNP transfers recurring
patterns while adapting to local context and produces more informative
predictive bands than TETNP and SConvCNP.

\vspace*{\fill}
\begin{figure}[H]
\centering
\captionof{table}{Chimet climate forecasting performance on the test set. Lower is
better.}
\label{tab:chimet-forecast-results}
\vspace{-0.15em}
{\scriptsize
\setlength{\tabcolsep}{2.8pt}
\renewcommand{\arraystretch}{0.92}
\resizebox{\textwidth}{!}{%
\begin{tabular}{lccccccccccc}
\toprule
Metric & NP & CNP & ANP & CANP & BNP & BANP & ConvCNP & SConvCNP & TNP & TETNP & \textbf{Ours} \\
\midrule
RMSE & 3.217 & 3.165 & 1.887 & 1.900 & 3.154 & 1.990 & 1.809 & 1.713 & 2.204 & 1.725 & \textbf{1.469} \\
MAE & 2.029 & 2.002 & 1.441 & 1.485 & 2.001 & 1.530 & 1.367 & 1.295 & 1.729 & 1.299 & \textbf{1.098} \\
\bottomrule
\end{tabular}}}

\vspace{0.25em}

\includegraphics[width=\textwidth]{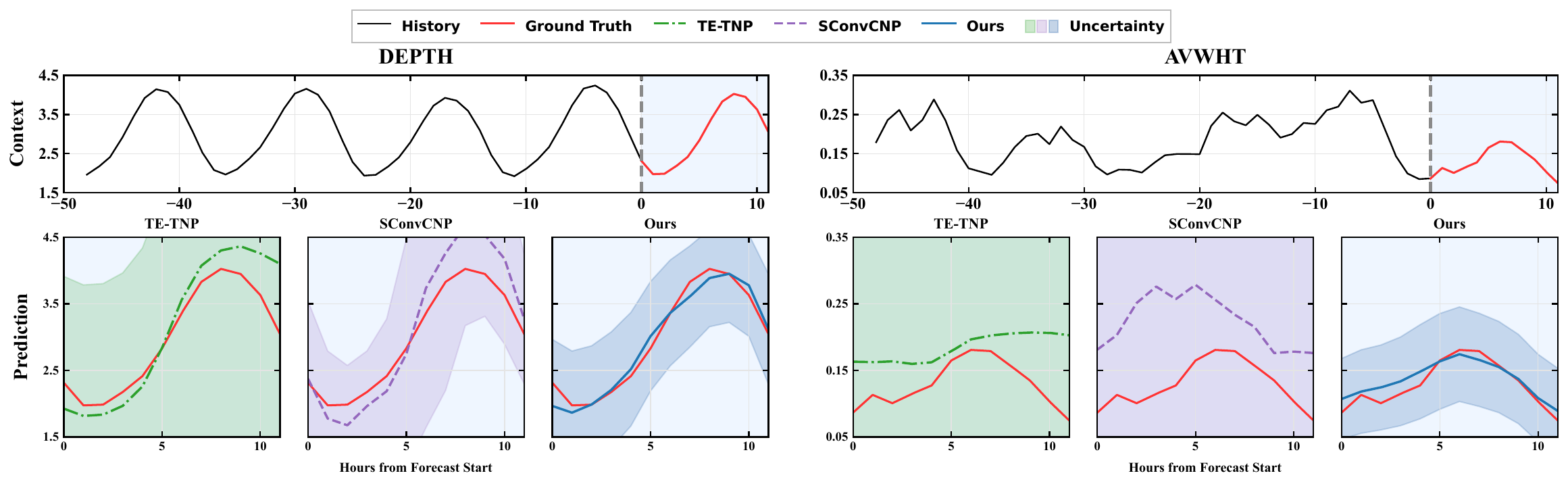}
\vspace{-0.5em}
\captionof{figure}{Representative short-horizon forecasts on the Chimet climate
dataset for tidal depth and average wave height, comparing TETNP, SConvCNP, and
STNP. Red curves denote ground truth, black curves denote context observations,
coloured curves denote model predictions, and shaded regions show qualitative
predictive bands derived from Gaussian output heads.}
\label{fig:chimet-depth-avwht-forecast}
\end{figure}

\clearpage

\subsection{Ablation Studies}
\label{sec:ablation-study}
\label{sec:further-discussion}

\begin{table}[!t]
\centering
\caption{Targeted ablation studies. The periodic-kernel GP
columns report target-point log-likelihoods. The
ETTh1 columns report MAE and MSE.}
\label{tab:ablation-main}
\resizebox{\textwidth}{!}{%
\begin{tabular}{llS[table-format=-1.2]S[table-format=-1.2]S[table-format=-1.2]@{\hspace{1.2em}}S[table-format=1.3]S[table-format=1.3]}
\toprule
& & \multicolumn{3}{c}{Periodic-kernel Gaussian process} & \multicolumn{2}{c}{ETTh1} \\
\cmidrule(lr){3-5}\cmidrule(lr){6-7}
Category & Variant & \multicolumn{1}{c}{\(m_{\min}=10\)} & \multicolumn{1}{c}{\(m_{\min}=15\)} & \multicolumn{1}{c}{\(m_{\min}=20\)} & \multicolumn{1}{c}{MAE} & \multicolumn{1}{c}{MSE} \\
\midrule
Full model & STNP (ours) & 0.63 & 1.02 & 1.14 & 0.487 & 0.539 \\
\midrule
\multirow{3}{*}{Embedding parameterisation}
& \([\cos(W_p x)\mid\sin(W_p x)\mid\phi_{\mathrm{mlp}}(x,y)]\) & -0.36 & -0.35 & -0.35 & 0.512 & 0.592 \\
& \([\phi_{\mathrm{disc}}(x)\mid\phi_{\mathrm{mlp}}(x,y)]\) & 0.10 & 0.21 & 0.29 & 0.513 & 0.595 \\
& \([\phi_{\mathrm{rff}}(x)\mid\phi_{\mathrm{mlp}}(x,y)]\) & 0.46 & 0.60 & 0.71 & 0.513 & 0.593 \\
\midrule
\multirow{2}{*}{Responsibility network}
& CNN (ours) & 0.63 & 1.02 & 1.14 & 0.487 & 0.539 \\
& FFN & 0.05 & 0.67 & 0.81 & 0.495 & 0.541 \\
\bottomrule
\end{tabular}}

\end{table}
\textbf{Spectral features and decomposition.}
We conduct targeted ablations to isolate the contribution of the spectral branch. 
Table~\ref{tab:ablation-main} compares the full STNP with three embedding
variants and two responsibility-network choices on the periodic-kernel GP
benchmark and ETTh1. The FAN-style variant performs poorly on periodic regression,
showing that the gains do not come from adding sinusoidal features alone.
Directly using the empirical discrete spectrum, \(\phi_{\rm disc}\), improves
over fixed Fourier features but remains far below the full model, indicating
that raw spectral estimation is insufficient. Sampling frequencies from this
empirical spectrum, \(\phi_{\rm rff}\), further improves over
\(\phi_{\rm disc}\). However, it still underperforms STNP, which
shows that the learned responsibility-based decomposition and continuous mixture
compression are essential. The responsibility-network ablation further supports this conclusion. Replacing
the CNN with an FFN degrades performance, especially on sparse periodic
regression, suggesting that local interactions along the frequency axis are
important for robust spectral decomposition. On ETTh1, the differences are
smaller but the full STNP remains the best overall, indicating that the proposed
spectral branch is most beneficial when frequency estimation is difficult and
context observations are sparse. Additional ablation studies are provided in
Appendix~\ref{app:ablation}.

\begin{table}[!b]
\centering
\small
\caption{TNP capacity diagnostic on periodic regression.}
\label{tab:tnp-capacity-main}
\begin{tabular*}{\textwidth}{@{\extracolsep{\fill}}lrrrr@{}}
\toprule
Parameters (M) & 0.25 & 0.50 & 1.00 & 2.00 \\
\midrule
Target LL $\uparrow$ & -0.34772 & -0.34767 & -0.34767 & -0.34770 \\
\bottomrule
\end{tabular*}

\medskip
\caption{Training and inference costs of STNP relative to TNP.}
\label{tab:runtime-main}
\begin{tabular*}{\textwidth}{@{\extracolsep{\fill}}lcccc@{}}
\toprule
 & \multicolumn{2}{c}{Wall-clock time} & \multicolumn{2}{c}{Peak memory} \\
\cmidrule(lr){2-3}\cmidrule(lr){4-5}
Benchmark & Training & Inference & Training & Inference \\
\midrule
Synthetic GP ($N=50$) & $1.183\times$ & $1.437\times$ & $1.171\times$ & $1.034\times$ \\
DTD ($N=4096$) & $1.072\times$ & $1.114\times$ & $1.146\times$ & $1.004\times$ \\
\bottomrule
\end{tabular*}
\end{table}

\textbf{Model capacity.}
To test whether TNP's poor performance on periodic regression is caused by
insufficient model capacity, we evaluate TNP models with parameter counts
ranging from 0.25M to 2M. Table~\ref{tab:tnp-capacity-main} shows that target
log-likelihood remains near $-0.348$ across these model sizes, indicating that
increasing capacity alone does not improve performance on this task. STNP achieves
$1.12\pm0.02$ with 0.233M parameters (Table~\ref{tab:regression_results}),
highlighting the importance of the spectral representation.
Figure~\ref{fig:periodic-attention-recovery} shows STNP attending to
phase-aligned context observations, while TNP concentrates on a few context
points across target phases.

\textbf{Computational overhead.}
To compare the computational costs of STNP and TNP in practice, we measure
wall-clock time and peak GPU memory during training and inference. Profiling
is performed on a single NVIDA 2080 Ti GPU. The GP profile uses the
synthetic-regression setup in Section~\ref{sec:synthetic-regression}
(Appendix~\ref{app:synthetic-regression-details}), with \(N=50\) one-dimensional
points per task. For the DTD image-completion task, each image contains
\(N=4096\) pixels. We profile this task with 512 context pixels and
3584 target pixels.
Table~\ref{tab:runtime-main} reports the results. Relative to TNP,
STNP increases training and inference
time by 18.3\% and 43.7\% on GP regression, and by 7.2\% and 11.4\% on image
completion. Peak memory increases by at most 17.1\%. For fixed architecture
and spectral hyperparameters, Appendix~\ref{app:complexity} derives an
\(\mathcal O(N)\) embedding cost and an \(\mathcal O(N^2)\) total forward cost.
These measurements are consistent with the analysis and indicate that the
spectral branch adds modest runtime and memory overhead. As the number of
points increases, the TNP backbone's quadratic attention computation dominates
the overall computational cost.

\section{Related Work}

CNPs \cite{Garnelo2018ConditionalNP} learn a permutation-invariant
representation of the context set and decode target-wise predictive
distributions from this shared deterministic summary. Latent NPs
\cite{Garnelo2018NeuralP} introduce a global latent variable, allowing
uncertainty over functions to be shared across target predictions.
ANPs \cite{Kim2019AttentiveNP} alleviate the underfitting of mean aggregation
by using cross-attention to give each target a representation weighted toward
relevant context observations. TNPs \cite{Nguyen2022TransformerNP} cast
conditional prediction as sequence modelling with attention masks tailored to
context and target roles. Their independent Gaussian, correlated Gaussian,
and autoregressive variants offer different trade-offs between predictive
expressivity and computational cost.

Several extensions further refine uncertainty modelling. Bootstrapping NPs
\cite{Lee2020BootstrappingNP} use bootstrap resampling of context sets and combine the resulting
predictions to represent functional uncertainty, whereas Gaussian NPs
\cite{Bruinsma2021TheGN} directly parameterise Gaussian predictive processes to capture
correlations across targets with a tractable likelihood. RNPs \cite{Wang2024RnyiNP}
use a R\'enyi-divergence objective to mitigate the effect of misspecified
conditional priors.
Other extensions adapt NPs to structured observations: Sequential NPs
\cite{Singh2019SequentialNP} introduce temporal state transitions to model
evolving sequences of stochastic processes, while Gridded TNPs \cite{Ashman2025GriddedTN} encode
irregular observations into gridded pseudo-tokens, enabling efficient attention
and prediction at arbitrary spatiotemporal locations.

Complementing these advances in aggregation, uncertainty, and data handling,
another line of work encodes spatial symmetries and similarity priors. ConvCNPs
\cite{Gordon2019ConvolutionalCN} use convolutions on context functions to
encode translation equivariance, while TETNPs
\cite{Ashman2024TranslationET} preserve it in the predictive map using relative
input offsets. Relational CNPs \cite{Huang2023RelationalCNP} impose exact
equivariances via invariant pairwise comparisons, avoiding grids in higher
dimensions. Approximately equivariant NPs
\cite{Ashman2024ApproximatelyEquivariantNP} use learnable fixed inputs to
model departures from exact symmetry. DINPs
\cite{Venkataramanan2025DistanceinformedNP} encode locality using
distance-preserving embeddings. Biased Scan Attention TNPs
\cite{Jenson2025ScalableSI} bias attention with Gaussian kernels.

For spectral modelling, SConvCNPs \cite{Mohseni2024SpectralCC} use Fourier
neural operators for global grid interactions. MetaVRF
\cite{Zhen2020VariationalRandomFeatures} treats random Fourier bases as
latent variables and learns task-adaptive spectral distributions through
variational meta-learning for kernel ridge regression. STNPs instead summarise
the empirical context spectrum as a compact mixture and sample features to
augment TNP embeddings. These features provide task-adaptive periodic and
quasi-periodic biases alongside flexible learned attention.
Appendix~\ref{app:extended-related-work} provides further discussion.

\section{Conclusion}
We introduced Spectral Transformer Neural Processes (STNPs), which augment
TNP embeddings with task-adaptive spectral features for periodic and
quasi-periodic modelling. A Spectral Aggregator estimates the context spectrum,
compresses it into a spectral mixture, and samples compact features that
complement time-domain representations. The raw features induce a conditional
spectral mixture kernel in expectation, providing an interpretable inductive
bias alongside flexible learned attention. Experiments on synthetic regression,
texture image completion, and real-world forecasting and imputation demonstrate
improved predictive performance, while ablations support the contribution of
learned spectral decomposition, compression, and sampling.

\noindent\textbf{Limitations and Future Work.}
The current TNP-Diagonal output head assumes conditionally independent
Gaussian predictions across target
points and output channels. Although the spectral representation captures
recurring structure, this factorised likelihood does not explicitly represent
residual dependencies among outputs or non-Gaussian predictive uncertainty.
A main direction for future work is to combine the proposed spectral
representation with correlated or latent-variable predictive models and more
expressive likelihoods, enabling joint uncertainty modelling across target
locations and channels, as well as multimodal predictions.

\clearpage
\nocite{*}
\begingroup
\small
\bibliographystyle{plainnat}
\bibliography{example_paper}
\endgroup

\clearpage

\appendix

\section{Technical appendices and supplementary material}

\subsection{Symbols}
\label{app:symbols}

Table~\ref{tab:key-symbols} summarises the main symbols used throughout the
paper and appendix. Table~\ref{tab:training-step-complexity} summarises the
asymptotic training-step complexity of the compared methods.

\begin{table}[H]
\centering
\small
\renewcommand{\arraystretch}{1.12}
\caption{Key symbols used in the main text and appendix.}
\label{tab:key-symbols}
\begin{tabular}{>{\centering\arraybackslash}m{0.22\textwidth}>{\raggedright\arraybackslash}m{0.72\textwidth}}
\toprule
\textbf{Symbol} & \textbf{Meaning} \\
\midrule
\(\mathcal{X}, \mathcal{Y}\) & Input and output spaces. \\
\(f \sim p(f)\) & Task/function sampled from the meta-distribution. \\
\(\mathcal{D}, \mathcal{D}_C, \mathcal{D}_T\) & Full episode, context set, and target set. \\
\(N, M\) & Total number of points in an episode and number of context points. \\
\(\tau\) & Token sequence constructed from context and target pairs for the TNP encoder. \\
\(K\) & Number of fixed grid frequencies used to estimate the empirical spectrum. \\
\(\tilde{\omega}_k\) or \(\tilde{\boldsymbol\omega}_k\) & \(k\)-th scalar or vector frequency on the fixed spectral grid. \\
\(E_k, p_k\) & Empirical spectral energy at frequency \(\tilde{\omega}_k\) and the corresponding normalised discrete spectral mass. \\
\(\mathcal P, d_p\) & Selected input-coordinate subset used by the spectral branch and its dimension \(d_p=|\mathcal P|\). \\
\(\mathcal A\) & Selected output-channel subset used for spectralisation in multi-output settings. \\
\(E_{k,c}, E_k^{\mathcal A}\) & Channel-wise stabilised spectral energy and aggregated spectral energy over selected output channels. \\
\(Q\) & Number of latent spectral mixture components. \\
\(r_{qk}\) & Responsibility of component \(q\) for grid frequency \(\tilde{\omega}_k\) or \(\tilde{\boldsymbol\omega}_k\), i.e. \(\mathbb P_\phi(Z=q\mid \Omega_{\mathcal{D}_C}=\tilde{\omega}_k)\) in the scalar case. \\
\((w_q,\mu_q,\sigma_q^2)\) or \((w_q,\boldsymbol\mu_q,\mathbf\Sigma_q)\) & Weight, mean frequency, and variance or covariance of the \(q\)-th Gaussian mixture component. \\
\(D_0\) & Number of sampled frequencies drawn from each mixture component. \\
\(\omega_{q,d}\) or \(\boldsymbol\omega_{q,d}\) & \(d\)-th sampled scalar or vector frequency from mixture component \(q\). \\
\(\phi_q\) & Optional phase associated with the \(q\)-th spectral component. \\
\(\phi_{\mathrm{S}}(x)\) or \(\phi_{\mathrm{S}}(\mathbf x)\) & Context-conditioned spectral embedding evaluated at an input. \\
\(\phi_{\mathrm{S}}^{(c)}(\mathbf x)\) & Channel-wise spectral feature inferred from selected output channel \(c\). \\
\(e(x,y)\) & Final token embedding obtained by concatenating spectral and MLP features and projecting them to the model dimension. \\
\(\bm{\mu}_i,\bm{\sigma}_i,\diag(\bm{\sigma}_i^2)\) & Per-target predictive mean, standard deviation, and diagonal covariance in the Gaussian output distribution. \\
\bottomrule
\end{tabular}

\vspace{0.9em}
\caption{Training-step complexity per episode for the compared neural-process
methods. Here \(N\) denotes the total number of context and target points, \(B\)
the number of bootstrap or latent samples, \(G\) the number of grid points, and
\(S(G)\) the spectral convolution cost on the grid. For STNP, the spectral-grid
quantities \(K,Q,D_0,d_p\), and \(|\mathcal A|\) are treated as fixed
hyperparameters.}
\label{tab:training-step-complexity}
\begin{tabular}{lc}
\toprule
\textbf{Method} & \textbf{Training-step complexity per episode} \\
\midrule
NP & \(\mathcal O(N)\) \\
CNP & \(\mathcal O(N)\) \\
ANP & \(\mathcal O(N^2)\) \\
CANP & \(\mathcal O(N^2)\) \\
BNP & \(\mathcal O(BN)\) \\
BANP & \(\mathcal O(BN^2)\) \\
ConvCNP & \(\mathcal O(GN)\) \\
SConvCNP & \(\mathcal O(GN + S(G))\) \\
TNP & \(\mathcal O(N^2)\) \\
TETNP & \(\mathcal O(N^2)\) \\
STNP (ours) & \(\mathcal O(N^2)\) \\
\bottomrule
\end{tabular}
\end{table}

\subsection{Optional phase estimation}
\label{app:optional-phase-estimation}

When phase shifts are enabled, we estimate a component-wise phase from the
normalised complex spectral direction of a scalar spectral branch. The definition
is unchanged for multi-dimensional inputs because each frequency-location
projection remains scalar:
\(\tilde{\boldsymbol\omega}_k^\top \mathbf x_i^{\mathcal P}\).
For a branch using the selected input coordinates \(\mathcal P\), let
\(\{\tilde{\boldsymbol\omega}_k\}_{k=1}^K\subset\mathbb R^{d_p}\) be its
frequency grid and let \(a_k,b_k,E_k\) be the stabilised cosine-sine statistics
defined from that branch. We set
\[
u_k=\frac{a_k + i b_k}{\sqrt{E_k}},
\]
and define the component-wise weighted circular mean
\[
z_q
=
\sum_{k=1}^{K}
\mathbb{P}\!\left(
\Omega_{\mathcal{D}_C}=\tilde{\boldsymbol\omega}_k \mid Z=q
\right)u_k,
\qquad q=1,\dots,Q.
\]
The phase is then
\[
\phi_q = \arg(z_q).
\]
The sampled vector frequencies are subsequently evaluated as
\[
\cos\!\bigl(\boldsymbol\omega_{q,d}^{\top}\mathbf x^{\mathcal P}-\phi_q\bigr),
\qquad
\sin\!\bigl(\boldsymbol\omega_{q,d}^{\top}\mathbf x^{\mathcal P}-\phi_q\bigr).
\]

For channel-wise selected spectra in multi-output settings, the same formula is
applied independently to each selected output channel \(c\in\mathcal A\), using
\(a_{k,c}\), \(b_{k,c}\), and \(E_{k,c}\), and producing phases
\(\phi_q^{(c)}\) for the corresponding channel-wise feature
\(\phi_S^{(c)}(\mathbf x)\). A shared selected-channel energy
\(E_k^{\mathcal A}\) does not by itself determine a unique complex phase
direction; if no scalar or projected response direction is specified for such a
shared branch, we disable phase shifts and set \(\phi_q=0\).

\subsection{Multi-dimensional extension}
\label{app:multidim}
\paragraph{Experimental instantiations.}
The experiments instantiate the cases developed below as follows.
\begin{enumerate}[leftmargin=*,itemsep=0.15em,topsep=0.2em]
\item \textbf{Synthetic regression and California traffic flow.}
These are scalar-input, scalar-output settings with \(d_x=d_y=1\) and
\(\mathcal P=\{1\}\).

\item \textbf{DTD image completion.}
The spectral branch uses the full two-dimensional pixel coordinate,
\(\mathcal P=\{1,2\}\), and the shared selected-channel spectrum with all RGB
channels, \(\mathcal A=\{1,2,3\}\).

\item \textbf{Six-dataset forecasting benchmark.}
The spectral branch is applied to the temporal coordinate, so
\(\mathcal P=\{1\}\). For multivariate datasets, the spectral summary uses the
shared all-channel aggregation, \(\mathcal A=\{1,\dots,d_y\}\).

\item \textbf{Chimet.}
The input side also uses the temporal coordinate, so \(\mathcal P=\{1\}\). The
prediction head outputs all five variables
\(\{\mathrm{WSPD},\mathrm{ATMP},\mathrm{DEPTH},\mathrm{AVWHT},\mathrm{BARO}\}\),
but the spectral branch uses channel-wise selected spectra with
\(\mathcal A=\{\mathrm{ATMP},\mathrm{DEPTH}\}\). Specifically, one scalar
spectral branch is built from ATMP and one from DEPTH, using the period ranges
reported in Table~\ref{tab:chimet-ours-hyperparameters}; their spectral
features are concatenated with the time-domain embeddings before the prediction
head outputs all five variables. Thus, the Chimet setting uses selected output
channels for spectralisation while retaining a five-dimensional forecasting
target.
\end{enumerate}

\paragraph{Core construction for multi-dimensional inputs.}
The one-dimensional construction extends to multi-dimensional inputs by replacing
scalar frequencies with frequency vectors and scalar Gaussian components with
multivariate Gaussian components. We first present the scalar-output
construction. Let \(\mathcal P\subseteq\{1,\dots,d_x\}\) denote the input
coordinates used by the spectral branch, let \(d_p:=|\mathcal P|\), and write
\(\mathbf x_i^{\mathcal P}\in\mathbb R^{d_p}\) for the corresponding subvector.
The full-input and input-subset choices for \(\mathcal P\) are discussed after
the core construction.

\paragraph{Step 1: Empirical discrete spectral distribution in \(\mathbb R^{d_p}\).}
Let \(\{\tilde{\boldsymbol\omega}_k\}_{k=1}^K \subset \mathbb R^{d_p}\) be a fixed
frequency grid, and let
\[
\mathcal{D}_C = \{(\mathbf x_i, y_i)\}_{i=1}^M
\]
denote the context set. For scalar outputs, we first centre the responses:
\[
y_i^{(s)} = y_i - \bar y,
\qquad
\bar y = \frac{1}{M}\sum_{i=1}^M y_i.
\]
We then estimate spectral energy on the fixed frequency grid by projecting the
centred context responses onto the cosine--sine basis associated with each
frequency vector \(\tilde{\boldsymbol\omega}_k\):
\[
a_k
=
\frac{1}{M}\sum_{i=1}^M y_i^{(s)}
\cos\!\bigl(\tilde{\boldsymbol\omega}_k^\top \mathbf x_i^{\mathcal P}\bigr),
\qquad
b_k
=
\frac{1}{M}\sum_{i=1}^M y_i^{(s)}
\sin\!\bigl(\tilde{\boldsymbol\omega}_k^\top \mathbf x_i^{\mathcal P}\bigr),
\]
\[
E_k = a_k^2 + b_k^2 + \varepsilon.
\]
The small constant \(\varepsilon>0\) stabilises the logarithms, square-root
normalisations, and discrete probabilities below; we use the same stabilised
energies in the experiments. We normalise these energies into a discrete
distribution over the frequency grid:
\[
p_k = \frac{E_k}{\sum_{j=1}^K E_j},
\qquad
\sum_{k=1}^K p_k = 1.
\]
Equivalently, this induces a discrete random variable
\[
\Omega_{\mathcal{D}_C} \in
\{\tilde{\boldsymbol\omega}_1,\dots,\tilde{\boldsymbol\omega}_K\}
\]
with
\[
\mathbb P(\Omega_{\mathcal{D}_C} = \tilde{\boldsymbol\omega}_k) = p_k,
\]
and the associated empirical discrete spectral measure
\[
\hat{\nu}_{\mathcal{D}_C}
=
\sum_{k=1}^K p_k\,\delta_{\tilde{\boldsymbol\omega}_k}.
\]

\paragraph{Step 2: Latent spectral decomposition with a responsibility network.}
As in the one-dimensional case, we introduce a latent mixture variable
\[
Z \in \{1,\dots,Q\},
\]
which indicates the spectral component to which a grid frequency belongs. For
each grid point \(k\), we form the spectral summary
\[
\mathbf S_k
=
\left[
\log E_k,\;
\frac{a_k}{\sqrt{E_k}},\;
\frac{b_k}{\sqrt{E_k}},\;
\frac{\tilde{\boldsymbol\omega}_k}{\omega_{\max}}
\right]
\in \mathbb R^{3+d_p},
\qquad
\omega_{\max} := \max_{1\le j\le K}\|\tilde{\boldsymbol\omega}_j\|_2 .
\]
Collecting all grid points yields a sequence or tensor
\[
\{\mathbf S_k\}_{k=1}^K.
\]
These summaries are processed by a lightweight responsibility network
\(g_\phi\), which outputs component logits \(\ell_{qk}\) and the corresponding
responsibilities
\[
r_{qk}
=
\mathbb P_\phi\!\left(Z=q \mid \Omega_{\mathcal{D}_C}=\tilde{\boldsymbol\omega}_k\right)
=
\frac{\exp(\ell_{qk})}{\sum_{q'=1}^Q \exp(\ell_{q'k})},
\qquad
\sum_{q=1}^Q r_{qk} = 1.
\]
When the frequency grid admits a tensor-product structure, \(g_\phi\) may be
implemented as a \(d_p\)-dimensional convolutional network over the grid. In
practice, one may also flatten the grid using a fixed ordering and apply a
one-dimensional convolutional or sequence model. In either case, the purpose of
\(g_\phi\) is to encourage local continuity and neighbourhood-aware decomposition
across nearby frequencies, rather than treating grid points independently.

\paragraph{Step 3: Probabilistic spectral compression.}
Given the discrete spectral distribution \(p_k\) and the responsibilities
\(r_{qk}\), the induced joint distribution and mixture weights are
\[
\mathbb P\!\left(Z=q,\Omega_{\mathcal{D}_C}=\tilde{\boldsymbol\omega}_k\right)
=
r_{qk}p_k,
\qquad
w_q = \mathbb P(Z=q) = \sum_{k=1}^K r_{qk}p_k.
\]
This gives the component-wise conditional distribution over the frequency grid:
\[
\mathbb P\!\left(
\Omega_{\mathcal{D}_C}=\tilde{\boldsymbol\omega}_k \mid Z=q
\right)
=
\frac{r_{qk}p_k}{w_q}.
\]
We then define the component moments in \(\mathbb R^{d_p}\):
\[
\boldsymbol\mu_q
=
\mathbb E[\Omega_{\mathcal{D}_C} \mid Z=q]
=
\frac{\sum_{k=1}^K r_{qk}p_k\,\tilde{\boldsymbol\omega}_k}{w_q}
\in \mathbb R^{d_p},
\]
\[
\mathbf\Sigma_q
=
\mathrm{Cov}(\Omega_{\mathcal{D}_C} \mid Z=q)
=
\frac{
\sum_{k=1}^K r_{qk}p_k
(\tilde{\boldsymbol\omega}_k-\boldsymbol\mu_q)
(\tilde{\boldsymbol\omega}_k-\boldsymbol\mu_q)^\top
}{w_q}
\in \mathbb R^{d_p\times d_p}.
\]
This yields a compact Gaussian-mixture approximation to the empirical discrete
spectrum:
\[
g_{\theta_{\mathcal{D}_C}}(\boldsymbol\omega)
=
\sum_{q=1}^Q
w_q\,\mathcal N(\boldsymbol\omega;\boldsymbol\mu_q,\mathbf\Sigma_q),
\qquad
\theta_{\mathcal{D}_C}
=
\{(w_q,\boldsymbol\mu_q,\mathbf\Sigma_q)\}_{q=1}^Q.
\]
Equivalently, this defines a continuous Gaussian-mixture measure
\[
\nu_{\theta_{\mathcal{D}_C}}(d\boldsymbol\omega)
=
g_{\theta_{\mathcal{D}_C}}(\boldsymbol\omega)\,d\boldsymbol\omega,
\]
which serves as a low-dimensional continuous approximation to the discrete
spectral measure \(\hat{\nu}_{\mathcal{D}_C}\).

For numerical stability, one may optionally use an isotropic approximation
\[
\mathbf\Sigma_q \approx \sigma_q^2 \mathbf I_{d_p},
\qquad
\sigma_q^2 := \frac{1}{d_p}\mathrm{tr}(\mathbf\Sigma_q),
\]
or a diagonal approximation
\[
\mathbf\Sigma_q \approx \mathrm{diag}(\boldsymbol\sigma_q^2).
\]

\paragraph{Optional phase estimation.}
Phase estimation extends directly to the multi-dimensional setting because
\(\tilde{\boldsymbol\omega}_k^\top \mathbf x^{\mathcal P}\) remains scalar. Define the
normalised complex spectral coefficient
\[
u_k := \frac{a_k + i b_k}{\sqrt{E_k}}.
\]
Then, for each component \(q\), define the weighted circular mean
\[
z_q
=
\sum_{k=1}^K
\mathbb P\!\left(
\Omega_{\mathcal{D}_C}=\tilde{\boldsymbol\omega}_k \mid Z=q
\right) u_k.
\]
The corresponding component-wise phase is
\[
\phi_q = \arg(z_q),
\qquad
q=1,\dots,Q.
\]
This allows the learned spectral basis to capture task-level phase offsets in
addition to dominant frequencies and bandwidths. For multi-output settings, the
same definition applies to each scalar or channel-wise spectral branch, as
described in Appendix~\ref{app:optional-phase-estimation}.

\paragraph{Step 4: Sampling frequency vectors and constructing the spectral embedding.}
For each component \(q\), we draw \(D_0\) frequency vectors via
reparameterisation:
\[
\boldsymbol\omega_{q,d}
=
\boldsymbol\mu_q + \mathbf L_q \boldsymbol\epsilon_{q,d},
\qquad
\boldsymbol\epsilon_{q,d}\sim \mathcal N(\mathbf 0,\mathbf I_{d_p}),
\qquad
\mathbf L_q \mathbf L_q^\top = \mathbf\Sigma_q,
\qquad
d=1,\dots,D_0.
\]
Under the isotropic approximation, this reduces to
\[
\boldsymbol\omega_{q,d}
=
\boldsymbol\mu_q + \sigma_q \boldsymbol\epsilon_{q,d},
\qquad
\boldsymbol\epsilon_{q,d}\sim \mathcal N(\mathbf 0,\mathbf I_{d_p}).
\]
For an input \(\mathbf x\), using the coordinates \(\mathbf x^{\mathcal P}\), we
construct the spectral embedding
\[
\phi_S(\mathbf x)
=
\bigoplus_{q=1}^Q \bigoplus_{d=1}^{D_0}
\sqrt{\frac{w_q}{D_0}}
\begin{bmatrix}
\cos\!\bigl(\boldsymbol\omega_{q,d}^\top \mathbf x^{\mathcal P}-\phi_q\bigr)\\[0.2em]
\sin\!\bigl(\boldsymbol\omega_{q,d}^\top \mathbf x^{\mathcal P}-\phi_q\bigr)
\end{bmatrix}
\in \mathbb R^{2QD_0}.
\]
If phases are disabled, we simply set \(\phi_q=0\). As in the one-dimensional
setting, the spectral embedding may be concatenated with a non-spectral branch
that processes the remaining input coordinates or other non-periodic structure.
When \(\mathcal P=\{1,\dots,d_x\}\), this embedding is defined on the full input.

\paragraph{Input-coordinate cases.}
The input side has two cases. In the full-input case,
\(\mathcal P=\{1,\dots,d_x\}\), \(d_p=d_x\), and
\(\mathbf x_i^{\mathcal P}=\mathbf x_i\), so the spectral branch operates on all
input coordinates. In the input-subset case,
\(\mathcal P\subsetneq\{1,\dots,d_x\}\), and the spectral branch is restricted
to coordinates known to carry the dominant oscillatory structure. This choice
affects only the argument \(\mathbf x^{\mathcal P}\) used in the spectral
features; the remaining input coordinates can still be processed by the
time-domain embedding.

\paragraph{Output-channel cases.}
When \(y_i\in\mathbb R^{d_y}\), the output dimension of the prediction head need
not match the channels used to estimate the spectrum. We distinguish the
following cases.

\textbf{Scalar output.}
If \(d_y=1\), the scalar-output construction above applies directly.

\textbf{Projected output.}
One can project the centred responses to a scalar channel,
\[
\tilde y_i^{(s)} = \mathbf u^\top (y_i-\bar{\mathbf y}),
\qquad
\|\mathbf u\|_2=1,
\]
and then reuse the scalar-output construction above.

\textbf{Shared selected-channel spectrum.}
More generally, let
\(\mathcal A\subseteq\{1,\dots,d_y\}\) denote the subset of output channels used
by the spectral branch. This is distinct from the input-coordinate subset
\(\mathcal P\). For each selected channel \(c\in\mathcal A\), compute
\[
a_{k,c}
=
\frac{1}{M}\sum_{i=1}^M (y_{i,c}-\bar y_c)
\cos\!\bigl(\tilde{\boldsymbol\omega}_k^\top \mathbf x_i^{\mathcal P}\bigr),
\qquad
b_{k,c}
=
\frac{1}{M}\sum_{i=1}^M (y_{i,c}-\bar y_c)
\sin\!\bigl(\tilde{\boldsymbol\omega}_k^\top \mathbf x_i^{\mathcal P}\bigr),
\]
and the stabilised channel energy
\[
E_{k,c}=a_{k,c}^2+b_{k,c}^2+\varepsilon.
\]
If the selected channels are used to form one shared spectrum, we aggregate
their energy as
\[
E_k^{\mathcal A}=\sum_{c\in\mathcal A} E_{k,c},
\qquad
p_k^{\mathcal A}
=
\frac{E_k^{\mathcal A}}{\sum_{j=1}^{K}E_j^{\mathcal A}}.
\]
The corresponding responsibility-network input can be written as
\[
\mathbf S_k^{\mathcal A}
=
\left[
\log E_k^{\mathcal A},\;
\left\{\frac{a_{k,c}}{\sqrt{E_k^{\mathcal A}}},
\frac{b_{k,c}}{\sqrt{E_k^{\mathcal A}}}\right\}_{c\in\mathcal A},\;
\frac{\tilde{\boldsymbol\omega}_k}{\omega_{\max}}
\right]
\in\mathbb R^{1+2|\mathcal A|+d_p}.
\]
Taking \(\mathcal A=\{1,\dots,d_y\}\) recovers all-channel aggregation.

\textbf{Channel-wise selected spectra.}
When different selected output channels have distinct period ranges, one may
instead run a separate scalar spectral branch for each
\(c\in\mathcal A\), producing channel-wise parameters
\(\theta_{\mathcal{D}_C}^{(c)}\) and features
\(\phi_S^{(c)}(\mathbf x)\). Each branch may use its own frequency grid and
hyperparameters, such as \(K_c\), \(Q_c\), \(D_{0,c}\), and period range. The
final spectral feature is the concatenation
\[
\phi_S^{\mathcal A}(\mathbf x)
=
\bigoplus_{c\in\mathcal A}\phi_S^{(c)}(\mathbf x).
\]
In both cases, \(\mathcal A\) only controls how the spectral summary is inferred;
the prediction head can still output all \(d_y\) channels.

\paragraph{Interpretation.}
The multi-dimensional extension therefore does not require the periodic
coordinates to be known a priori. In its most general form, the spectral branch
operates on the full input and should be viewed as a task-adaptive harmonic
feature extractor. Restricting the construction to a subset
\(\mathcal P\) is optional and serves primarily as a structured, lower-dimensional,
and more interpretable special case when prior knowledge indicates that only part
of the input carries dominant oscillatory structure.

\subsection{Complexity Analysis}
\label{app:complexity}

We analyse the evaluation-time complexity of our model in terms of the total
number of points \(N\) and the number of context points \(M\), so that the number
of target points is \(N-M\). Our method does not modify the TNP backbone itself;
instead, it replaces the original input embedding with the proposed
task-adaptive spectral embedding. Therefore, the total forward complexity can be
decomposed into two stages. The first is the construction of the SMK embedding and the second
is forward propagation through the standard TNP backbone, i.e., 
\(
T_{\mathrm{ours}}(N,M)
\in
\mathcal O\!\Bigl(
T_{\mathrm{embed}}^{\mathrm{SMK}}(N,M)
+
T_{\mathrm{backbone}}^{\mathrm{TNP}}(N,M)
\Bigr).
\)

The resulting asymptotic training-step complexities are summarised in
Table~\ref{tab:training-step-complexity}.

\paragraph{TNP backbone complexity.}
The proposed embedder first constructs per-token features by concatenating a
spectral branch and a non-spectral branch. For a single shared spectral branch,
the spectral width is \(2QD_0\). For channel-wise selected spectra, it is
\(2\sum_{c\in\mathcal A}Q_cD_{0,c}\), where each selected output channel may use
its own frequency grid and mixture allocation. After concatenation with the
non-spectral branch of width \(h\), the features are projected back to the fixed
model width \(d_{\mathrm{model}}\). Hence the TNP encoder still receives the
same type of input as in the original model, namely a length-\(N\) sequence of
\(d_{\mathrm{model}}\)-dimensional token embeddings. Its dominant cost is the
self-attention computation over all context and target tokens, which scales
quadratically with the sequence length
\(
T_{\mathrm{backbone}}^{\mathrm{TNP}}(N,M)
=
\mathcal O(N^2).
\)

\paragraph{Embedding complexity.}
We further decompose the SMK embedding cost into
\[
T_{\mathrm{embed}}^{\mathrm{SMK}}(N,M)
=
T_{\theta}(M)
+
T_{\phi}(N,M),
\]
where \(T_{\theta}(M)\) is the cost of estimating the global spectral mixture
parameters from the context set only and \(T_{\phi}(N,M)\) is the cost of
evaluating the resulting embedding for all \(N\) tokens.

\paragraph{1. Context-side spectral parameter estimation.}
This stage operates only on the context set and consists of three operations.

\smallskip
\noindent\textbf{(a) Irregular spectral energy estimation.}
Let \(K\) be the number of grid frequencies and let
\(\mathcal P\subseteq\{1,\dots,d_x\}\) be the input-coordinate subset used by
the spectral branch, with \(d_p=|\mathcal P|\). For each of the \(M\) context
points and each of the \(K\) vector frequency grid points, the model computes a
frequency-location projection of the form
\(
\tilde{\boldsymbol\omega}_k^\top \mathbf x_i^{\mathcal P},
\)
followed by cosine/sine evaluations and aggregation into spectral statistics.
Let \(\alpha\) be the number of scalar response channels used to estimate the
spectrum: \(\alpha=1\) for scalar or projected outputs, and
\(\alpha=|\mathcal A|\) for a shared selected-channel spectrum. The
frequency-location projection costs \(\mathcal O(d_p)\), and the response-side
aggregation costs \(\mathcal O(\alpha)\), so this step costs
\(
\mathcal O\!\left(MK(d_p+\alpha)\right).
\)

\smallskip
\noindent\textbf{(b) Frequency-axis convolutional responsibility network.}
The resulting spectral summary is arranged as a length-\(K\) sequence and passed
through a 1D convolutional stack. Let \(C\) be the hidden channel width,
\(\kappa\) the kernel size, \(Q\) the number of spectral mixture components, and
\(L_c\) the number of convolutional layers (counting the final output head). The
input channel dimension is
\[
C_{\mathrm{in}} = 1 + 2\alpha + d_p,
\]
which reduces to \(3+d_p\) for scalar or projected outputs and to
\(1+2|\mathcal A|+d_p\) for a shared selected-channel spectrum.

A standard 1D convolution with sequence length \(K\), kernel width \(\kappa\),
\(C_{\mathrm{in}}\) input channels, and \(C_{\mathrm{out}}\) output channels has
complexity
\(
\mathcal O(K \kappa C_{\mathrm{in}} C_{\mathrm{out}}),
\)
because each of the \(K\) output positions and \(C_{\mathrm{out}}\) output
channels requires a weighted sum over \(\kappa\) neighbouring positions and
\(C_{\mathrm{in}}\) input channels.

Applying this to our convolutional stack:
\begin{itemize}
    \item the first convolution maps the spectral summary, whose channel
    dimension is \(C_{\mathrm{in}}\), to \(C\) hidden channels, contributing
    \[
    \mathcal O(K\kappa C_{\mathrm{in}}C);
    \]
    \item the \(L_c-2\) hidden convolutions each map \(C\) channels to \(C\)
    channels, contributing
    \[
    \mathcal O((L_c-2)K\kappa C^2);
    \]
    \item the final \(1\times 1\) convolutional head maps the \(C\) hidden
    channels to the \(Q\) spectral mixture responsibilities, contributing
    \[
    \mathcal O(K C Q).
    \]
\end{itemize}
Hence the total convolutional cost is
\[
\mathcal O\!\left(
K\bigl(\kappa C_{\mathrm{in}}C + (L_c-2)\kappa C^2 + C Q\bigr)
\right).
\]
Here \(K\) denotes the total number of retained grid frequencies. If a
tensor-product frequency grid is processed by a \(d_p\)-dimensional convolution
rather than by a flattened one-dimensional stack, the same expression applies
after interpreting \(\kappa\) as the convolutional kernel volume.

\smallskip
\noindent\textbf{(c) Moment-based spectral compression.}
Given the discrete spectral weights and the learned responsibilities, the model
computes the mixture parameters
\((w_q,\boldsymbol\mu_q,\mathbf\Sigma_q)\) for \(q=1,\dots,Q\) by weighted
moment matching over the \(K\)-point grid in \(\mathbb R^{d_p}\). Each component
must aggregate over all \(K\) grid frequencies. Computing the mean costs
\(\mathcal O(QKd_p)\), while forming a full covariance costs
\(
\mathcal O(QKd_p^2).
\)
If the isotropic or diagonal covariance approximation from
Appendix~\ref{app:multidim} is used, the covariance term reduces to
\(\mathcal O(QKd_p)\).

\smallskip
\noindent Combining the three parts above,
\[
T_{\theta}(M)
=
\mathcal O\!\left(
M K(d_p+\alpha)
+
K\bigl(\kappa C_{\mathrm{in}}C + (L_c-2)\kappa C^2 + C Q\bigr)
+
Q K d_p^2
\right).
\]
The last term should be read as \(QKd_p\) under an isotropic or diagonal
covariance parameterisation.

\paragraph{2. Token-wise embedding construction.}
Once the spectral parameters are estimated from the context set, they are used to
construct token embeddings for all \(N\) tokens. This stage consists of two
parts.

\smallskip
\noindent\textbf{(a) Spectral features.}
For each of the \(Q\) mixture components, the model samples \(D_0\) vector
frequencies in \(\mathbb R^{d_p}\). Sampling with a full covariance requires
\(\mathcal O(Qd_p^3)\) operations to factorise the component covariances and
\(\mathcal O(QD_0d_p^2)\) operations for the matrix-vector products; this
reduces to \(\mathcal O(QD_0d_p)\) for diagonal or isotropic covariances. Each
token is then evaluated against \(QD_0\) sampled frequencies. Since the
frequency-location projection
\(\boldsymbol\omega_{q,d}^{\top}\mathbf x^{\mathcal P}\) for one
token-frequency pair costs \(\mathcal O(d_p)\), the spectral branch costs
\[
\mathcal O(Qd_p^3 + QD_0d_p^2 + N Q D_0 d_p)
\]
in the full-covariance case. This corresponds to evaluating the Fourier features
\(\cos(\boldsymbol\omega^\top \mathbf x^{\mathcal P})\) and
\(\sin(\boldsymbol\omega^\top \mathbf x^{\mathcal P})\) for every token after
the task-level spectral parameters have been inferred.

\smallskip
\noindent\textbf{(b) Non-spectral branch and final projection.}
The non-spectral MLP is applied independently to each token, and its output is
concatenated with the spectral branch before a final projection back to the fixed
model width \(d_{\mathrm{model}}\). Since the hidden widths and depths of this
branch are treated as fixed hyperparameters, this contributes only linearly in
the number of tokens
\(
\mathcal O(N).
\)
Therefore,
\(
T_{\phi}(N,M)
=
\mathcal O\!\left(
Q d_p^3
+
Q D_0 d_p^2
+
N Q D_0 d_p
+
N
\right).
\)

\paragraph{3. Total embedding complexity.}
By combining the context-side estimation and the token-wise embedding
construction, we obtain
\begin{align*}
T_{\mathrm{embed}}^{\mathrm{SMK}}(N,M)
=
\mathcal O\Bigl(
&
M K(d_p+\alpha)
+
K\bigl(\kappa C_{\mathrm{in}}C + (L_c-2)\kappa C^2 + C Q\bigr)
+
Q K d_p^2
\\
&+
Q d_p^3
+
Q D_0 d_p^2
+
N Q D_0 d_p
+
N
\Bigr).
\end{align*}
For channel-wise selected spectra, such as the Chimet configuration, the same
expression is applied independently to each selected channel
\(c\in\mathcal A\) using that branch's \((K_c,Q_c,D_{0,c})\), and the resulting
spectral costs are summed. The non-spectral branch and final projection are
still evaluated once for the full token sequence.

\paragraph{4. Overall complexity.}
Substituting this into the backbone-plus-embedding decomposition yields
\begin{align*}
T_{\mathrm{ours}}(N,M)
\in
\mathcal O\Bigl(
&
N^2
+
M K(d_p+\alpha)
+
K\bigl(\kappa C_{\mathrm{in}}C + (L_c-2)\kappa C^2 + C Q\bigr)
\\
&+
Q K d_p^2
+
Q d_p^3
+
Q D_0 d_p^2
+
N Q D_0 d_p
+
N
\Bigr).
\end{align*}

When \(K,Q,D_0,d_p,\alpha,\kappa,C\), and \(L_c\) (or the corresponding
channel-wise sets \(\{K_c,Q_c,D_{0,c}\}_{c\in\mathcal A}\)) are treated as
fixed hyperparameters, all terms except those depending on \(N\) and \(M\)
reduce to constants. Thus,
\(
T_{\mathrm{embed}}^{\mathrm{SMK}}(N,M)
=
\mathcal O(M+N).
\)
Since \(M \le N\), this simplifies to
\(
T_{\mathrm{embed}}^{\mathrm{SMK}}(N,M)=\mathcal O(N).
\)
Hence,
\[
T_{\mathrm{ours}}(N,M)
\in
\mathcal O(N^2+N)=\mathcal O(N^2).
\]

Therefore, the proposed SMK module increases the pre-encoder computation only by
a lower-order term and does not change the overall asymptotic complexity class
of the model. The dominant computational bottleneck remains the quadratic
self-attention in the TNP backbone.

\subsection{Proofs for Section~\ref{sec:smfe}}
\label{app:sec3-proofs}

\subsubsection{Definitions of translation equivariance and translation invariance}
\label{app:translation-definitions}

We state the definitions for the multi-dimensional spectral branch used in
Appendix~\ref{app:multidim}. Let
\(\mathcal P\subseteq\{1,\dots,d_x\}\) be the selected input-coordinate subset,
let \(d_p=|\mathcal P|\), and write
\(\mathbf x^{\mathcal P}\in\mathbb R^{d_p}\) for the corresponding coordinate
subvector. For \(\boldsymbol\Delta\in\mathbb R^{d_p}\), define the translation
operator \(T_{\boldsymbol\Delta}\) by
\[
(T_{\boldsymbol\Delta}\mathbf x)^{\mathcal P}
=
\mathbf x^{\mathcal P}+\boldsymbol\Delta,
\]
with coordinates outside \(\mathcal P\) left unchanged when
\(\mathcal P\ne\{1,\dots,d_x\}\). The full-input case is obtained by taking
\(\mathcal P=\{1,\dots,d_x\}\).

\begin{definition}[Translation equivariance]
\label{def:translation-equivariance}
Let \(\phi:\mathcal X\to\mathcal V\) be a feature map into a representation
space \(\mathcal V\). We say that \(\phi\) is translation equivariant with
respect to the selected coordinates \(\mathcal P\) if, for every
\(\boldsymbol\Delta\in\mathbb R^{d_p}\), there exists a linear operator
\(\rho_{\boldsymbol\Delta}:\mathcal V\to\mathcal V\) such that
\[
\phi(T_{\boldsymbol\Delta}\mathbf x)
=
\rho_{\boldsymbol\Delta}\phi(\mathbf x)
\qquad \text{for all }\mathbf x\in\mathcal X.
\]
\end{definition}

\begin{definition}[Translation invariance]
\label{def:translation-invariance}
Let \(k:\mathcal X\times\mathcal X\to\mathbb R\) be a scalar-valued function.
We say that \(k\) is translation invariant with respect to the selected
coordinates \(\mathcal P\) if, for every
\(\boldsymbol\Delta\in\mathbb R^{d_p}\),
\[
k(T_{\boldsymbol\Delta}\mathbf x,T_{\boldsymbol\Delta}\mathbf x')
=
k(\mathbf x,\mathbf x')
\qquad \text{for all }\mathbf x,\mathbf x'\in\mathcal X.
\]
\end{definition}

\subsubsection{Proof of Proposition~\ref{prop:translation-equivariance}}
\label{app:proof-translation-equivariance}
\begin{proof}
We prove the multi-dimensional statement; the one-dimensional proposition in
the main text is recovered by taking \(d_p=1\). Fix the inferred spectral
mixture parameters
\[
\theta_{\mathcal{D}_C}
=
\{(w_q,\boldsymbol\mu_q,\mathbf\Sigma_q,\phi_q)\}_{q=1}^Q
\]
and fix the sampled frequency vectors
\(\boldsymbol\omega_{q,d}\in\mathbb R^{d_p}\). For one \((q,d)\) block, define
\[
\psi_{q,d}(\mathbf x)
=
\sqrt{\frac{w_q}{D_0}}
\begin{bmatrix}
\cos(\boldsymbol\omega_{q,d}^{\top}\mathbf x^{\mathcal P}-\phi_q)\\
\sin(\boldsymbol\omega_{q,d}^{\top}\mathbf x^{\mathcal P}-\phi_q)
\end{bmatrix}.
\]
For a selected-coordinate translation \(\boldsymbol\Delta\), we have
\[
\boldsymbol\omega_{q,d}^{\top}
(T_{\boldsymbol\Delta}\mathbf x)^{\mathcal P}
-\phi_q
=
\boldsymbol\omega_{q,d}^{\top}\mathbf x^{\mathcal P}
-\phi_q
+\boldsymbol\omega_{q,d}^{\top}\boldsymbol\Delta .
\]
Using the angle-addition identities with
\(a=\boldsymbol\omega_{q,d}^{\top}\mathbf x^{\mathcal P}-\phi_q\) and
\(b=\boldsymbol\omega_{q,d}^{\top}\boldsymbol\Delta\), we obtain
\[
\psi_{q,d}(T_{\boldsymbol\Delta}\mathbf x)
=
R(\boldsymbol\omega_{q,d}^{\top}\boldsymbol\Delta)\psi_{q,d}(\mathbf x),
\]
where
\[
R(\alpha)
=
\begin{bmatrix}
\cos\alpha & -\sin\alpha\\
\sin\alpha & \cos\alpha
\end{bmatrix}.
\]
The full spectral feature is the direct sum over all \((q,d)\) blocks, hence
\[
\phi_S(T_{\boldsymbol\Delta}\mathbf x)
=
\left(
\bigoplus_{q=1}^{Q}\bigoplus_{d=1}^{D_0}
R(\boldsymbol\omega_{q,d}^{\top}\boldsymbol\Delta)
\right)
\phi_S(\mathbf x).
\]
Thus \(\phi_S\) is translation equivariant with
\[
\rho_{\boldsymbol\Delta}(\theta_{\mathcal{D}_C})
=
\bigoplus_{q=1}^{Q}\bigoplus_{d=1}^{D_0}
R(\boldsymbol\omega_{q,d}^{\top}\boldsymbol\Delta).
\]
If multi-output channel-wise spectral features
\(\phi_S^{\mathcal A}(\mathbf x)=\bigoplus_{c\in\mathcal A}\phi_S^{(c)}(\mathbf x)\)
are used, the same argument applies to each channel-wise branch and the
corresponding representation is the direct sum of the branch-wise rotation
operators.
\end{proof}

\subsubsection{Proof of Corollary~\ref{cor:translation-invariance}}
\label{app:proof-translation-invariance}
\begin{proof}
By Proposition~\ref{prop:translation-equivariance},
\[
\phi_S(T_{\boldsymbol\Delta}\mathbf x)
=
\rho_{\boldsymbol\Delta}(\theta_{\mathcal{D}_C})\phi_S(\mathbf x),
\qquad
\phi_S(T_{\boldsymbol\Delta}\mathbf x')
=
\rho_{\boldsymbol\Delta}(\theta_{\mathcal{D}_C})\phi_S(\mathbf x').
\]
Each block
\(R(\boldsymbol\omega_{q,d}^{\top}\boldsymbol\Delta)\) is orthogonal, so the
block-diagonal matrix \(\rho_{\boldsymbol\Delta}(\theta_{\mathcal{D}_C})\) is
orthogonal:
\[
\rho_{\boldsymbol\Delta}(\theta_{\mathcal{D}_C})^\top
\rho_{\boldsymbol\Delta}(\theta_{\mathcal{D}_C})
=I.
\]
Therefore the induced random-feature inner product satisfies
\[
\begin{aligned}
\hat k_{\theta_{\mathcal{D}_C}}
(T_{\boldsymbol\Delta}\mathbf x,T_{\boldsymbol\Delta}\mathbf x')
&=
\phi_S(\mathbf x)^\top
\rho_{\boldsymbol\Delta}(\theta_{\mathcal{D}_C})^\top
\rho_{\boldsymbol\Delta}(\theta_{\mathcal{D}_C})
\phi_S(\mathbf x')\\
&=
\phi_S(\mathbf x)^\top\phi_S(\mathbf x')
=
\hat k_{\theta_{\mathcal{D}_C}}(\mathbf x,\mathbf x').
\end{aligned}
\]
This proves translation invariance with respect to the selected coordinates.
\end{proof}

\subsubsection{Proof of Proposition~\ref{prop:stationary-kernel}}
\label{app:proof-stationary-kernel}
\begin{proof}
Let
\[
\boldsymbol\tau
=
\mathbf x^{\mathcal P}-(\mathbf x')^{\mathcal P}
\in\mathbb R^{d_p}.
\]
For a fixed \((q,d)\) block,
\[
\begin{aligned}
&
\begin{bmatrix}
\cos(\boldsymbol\omega_{q,d}^{\top}\mathbf x^{\mathcal P}-\phi_q)\\
\sin(\boldsymbol\omega_{q,d}^{\top}\mathbf x^{\mathcal P}-\phi_q)
\end{bmatrix}^{\!\top}
\begin{bmatrix}
\cos(\boldsymbol\omega_{q,d}^{\top}(\mathbf x')^{\mathcal P}-\phi_q)\\
\sin(\boldsymbol\omega_{q,d}^{\top}(\mathbf x')^{\mathcal P}-\phi_q)
\end{bmatrix} \\
&\qquad
=
\cos\!\left(
\boldsymbol\omega_{q,d}^{\top}
(\mathbf x^{\mathcal P}-(\mathbf x')^{\mathcal P})
\right)
=
\cos(\boldsymbol\omega_{q,d}^{\top}\boldsymbol\tau).
\end{aligned}
\]
Hence, conditional on the sampled frequencies,
\[
\phi_S(\mathbf x)^\top\phi_S(\mathbf x')
=
\sum_{q=1}^Q\frac{w_q}{D_0}
\sum_{d=1}^{D_0}
\cos(\boldsymbol\omega_{q,d}^{\top}\boldsymbol\tau).
\]
Taking expectation over the sampled frequencies conditional on
\(\theta_{\mathcal{D}_C}\), with
\(\boldsymbol\omega\sim\mathcal N(\boldsymbol\mu_q,\mathbf\Sigma_q)\), gives
\[
k_{\theta_{\mathcal{D}_C}}(\mathbf x,\mathbf x')
=
\sum_{q=1}^{Q}
w_q\,
\mathbb E_{\boldsymbol\omega\sim
\mathcal N(\boldsymbol\mu_q,\mathbf\Sigma_q)}
\left[\cos(\boldsymbol\omega^\top\boldsymbol\tau)\right].
\]
Using the characteristic function of a multivariate Gaussian,
\[
\mathbb E
\left[
\exp(i\boldsymbol\omega^\top\boldsymbol\tau)
\right]
=
\exp\!\left(
i\boldsymbol\mu_q^\top\boldsymbol\tau
-\frac12\boldsymbol\tau^\top\mathbf\Sigma_q\boldsymbol\tau
\right),
\]
and taking real parts yields
\[
\mathbb E
\left[\cos(\boldsymbol\omega^\top\boldsymbol\tau)\right]
=
\exp\!\left(
-\frac12\boldsymbol\tau^\top\mathbf\Sigma_q\boldsymbol\tau
\right)
\cos(\boldsymbol\mu_q^\top\boldsymbol\tau).
\]
Therefore
\[
k_{\theta_{\mathcal{D}_C}}(\boldsymbol\tau)
=
\sum_{q=1}^{Q}
w_q
\exp\!\left(
-\frac12\boldsymbol\tau^\top\mathbf\Sigma_q\boldsymbol\tau
\right)
\cos(\boldsymbol\mu_q^\top\boldsymbol\tau).
\]
This expression depends on \(\mathbf x\) and \(\mathbf x'\) only through
\(\boldsymbol\tau=\mathbf x^{\mathcal P}-(\mathbf x')^{\mathcal P}\), so the
conditional kernel is stationary on the selected input coordinates. Under the
isotropic approximation \(\mathbf\Sigma_q=\sigma_q^2\mathbf I_{d_p}\), this
reduces to
\[
\sum_{q=1}^{Q}
w_q
\exp\!\left(-\frac12\sigma_q^2\|\boldsymbol\tau\|_2^2\right)
\cos(\boldsymbol\mu_q^\top\boldsymbol\tau),
\]
which recovers the scalar formula in the main text when \(d_p=1\).
\end{proof}

\subsubsection{Kolmogorov Extension Theorem}
\label{app:kolmogorov-extension}

For each fixed context set \(\mathcal{D}_C\), consider the family of
finite-dimensional predictive distributions
\[
\left\{
P_\Theta(\mathbf y_T\mid \mathbf x_T,\mathcal{D}_C)
\right\}_{n\in\mathbb N,\;\mathbf x_T\in\mathcal X^n},
\]
where
\(\mathbf x_T=(\mathbf x_1,\dots,\mathbf x_n)\) is a finite tuple of target
inputs and
\(\mathbf y_T=(\mathbf y_1,\dots,\mathbf y_n)\) denotes the corresponding
outputs. We say that this family is \emph{target-exchangeable} if, for every
permutation \(\pi\) of \(\{1,\dots,n\}\),
\[
P_\Theta(\mathbf y_1,\dots,\mathbf y_n
\mid
\mathbf x_1,\dots,\mathbf x_n,\mathcal{D}_C)
=
P_\Theta(\mathbf y_{\pi(1)},\dots,\mathbf y_{\pi(n)}
\mid
\mathbf x_{\pi(1)},\dots,\mathbf x_{\pi(n)},\mathcal{D}_C).
\]
We say that this family is \emph{marginalisable} if, for every
\(A\subseteq\{1,\dots,n\}\),
\[
P_\Theta(\mathbf y_A\mid \mathbf x_A,\mathcal{D}_C)
=
\int
P_\Theta(\mathbf y_T\mid \mathbf x_T,\mathcal{D}_C)
d\mathbf y_{T\setminus A}.
\]
By the Kolmogorov extension theorem, any target-exchangeable and
marginalisable family defines a conditional stochastic process indexed by
\(\mathcal X\).

\subsubsection{Proof of Proposition~\ref{prop:process-consistency}}
\label{app:proof-process-consistency}
\begin{proof}
Fix a context set \(\mathcal{D}_C\) and a finite target tuple
\(\mathbf x_T=(\mathbf x_1,\dots,\mathbf x_n)\). Under the target-wise
predictive parameterisation used by STNP, the predictive distribution
factorises as
\[
p_\Theta(\mathbf y_T\mid \mathbf x_T,\mathcal{D}_C)
=
\prod_{j=1}^{n}
p_\Theta(\mathbf y_j\mid \mathbf x_j,\mathcal{D}_C).
\]
Each factor is obtained by evaluating the same conditional prediction rule at
the target input \(\mathbf x_j\), with the context-dependent spectral summary
and the non-spectral branch fixed by \(\mathcal{D}_C\).

\paragraph{Context exchangeability.}
The empirical spectral statistics \(a_k,b_k,E_k\), the discrete probabilities
\(p_k\), and the component moments
\((w_q,\boldsymbol\mu_q,\mathbf\Sigma_q)\) are all sums or normalised sums over
the context set. They are therefore invariant to permutations of the elements
of \(\mathcal{D}_C\). The same holds for shared selected-channel and
channel-wise selected spectra because they are formed by sums over selected
output channels and over context points. Consequently, every predictive factor
\(p_\Theta(\mathbf y_j\mid \mathbf x_j,\mathcal{D}_C)\) is invariant to
permutations of the context set, so the predictive family is context
exchangeable.

\paragraph{Target exchangeability.}
For any permutation \(\pi\) of \(\{1,\dots,n\}\),
\[
\prod_{j=1}^{n}
p_\Theta(\mathbf y_{\pi(j)}
\mid
\mathbf x_{\pi(j)},\mathcal{D}_C)
=
\prod_{j=1}^{n}
p_\Theta(\mathbf y_j\mid \mathbf x_j,\mathcal{D}_C),
\]
because the product factors are merely reordered. Thus the family is
exchangeable in the target tuple.

\paragraph{Marginalisability.}
Let \(A\subseteq\{1,\dots,n\}\) and \(B=\{1,\dots,n\}\setminus A\). Then
\[
p_\Theta(\mathbf y_T\mid \mathbf x_T,\mathcal{D}_C)
=
\left(
\prod_{j\in A}
p_\Theta(\mathbf y_j\mid \mathbf x_j,\mathcal{D}_C)
\right)
\left(
\prod_{j\in B}
p_\Theta(\mathbf y_j\mid \mathbf x_j,\mathcal{D}_C)
\right).
\]
Integrating out \(\mathbf y_B\) gives
\[
\int
p_\Theta(\mathbf y_T\mid \mathbf x_T,\mathcal{D}_C)
d\mathbf y_B
=
\prod_{j\in A}
p_\Theta(\mathbf y_j\mid \mathbf x_j,\mathcal{D}_C),
\]
because each factor over \(\mathbf y_j\), \(j\in B\), is a normalised
predictive density. The right-hand side is exactly
\(p_\Theta(\mathbf y_A\mid \mathbf x_A,\mathcal{D}_C)\). Hence the family is
marginalisable, equivalently projectively consistent.

\paragraph{Existence of a conditional process.}
For every fixed context set \(\mathcal{D}_C\), the finite-dimensional predictive
family is target-exchangeable and marginalisable. The Kolmogorov extension
theorem therefore guarantees a conditional stochastic process
\(\{F_{\mathcal{D}_C}(\mathbf x)\}_{\mathbf x\in\mathcal X}\) whose
finite-dimensional distributions coincide with those defined by STNP. Thus
STNP defines a valid conditional stochastic process.
\end{proof}

\subsection{Synthetic regression details}
\label{app:synthetic-regression-details}

\paragraph{Data generation.}
All datasets in Section~\ref{sec:synthetic-regression} are synthetic
one-dimensional regression task distributions. During training, we do not use a
fixed pre-generated training set; instead, a fresh batch of tasks is sampled
online at every optimisation step, and all models are trained for 100,000
optimisation steps. For evaluation, we generate a fixed cached
evaluation set with an independent random seed and reuse exactly the same
cached tasks for all compared methods. In the implementation, this cache
contains 3000 batches with batch size 16 and is reused across all models.

For the standard GP benchmarks, the input locations are sampled uniformly from
\([-2,2]\), and the outputs are drawn from a zero-mean Gaussian process with a
kernel-specific covariance function. For the RBF and Mat\'ern-$5/2$ tasks, we
sample the number of context points as \(m \sim \mathcal{U}[3,47]\) and then
sample the number of target points as \((N-m) \sim \mathcal{U}[3,50-m]\). For
the periodic task, severe aliasing under sparse observations makes the
underlying periodic structure more difficult to recover, so we instead choose a
minimum context budget \(m_{\min}\in\{10,15,20,25\}\), draw
\(m \sim \mathcal{U}[m_{\min},47]\), and again sample
\((N-m) \sim \mathcal{U}[3,50-m]\). In all three GP benchmarks, we add a
diagonal observation-noise term with \(\sigma_{\varepsilon}=0.02\).

\begin{figure}[t]
\centering
\includegraphics[width=\textwidth]{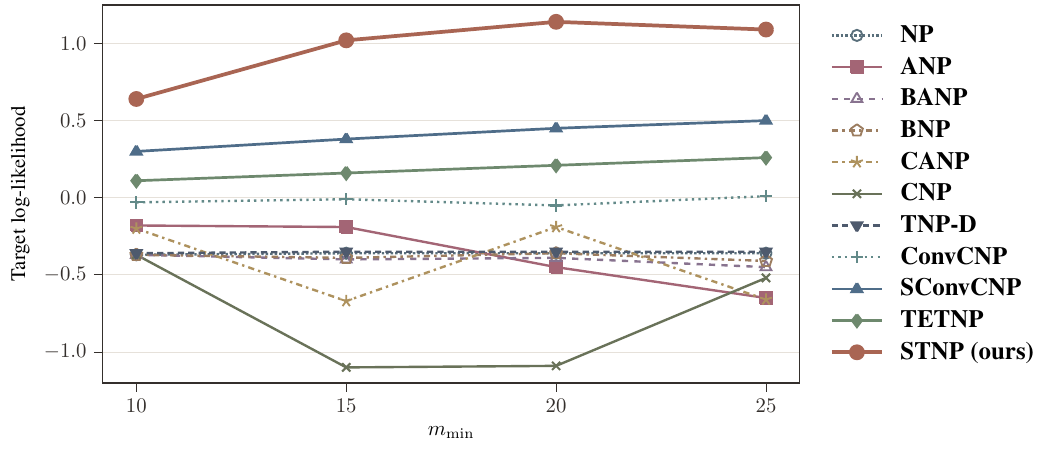}
\caption{Effect of the minimum context budget on predictive log-likelihood for
the periodic-kernel synthetic regression task. The evaluation varies
\(m_{\min}\), where context sizes are sampled from
\(m \sim \mathcal{U}[m_{\min},47]\); higher log-likelihood is better.}
\label{fig:synthetic-mmin-loglikelihood}
\end{figure}

For the RBF kernel, we use
\[
k_{\mathrm{RBF}}(x,x')
=
s^2 \exp\!\left(-\frac{|x-x'|^2}{2\ell^2}\right)
+
\sigma_{\varepsilon}^2 \mathbf{1}\{x=x'\},
\]
where \(s \sim \mathcal{U}(0.1,1.0)\) and \(\ell \sim \mathcal{U}(0.1,0.6)\).
For the Mat\'ern-$5/2$ kernel, we use
\[
k_{\mathrm{Mat}}(x,x')
=
s^2\left(1+\frac{\sqrt{5}\,r}{\ell}+\frac{5r^2}{3\ell^2}\right)
\exp\!\left(-\frac{\sqrt{5}\,r}{\ell}\right)
+
\sigma_{\varepsilon}^2 \mathbf{1}\{x=x'\},
\qquad r = |x-x'|.
\]
We use the same ranges for \(s\) and \(\ell\), together with the same diagonal
noise term. For the periodic kernel, we use
\[
k_{\mathrm{per}}(x,x')
=
s^2 \exp\!\left(-\frac{2\sin^2(\pi |x-x'|/p)}{\ell^2}\right)
+
\sigma_{\varepsilon}^2 \mathbf{1}\{x=x'\},
\]
where \(s \sim \mathcal{U}(0.1,1.0)\), \(\ell \sim \mathcal{U}(0.6,1.0)\), and
\(p \sim \mathcal{U}(0.1,0.5)\). These ranges randomise the amplitude,
smoothness, and period of each sampled function so that the model is trained on
a broad family of tasks rather than a single fixed kernel configuration.

Beyond the GP tasks, the synthetic generator also includes several structured
non-Gaussian benchmarks. For the sawtooth benchmark used in the qualitative
comparisons, we sample \(x\) uniformly from \([-2,2]\) and generate \(y\) from
the random sawtooth function
\[
y_{\mathrm{sawtooth}}(x)
:=
\frac{A}{2}
-
\frac{A}{\pi}
\sum_{k=1}^{\infty}
\frac{(-1)^k \sin\!\bigl(2\pi k f (x-\tau)\bigr)}{k},
\]
where \(A\) is the amplitude, \(f\) is the frequency, and \(\tau\) is the
horizontal shift. Throughout training, we fix \(A=1\), truncate the series at
an integer \(K\) sampled uniformly from \(\{10,\dots,20\}\), sample the
frequency \(f\) uniformly from \(\mathcal{U}(3,5)\), and sample the shift
\(\tau\) uniformly from \(\mathcal{U}(-5,5)\).

\paragraph{Attention recovery.}
Figure~\ref{fig:periodic-attention-recovery} visualises a representative
periodic episode with \(m=16\), where context and target points are
sorted by phase and the target-to-context attention of the final Transformer
layer is shown for TNP and STNP. In the baseline TNP, the attention maps are
dominated by nearly phase-independent vertical bands, meaning that many target
points repeatedly attend to the same few context points. By contrast, STNP
produces much sharper target-dependent patterns that align with the phase
structure of the signal. This suggests that the proposed spectral embedding
helps the inherited TNP attention recover the relevant context points from
which each target should gather information.

\begin{figure*}[t]
\centering
\includegraphics[width=\textwidth]{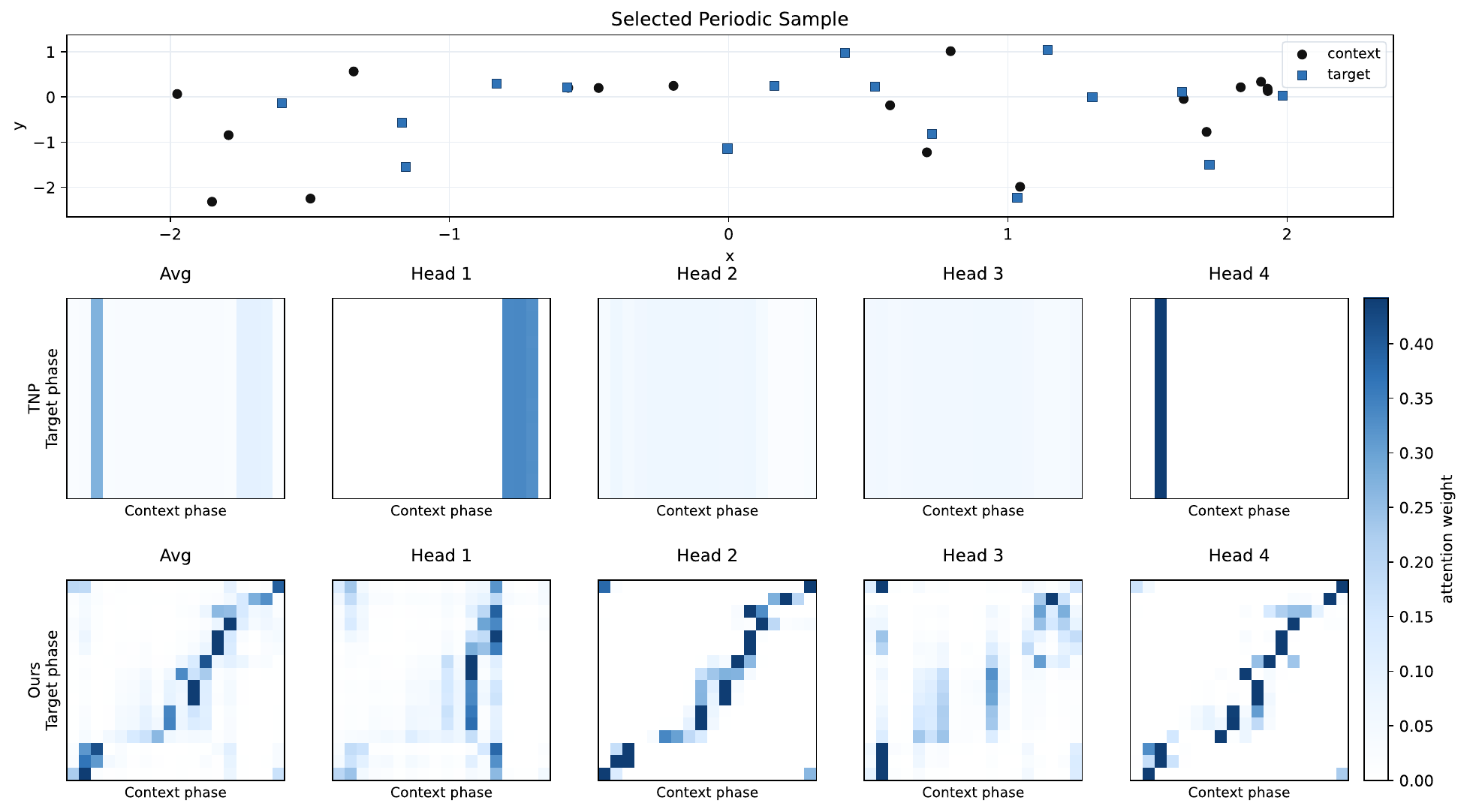}
\caption{Attention recovery on a representative periodic task with
\(m=16\). The top row shows the sampled context and target points. The
middle row shows the average and per-head target-to-context attention maps of
TNP after sorting both target and context points by phase, while the bottom row
shows the corresponding maps for STNP. TNP exhibits nearly phase-independent
vertical attention bands, whereas STNP displays phase-aligned, target-dependent
attention patterns, indicating that the spectral embedding helps the backbone
retrieve information from the appropriate context points.}
\label{fig:periodic-attention-recovery}
\end{figure*}

\paragraph{Hyperparameters.}
All baseline models use the default hyperparameters reported in their
respective prior work. For STNP, we use the TNP model configuration for the
underlying TNP backbone, together with the following settings relevant to the
spectral-feature implementation.
\begin{itemize}[leftmargin=1.4em,itemsep=1pt,topsep=2pt]
\item Learning rate: \(5\times 10^{-4}\) with cosine annealing scheduler.
\item Number of frequency-grid points: 128.
\item Number of spectral-mixture components: 6.
\item Number of sampled frequencies per component: 8.
\item Minimum period \(p_{\min}\): 0.1.
\item Maximum period \(p_{\max}\): 2.0.
\item Frequency-grid spacing: logarithmic.
\item Number of convolutional channels: 64.
\item Number of convolutional layers: 3.
\item Convolution kernel size: 5.
\item Convolution dilation growth: 1.
\item Convolutional dropout: 0.0.
\item Convolutional layer normalisation: disabled.
\end{itemize}

Table~\ref{tab:synthetic-parameter-counts} reports the trainable parameter
counts for the synthetic regression code configurations.

\begin{table}[t]
\centering
\small
\caption{Trainable parameter counts for the synthetic regression code
configurations.}
\label{tab:synthetic-parameter-counts}
\resizebox{\textwidth}{!}{%
\begin{tabular}{lcccccccccc}
\toprule
\textbf{Model} & NP & CNP & ANP & CANP & BNP & BANP & SConvCNP & TNP & TETNP & STNP (ours) \\
\midrule
\textbf{\# Parameters} & 232,194 & 215,682 & 348,418 & 331,906 & 248,450 & 364,674 & 225,770 & 222,146 & 249,754 & 232,706 \\
\bottomrule
\end{tabular}
}
\end{table}

\subsection{Ablation Studies}
\label{app:ablation}
\label{app:further-discussion}

This section expands the studies in Section~\ref{sec:ablation-study}.
We examine the spectral feature design and hyperparameters, then compare
STNP with a spectral mixture GP.

To isolate the contribution of each design choice in STNP, we include an
ablation study over the joint choice of the number of spectral components
\(Q\) and the number of sampled frequencies per component \(D_0\), the size of
the frequency grid \(K\), the architecture used for the responsibility
network, and the form of the Fourier feature map. We conduct these design ablations on
the periodic-kernel Gaussian process regression task, using three minimum
context budgets \(m_{\min}\in\{10,15,20\}\). This is the most challenging
setting among our synthetic experiments, because strong periodic structure
combined with sparse observations induces severe aliasing, making the
underlying frequencies difficult to recover reliably. Unless otherwise stated,
all remaining training and evaluation settings are kept identical to the
default configuration used in the main experiments.
Table~\ref{tab:ablation} summarises the embedding, joint \((Q,D_0)\),
frequency-grid, and responsibility-network ablations discussed below.

\begin{table*}[t]
\centering
\caption{Ablation study on periodic-kernel Gaussian process regression.
Entries report target-point log-likelihoods. Columns correspond to the three
minimum-context settings \(m_{\min}\in\{10,15,20\}\). In the joint
\((Q,D_0)\) sweep, all settings keep \(Q D_0=48\), so the total spectral
feature budget is fixed while its allocation across latent components and
within-component samples changes.}
\label{tab:ablation}
\resizebox{\textwidth}{!}{
\begin{tabular}{llccc}
\toprule
Category & Variant & \(m_{\min}=10\) & \(m_{\min}=15\) & \(m_{\min}=20\) \\
\midrule
Full model & STNP (ours) & 0.63 & 1.02 & 1.14 \\
\midrule
\multirow{3}{*}{Embedding parameterisation}
& \([\cos(W_p x)\mid\sin(W_p x)\mid\phi_{\mathrm{mlp}}(x,y)]\) & -0.36 & -0.35 & -0.35 \\
& \([\phi_{\mathrm{disc}}(x)\mid\phi_{\mathrm{mlp}}(x,y)]\) & 0.10 & 0.21 & 0.29 \\
& \([\phi_{\mathrm{rff}}(x)\mid\phi_{\mathrm{mlp}}(x,y)]\) & 0.46 & 0.60 & 0.71 \\
\midrule
\multirow{10}{*}{Joint choice of \(Q\) and \(D_0\)}
& \((Q,D_0)=(1,48)\) & 0.14 & 0.13 & 0.22 \\
& \((Q,D_0)=(2,24)\) & 0.14 & 0.19 & 0.76 \\
& \((Q,D_0)=(3,16)\) & 0.58 & 0.84 & 0.86 \\
& \((Q,D_0)=(4,12)\) & 0.81 & 0.76 & 0.86 \\
& \((Q,D_0)=(6,8)\) (default) & 0.63 & 1.02 & 1.14 \\
& \((Q,D_0)=(8,6)\) & 0.62 & 0.76 & 1.02 \\
& \((Q,D_0)=(12,4)\) & 0.74 & 1.07 & 0.93 \\
& \((Q,D_0)=(16,3)\) & 0.59 & 0.99 & 1.06 \\
& \((Q,D_0)=(24,2)\) & 0.88 & 0.61 & 0.18 \\
& \((Q,D_0)=(48,1)\) & 0.58 & 1.04 & 1.18 \\
\midrule
\multirow{4}{*}{Frequency grid size \(K\)}
& \(K=64\) & 0.12 & 0.26 & 0.26 \\
& \(K=128\) (default) & 0.63 & 1.02 & 1.14 \\
& \(K=256\) & 0.18 & 0.81 & 0.89 \\
& \(K=512\) & 0.52 & 0.69 & 0.90 \\
\midrule
\multirow{2}{*}{Responsibility network}
& CNN (ours) & 0.63 & 1.02 & 1.14 \\
& FFN & 0.05 & 0.67 & 0.81 \\
\bottomrule
\end{tabular}}
\end{table*}

\paragraph{Replacing the proposed spectral embedding with direct FAN-style Fourier features.}
We first remove the context-conditioned spectral branch and instead concatenate
the base MLP embedding with a direct Fourier feature map of the form
\([\cos(W_p x)\mid\sin(W_p x)\mid\phi_{\mathrm{mlp}}(x,y)]\), following the
parameterisation used in Fourier Analysis Networks
\cite{Dong2024FANFA}. This ablation tests whether the gains come mainly from
using sinusoidal features in general, or from the task-adaptive spectral
mixture estimated from the context set. Performance drops sharply to
\((-0.36, -0.35, -0.35)\), far below both the full model
\((0.63, 1.02, 1.14)\) and the direct discrete-spectrum embedding
\((0.10, 0.21, 0.29)\). The gap to STNP grows from \(0.99\) at
\(m_{\min}=10\) to \(1.49\) at \(m_{\min}=20\), showing that the improvement
cannot be explained by the mere presence of sinusoidal features. Instead, the
key benefit comes from conditioning the spectral representation on the context
set and adapting it to the task at hand; without that mechanism, the Fourier
features provide a poor inductive bias for sparse periodic regression.

\paragraph{Replacing spectral compression and sampling with a direct discrete-spectrum embedding.}
We further consider an alternative that bypasses the latent
spectral decomposition, mixture compression, and frequency sampling steps. In
this variant, we use the empirical discrete spectrum from Step 1 directly to
construct the embedding while keeping the frequency grid size \(K\) identical
to the original model,
\[
\phi_{\mathrm{disc}}(x)
:=
\bigoplus_{k=1}^{K}
\sqrt{p_k}
\begin{bmatrix}
\cos(\tilde{\omega}_k x)\\[0.2em]
\sin(\tilde{\omega}_k x)
\end{bmatrix},
\]
and concatenate it with \(\phi_{\mathrm{mlp}}(x,y)\). This ablation isolates
the benefit of the learned compression-and-sampling mechanism from the benefit
of spectral estimation alone. This direct
discrete-spectrum variant reaches \((0.10, 0.21, 0.29)\), substantially
improving over the fixed FAN-style embedding \((-0.36, -0.35, -0.35)\) in
every setting. However, it remains well below the full model
\((0.63, 1.02, 1.14)\), with gaps of \(0.53\), \(0.81\), and \(0.85\) as
\(m_{\min}\) increases. This indicates that spectral
estimation alone is not enough.

\paragraph{Random Fourier features over the discrete spectrum.}
The direct discrete embedding \(\phi_{\mathrm{disc}}\) puts every grid frequency
\(\tilde{\omega}_k\) into the feature vector, weighted by \(\sqrt{p_k}\). A
natural alternative is to treat \((\tilde{\omega}_k, p_k)\) as a discrete
spectral measure and approximate the same kernel by a Random Fourier Feature
(RFF) Monte-Carlo estimator \cite{Rahimi2007RandomFF}. Concretely, we draw
\(M\) frequencies independently per task from the empirical spectrum,
\(\omega^{(m)}\sim\mathrm{Categorical}(p_1,\dots,p_K)\) over
\(\{\tilde{\omega}_1,\dots,\tilde{\omega}_K\}\), and form
\[
\phi_{\mathrm{rff}}(x)
:=
\frac{1}{\sqrt{M}}
\bigoplus_{m=1}^{M}
\begin{bmatrix}
\cos(\omega^{(m)} x)\\[0.2em]
\sin(\omega^{(m)} x)
\end{bmatrix},
\]
which satisfies
\(\mathbb{E}[\langle\phi_{\mathrm{rff}}(x),\phi_{\mathrm{rff}}(x')\rangle]
=\sum_{k=1}^{K} p_k \cos(\tilde{\omega}_k(x-x'))
=\langle\phi_{\mathrm{disc}}(x),\phi_{\mathrm{disc}}(x')\rangle\), so
\(\phi_{\mathrm{rff}}\) is an unbiased estimator of the same kernel as
\(\phi_{\mathrm{disc}}\) but with a controllable feature budget \(M\) instead
of the fixed grid size \(K\). For comparability with the full STNP, we set
\(M=Q D_0 = 48\) so that the spectral feature width matches the default
configuration. This ablation isolates two effects that are entangled in
\(\phi_{\mathrm{disc}}\): (i) using all \(K\) grid frequencies versus a
\(p\)-weighted sample of \(M\), and (ii) the resulting feature dimensionality
\(2K\) versus \(2M\). Each forward pass redraws \(\omega^{(m)}\) per task, so
the encoder also sees stochastic Fourier features at training time. Replacing
\(\phi_{\mathrm{disc}}\) with \(\phi_{\mathrm{rff}}\) yields target log-likelihoods
of \((0.46, 0.60, 0.71)\) at \(m_{\min}\in\{10,15,20\}\). This
shows that, although the two feature maps share the same kernel in expectation,
\(\phi_{\mathrm{rff}}\) and \(\phi_{\mathrm{disc}}\) provide measurably different
inductive biases for the downstream Transformer, and neither closes the gap to
the full STNP pipeline.

\paragraph{Effect of the joint choice of \(Q\) and \(D_0\).}
We jointly vary the number of mixture components \(Q\) and the number of
sampled frequencies per component \(D_0\), while keeping the remaining
architecture fixed. This ablation is intended to study the trade-off between
the granularity of the latent spectral decomposition and the expressivity of
the resulting spectral features, and to identify whether performance is more
sensitive to increasing the number of components or the number of samples per
component. To isolate this trade-off from a trivial capacity increase, we keep
\(Q D_0=48\) fixed throughout the sweep, so every setting uses the same
\(2Q D_0=96\)-dimensional spectral features. The results show that performance is
highly sensitive to how this fixed budget is allocated, and that the preferred
allocation depends strongly on the context-budget setting. Under the sparsest setting
\(m_{\min}=10\), very small \(Q\) performs poorly: \((1,48)\) and \((2,24)\)
both reach only \(0.14\), indicating that too few mixture components cannot
separate the aliased modes induced by sparse periodic observations. In this
setting, a more refined decomposition is beneficial, and the best result is
obtained by \((24,2)\) with \(0.88\), followed by \((4,12)\) with \(0.81\).

As more context becomes available, the preferred balance shifts towards larger
\(Q\) and smaller \(D_0\). For \(m_{\min}=15\), the strongest setting is
\((12,4)\) with \(1.07\), closely followed by \((48,1)\) with \(1.04\) and the
default \((6,8)\) with \(1.02\). For \(m_{\min}=20\), \((48,1)\) achieves the
highest score \(1.18\), slightly above the default \(1.14\), while
\((16,3)\) remains competitive at \(1.06\). This pattern suggests that once
the empirical spectrum is estimated more reliably, the model benefits more from
a finer latent decomposition into many narrow bands and relies less on drawing
many samples within each band.

The dependence is also clearly non-monotonic, which shows that \(Q\) and
\(D_0\) are not interchangeable despite having the same product. In particular,
\((24,2)\) is the best configuration at \(m_{\min}=10\) but collapses to
\(0.18\) at \(m_{\min}=20\), whereas \((48,1)\) is only moderate in the
sparsest setting but becomes the strongest once additional context is
available. The default choice \((6,8)\) is therefore not always the single best
point, but it remains consistently competitive across all three budgets and
avoids the large failures seen in more extreme allocations. Overall, the sweep
shows a genuine trade-off: larger \(Q\) improves frequency selectivity, whereas
larger \(D_0\) stabilises the representation within each latent band. Moderate
allocations such as \((6,8)\), \((12,4)\), and \((16,3)\) provide the most
reliable performance when the available context varies across tasks. Even so,
how to select or adapt the optimal pair \((Q,D_0)\) for a given task remains an
open and worthwhile direction for future study.

\paragraph{Effect of the frequency grid size \(K\).}
We vary the size of the fixed frequency grid \(K\) used to discretise the
empirical spectrum in Step~1, with \(K=128\) as the default setting used by
the full model. The dependence on grid resolution is clearly non-monotonic. A
low-resolution grid with \(K=64\) performs poorly in all three settings
\((0.12,0.26,0.26)\), indicating that it under-resolves the periodic spectrum
and leaves the downstream mixture with insufficient spectral detail. Increasing
the grid to \(K=128\) yields the best target log-likelihood in every setting
\((0.63,1.02,1.14)\). Pushing the grid further to \(K=256\) or \(K=512\)
recovers part of the loss relative to \(K=64\), but does not improve over the
default. This suggests that, under sparse periodic observations, an overly
fine discrete spectrum becomes harder to estimate and compress reliably, so
the gain in resolution is offset by increased estimation noise. The penalty of
using a larger-than-default grid is strongest at \(m_{\min}=10\), while for
\(m_{\min}=20\) the gap between \(K=256\) and \(K=512\) narrows to \(0.01\),
indicating that finer grids become more viable as more context points are
available. Overall, \(K=128\) provides the best trade-off between spectral
resolution and robustness in the low-context periodic setting.

\paragraph{Replacing the CNN responsibility network with an FFN.}
To test whether the CNN responsibility network is
important, we replace it with a
feed-forward network that outputs the same \(Q\) logits. This comparison
isolates the effect of local spectral interaction modelling in Step 2 from the
rest of the STNP pipeline. The FFN variant obtains \((0.05, 0.67, 0.81)\),
consistently underperforming the CNN-based model \((0.63, 1.02, 1.14)\). The
deficit is especially pronounced at \(m_{\min}=10\), where the gap reaches
\(0.58\), and narrows to \(0.35\) and \(0.33\) at \(m_{\min}=15\) and
\(m_{\min}=20\). This pattern suggests that convolution across neighbouring
grid frequencies is most valuable in the severely undersampled setting, where
aliasing makes local spectral aggregation critical. Relative to the direct
discrete embedding, the FFN is already much stronger for \(m_{\min}\ge 15\),
which indicates that preserving the compression-and-sampling pipeline matters
more than the exact responsibility architecture once enough context is
available.

\subsubsection{Practical guidance for spectral hyperparameters}
\label{app:spectral-hyperparameters}

\paragraph{Frequency grid and period range.}
The period range determines the oscillatory scales covered by the frequency
grid and should be expressed in the same units as the input coordinates after
any rescaling. For scalar inputs, a period range \([p_{\min},p_{\max}]\)
corresponds to angular frequencies
\([2\pi/p_{\max},2\pi/p_{\min}]\). We recommend a conservatively broad range,
since weakly supported frequencies can receive low empirical energy and be
downweighted by the subsequent spectral decomposition. Moderate over-coverage
is therefore usually preferable to excluding a relevant frequency. Within
this range, the logarithmic spacing used in our experiments covers multiple
period scales. The range and grid size should be selected jointly, however,
because widening the range at fixed \(K\) reduces resolution.
Table~\ref{tab:ablation} supports \(K=128\) as a reasonable starting point for
one-dimensional grids, while further increases in \(K\) do not consistently
improve performance.

\paragraph{Number of mixture components \(Q\).}
Too few components may limit the mixture's ability to capture spectral
structure, whereas an
excessively large \(Q\) increases optimisation complexity and can lead to
overfitting or poorer predictive performance \cite{Chen2021SkewedLaplaceSM}.
Prior SM-kernel studies use \(Q=10\) as a practical heuristic
\cite{Chen2021SkewedLaplaceSM,Wilson2013GaussianPK}, and in several experiments
of Wilson and Adams, only seven of the ten components remain active. These
observations support a moderate, mildly over-complete mixture as a practical
starting point.

\paragraph{Samples per component \(D_0\).}
The role of \(D_0\) is to control the sampling resolution within each component,
not its spectral mass. Let \(\phi_{\mathrm S,q}(x)\) denote the subvector of
\(\phi_{\mathrm S}(x)\) associated with component \(q\). The normalisation
\(\sqrt{w_q/D_0}\) gives the exact identity
\begin{equation}
\label{eq:component-feature-norm}
\|\phi_{\mathrm S,q}(x)\|_2^2
=\sum_{d=1}^{D_0}\frac{w_q}{D_0}
\left[\cos^2(\omega_{q,d}x-\phi_q)
+\sin^2(\omega_{q,d}x-\phi_q)\right]
=w_q.
\end{equation}
Thus, increasing \(D_0\) represents the same mass with more sampled frequency
bases, without amplifying the component. Jung et al.~\cite{Jung2020ApproximateSM}
compare total spectral-sample budgets of \(3Q\), \(5Q\), and \(15Q\),
corresponding to average per-component budgets of \(3\), \(5\), and \(15\).
Their Figure~4 suggests smoother optimisation trajectories with larger budgets,
with much of the benefit visible by about five samples per component and
diminishing additional gains thereafter.

\paragraph{Practical starting point.}
Together, these observations suggest \(Q\approx10\) and \(D_0\ge5\) as a
useful starting region for new tasks, followed by joint validation of
\((Q,D_0)\) under the available feature budget \(2QD_0\).
Table~\ref{tab:ablation} shows that the preferred allocation depends on the
context budget. For reproducing our
synthetic benchmark, retain the reported default \((Q,D_0)=(6,8)\).

\FloatBarrier

\subsubsection{Amortised spectral inference and SM-GP fitting}
\label{app:smgp-comparison}
We evaluate a classical spectral mixture GP (SM-GP) alongside STNP on the
same periodic benchmark to compare per-task kernel fitting with amortised
spectral inference. The SM-GP uses the kernel of Wilson and
Adams~\cite{Wilson2013GaussianPK}. For each test task, we fit its kernel
hyperparameters to the context observations using four restarts with
200 optimisation steps each. STNP learns its context-to-spectrum mapping
across training tasks and is evaluated without test-time parameter optimisation.

\begin{table}[!t]
\centering
\small
\caption{Comparison of STNP with an exact spectral mixture GP on periodic
regression. Context LL is an in-sample posterior-predictive score.}
\label{tab:smgp-comparison}
\begin{tabular}{@{}lcc@{}}
\toprule
Metric or fitting budget & STNP & SM-GP \\
\midrule
Test-time optimisation & None & $4\times200$ steps \\
Context LL (in-sample) & $2.0191\pm0.0039$ & $2.1498\pm0.0097$ \\
Target LL $\uparrow$ & $\bm{1.1353\pm0.0069}$ & $-20.9163\pm0.6936$ \\
Target RMSE $\downarrow$ & $\bm{0.1870\pm0.0015}$ & $0.2989\pm0.0021$ \\
Predictive std & $0.1287\pm0.0009$ & $0.0663\pm0.0004$ \\
\bottomrule
\end{tabular}
\end{table}

Table~\ref{tab:smgp-comparison} shows that the SM-GP's higher context
log-likelihood does not translate into good predictions on held-out targets.
STNP achieves a target log-likelihood of $1.1353$, compared with $-20.9163$
for SM-GP, and also attains a lower target RMSE. The SM-GP's smaller
predictive standard deviation, combined with its low target likelihood,
suggests overconfidence when the context is sparse.

Kernel optimisation must also be repeated for every new task. Each
marginal-likelihood evaluation requires a Cholesky factorisation of the
\(M\times M\) context covariance matrix, costing $\mathcal O(M^3)$ time and
$\mathcal O(M^2)$ memory. STNP uses a meta-trained inference network to process
a new context in a forward pass.
The poor target generalisation and reliance on per-task optimisation therefore
make classical SM-GPs poorly suited to meta-learning settings that require
amortised inference across tasks.

\FloatBarrier

\subsection{DTD Image Completion Configurations}
\label{app:dtd_configs}
\label{app:image-completion-details}
\paragraph{Datasets and training.}
We report the configurations used for the DTD image-completion experiments \cite{Mohseni2024SpectralCC}.
Unless otherwise stated, all models share the same data preprocessing and
training protocol. Following the DTD setup in SConvCNP, we use the standard DTD
split with 1880 images for each of training, validation, and test. During
training, validation, and testing, we sample a random \(192\times192\) crop
from each image and downsample it to \(64\times64\). Pixel coordinates are
linearly mapped to \([-1,1]\) along each axis, and RGB intensities are
normalised to \([0,1]\) independently across channels. All models are trained
for 500 epochs using AdamW with learning rate \(10^{-4}\), gradient clipping
\(0.5\), and batch size 16. For each batch, the number of context pixels is
sampled as \(M\sim\mathcal U[5,1024)\), and all remaining pixels are treated as
query points, i.e., \(|\mathcal{I}_Q|=64\times64-M\). The learnable parameter
counts for all evaluated models are reported in
Table~\ref{tab:dtd-parameter-counts}.

\begin{table}[t]
\centering
\small
\caption{Learnable parameter counts for all models evaluated in the image
completion experiments.}
\label{tab:dtd-parameter-counts}
\begin{tabular}{lcccccc}
\toprule
 & CNP & AttCNP & TNP & ConvCNP & SConvCNP & STNP \\
\midrule
Number of parameters (million) & 5.0 & 8.4 & 13.7 & 12.2 & 14.2 & 14.6 \\
\bottomrule
\end{tabular}
\end{table}

\paragraph{Hyperparameters.}
Unless otherwise stated, all non-STNP models use the same hyperparameter
settings as in the SConvCNP benchmark \cite{Mohseni2024SpectralCC}. We
therefore detail only the STNP configuration here. Our image variant uses the same architecture as TNP, except that the vanilla
input embedding is replaced with a task-adaptive spectral embedding.
For the spectral embedder, we use:
\begin{itemize}[leftmargin=1.4em,itemsep=1pt,topsep=2pt]
\item Embedding depth: 3.
\item Number of spectral mixture components: 8.
\item Number of sampled frequencies per component: 8.
\item Number of frequency-grid points per dimension: 8.
\item Minimum period \(p_{\min}\): 0.08.
\item Maximum period \(p_{\max}\): 2.0.
\item Frequency-grid spacing: logarithmic.
\item Number of convolutional channels: 64.
\item Number of convolutional layers: 3.
\item Convolution kernel size: 3.
\item Convolutional dropout: 0.0.
\item Convolutional layer normalisation: enabled.
\end{itemize}
The output likelihood matches the Gaussian output parameterisation used in the
transformer baseline.

\subsection{Real-world time-series dataset details}
\label{app:forecast-datasets}
\label{app:realworld-timeseries-details}

\subsubsection{California traffic flow experiment details}

\paragraph{Data and training.}
For the California traffic-flow imputation benchmark in Section~4.3, we use
the 2020 portion of the LargeST dataset \cite{Liu2023LargeSTAB}, which contains
5-minute traffic-flow measurements from 8,600 loop-detector sensors deployed
across California highways between 2017 and 2021. We discard sensors with more
than 50\% missing entries. For the remaining sensors, missing segments between
observed values are filled by linear interpolation, while leading and trailing
gaps are completed by forward and backward filling, respectively. The sensors
are then randomly partitioned into training, validation, and test splits in a
6:1:3 ratio.

Each sensor time series is partitioned into 26 non-overlapping windows of 14
days, corresponding to 4,032 raw time steps per window. Within each window,
time is reset to start at zero with 5-minute increments. We then downsample by
a factor of 6 to obtain 30-minute resolution and rescale time to units of
days. Traffic-flow values are min--max normalised using statistics computed on
the training split. This preprocessing yields 133,094 training windows from
5,119 sensors, 22,178 validation windows from 853 sensors, and 66,586 test
windows from 2,561 sensors.

Models are trained for 100 epochs using AdamW with learning rate \(10^{-4}\)
and gradient clipping with maximum norm \(0.5\). Each epoch contains roughly
4,160 iterations with batch size 32 tasks. For each batch, we fix \(N=50\) and
sample the context size as \(M\sim\mathcal U[5,25)\), sharing \(M\) across all
tasks in the batch. A task is then formed by sampling \(N\) timestamps uniformly
without replacement from a window, using \(M\) of them as context points and the
remaining \(N-M\) as target points. Training tasks are generated on the fly from
the training windows. For validation and testing, we precompute one task per
window with the same fixed \(N=50\) and context-size sampling protocol. The
learnable parameter counts for the evaluated models are reported in
Table~\ref{tab:california-parameter-counts}.

\begin{table}[t]
\centering
\small
\caption{Learnable parameter counts for the California traffic-flow imputation
experiment. Values are reported in millions.}
\label{tab:california-parameter-counts}
\begin{tabular}{lccccccc}
\toprule
 & CNP & AttCNP & TNP & TETNP & ConvCNP & SConvCNP & STNP \\
\midrule
Parameters (M) & 1.3 & 2.1 & 2.6 & 3.1 & 3.8 & 3.7 & 2.8 \\
\bottomrule
\end{tabular}
\end{table}

\paragraph{Hyperparameters.}
Unless otherwise stated, we follow the hyperparameter settings used in the
SConvCNP benchmark for the California traffic-flow imputation task, including
the baseline model configurations and the shared non-spectral backbone used by
STNP. The only additional design choices are those of our spectral branch, for
which we use:
\begin{itemize}[leftmargin=1.4em,itemsep=1pt,topsep=2pt]
\item Embedding depth: 3.
\item Number of spectral mixture components: 6.
\item Number of sampled frequencies per component: 8.
\item Number of frequency-grid points per dimension: 128.
\item Minimum period \(p_{\min}\): 0.5.
\item Maximum period \(p_{\max}\): 14.0.
\item Frequency-grid spacing: logarithmic.
\item Number of convolutional channels: 64.
\item Number of convolutional layers: 3.
\item Convolution kernel size: 5.
\item Convolutional dropout: 0.0.
\item Convolutional layer normalisation: disabled.
\end{itemize}

\subsubsection{General benchmark experiment details}

\begin{table}[t]
\centering
\caption{Summarised feature details of the six forecasting datasets.}
\label{tab:forecast-dataset-summary}
\small
\begin{tabular}{l|ccc}
\toprule
\textsc{Dataset} & \textsc{Len} & \textsc{Dim} & \textsc{Freq} \\
\midrule
ETTh1 & 17420 & 8 & 1H \\
Electricity & 26304 & 322 & 1H \\
Exchange & 7588 & 9 & 1 day \\
Traffic & 17544 & 863 & 1H \\
Weather & 52696 & 22 & 10 min \\
ILI & 966 & 8 & 7 days \\
\bottomrule
\end{tabular}
\end{table}

\paragraph{Data and training.}
Section~4.3 evaluates long-horizon forecasting on six standard real-world
benchmarks. (1) \(\texttt{ETTh1}\) is the hourly subset of the ETT benchmark
\cite{Zhou2020InformerBE}, which records electricity transformer measurements,
including load and oil temperature, over the period from July 2016 to July
2018. (2) \(\texttt{Electricity}\)\footnote{\url{https://archive.ics.uci.edu/ml/datasets/ElectricityLoadDiagrams20112014}} contains the hourly electricity consumption
of 321 clients from 2012 to 2014. (3) \(\texttt{Exchange~Rate}\)
\cite{Lai2017ModelingLA} records the daily exchange rates of eight countries
ranging from 1990 to 2016. (4) \(\texttt{Traffic}\)\footnote{\url{http://pems.dot.ca.gov/}} is a collection of hourly
road occupancy measurements from the California Department of Transportation,
capturing sensor readings on San Francisco Bay Area freeways. (5)
\(\texttt{Weather}\)\footnote{\url{https://www.bgc-jena.mpg.de/wetter/}} is recorded every 10 minutes over the year 2020 and
contains 21 meteorological variables, including air temperature, humidity, and
related atmospheric indicators. (6) \(\texttt{National~Illness}\)\footnote{\url{https://gis.cdc.gov/grasp/fluview/fluportaldashboard.html}} (ILI)
contains weekly influenza-like illness statistics reported by the U.S. Centers
for Disease Control and Prevention between 2002 and 2021, including both the
ratio of patients with ILI and the total number of patients.

We follow the standard preprocessing and benchmark protocol used in prior
long-horizon forecasting work
\cite{Wu2021AutoformerDT,Zhang2023MultiresolutionTT,Zhou2022FEDformerFE}.
Table~\ref{tab:forecast-dataset-summary} summarises the sequence lengths, raw
data dimensions, and sampling frequencies of these datasets.
For reference, Table~\ref{tab:forecast-np-family-results} reports the full
NP-family results, where STNP is strongest among the NP-family methods, while
Table~\ref{tab:forecast-reference-results} lists representative standard
time-series forecasting models as reference points.
Entries are reported as MAE/MSE, and lower values indicate better performance.

\begin{table}[t]
\centering
\caption{NP-family results on the general benchmark. Entries are
MAE/MSE, and lower is better.}
\label{tab:forecast-np-family-results}
\scriptsize
\setlength{\tabcolsep}{3pt}
\renewcommand{\arraystretch}{1.05}
\resizebox{\textwidth}{!}{%
\begin{tabular}{lcccccc}
\toprule
Method & Electricity & ETTh1 & Exchange Rate & National Illness (ILI) & Traffic & Weather \\
\midrule
NP & 0.405 / 0.340 & 0.807 / 1.110 & 0.813 / 0.924 & 1.276 / 3.794 & 0.402 / 0.706 & 0.551 / 0.625 \\
CNP & 0.401 / 0.330 & 0.779 / 1.061 & 0.801 / 0.895 & 1.270 / 3.898 & 0.397 / 0.722 & 0.568 / 0.619 \\
ANP & 0.364 / 0.286 & 0.820 / 1.182 & 0.840 / 0.991 & 1.395 / 3.885 & 0.372 / 0.670 & 0.516 / 0.518 \\
CANP & 0.366 / 0.285 & 0.798 / 1.106 & 0.853 / 1.019 & 1.396 / 4.314 & 0.366 / 0.661 & 0.418 / 0.344 \\
BNP & 0.399 / 0.332 & 0.794 / 1.067 & 0.845 / 1.103 & 1.337 / 4.021 & 0.393 / 0.702 & 0.599 / 0.686 \\
BANP & 0.367 / 0.289 & 0.757 / 1.002 & 0.890 / 1.110 & 1.351 / 4.020 & 0.365 / 0.662 & 0.464 / 0.406 \\
ConvCNP & 0.360 / 0.287 & 0.837 / 1.139 & 0.835 / 1.083 & 1.272 / 3.955 & 0.370 / 0.690 & 0.510 / 0.509 \\
SConvCNP & 0.364 / 0.284 & 0.827 / 1.105 & 0.846 / 1.048 & 1.281 / 3.733 & 0.332 / 0.618 & 0.525 / 0.593 \\
TNP & 0.830 / 1.012 & 0.800 / 1.106 & 1.443 / 3.036 & 1.827 / 6.593 & 0.800 / 1.462 & 0.610 / 0.641 \\
TETNP & 0.367 / 0.283 & 0.639 / 0.723 & 0.783 / 0.898 & 1.528 / 5.308 & 0.328 / 0.627 & 0.257 / 0.202 \\
\textbf{STNP (ours)} & \textbf{0.281 / 0.179} & \textbf{0.487 / 0.539} & \textbf{0.208 / 0.087} & \textbf{1.040 / 3.080} & \textbf{0.318 / 0.582} & \textbf{0.227 / 0.185} \\
\bottomrule
\end{tabular}}
\end{table}

\begin{table}[t]
\centering
\caption{Reference results for standard time-series forecasting models on the
general benchmark. Entries are MAE/MSE, and lower is better.}
\label{tab:forecast-reference-results}
\scriptsize
\setlength{\tabcolsep}{3pt}
\renewcommand{\arraystretch}{1.05}
\resizebox{\textwidth}{!}{%
\begin{tabular}{lcccccc}
\toprule
Method & Electricity & ETTh1 & Exchange Rate & National Illness (ILI) & Traffic & Weather \\
\midrule
FEDformer & 0.308 / 0.193 & \textbf{0.419 / 0.376} & 0.278 / 0.148 & 1.260 / 3.228 & 0.366 / 0.587 & 0.296 / 0.217 \\
Autoformer & 0.317 / 0.201 & 0.459 / 0.449 & 0.323 / 0.197 & 1.287 / 3.483 & 0.388 / 0.613 & 0.336 / 0.266 \\
Informer & 0.368 / 0.274 & 0.713 / 0.865 & 0.752 / 0.847 & 1.677 / 5.764 & 0.391 / 0.719 & 0.384 / 0.300 \\
LogTrans~\cite{LI2019EnhancingTL} & 0.357 / 0.258 & 0.740 / 0.878 & 0.812 / 0.968 & 1.444 / 4.480 & 0.384 / 0.684 & 0.490 / 0.458 \\
Reformer~\cite{Kitaev2020ReformerTE} & 0.402 / 0.312 & 0.728 / 0.837 & 0.829 / 1.065 & 1.382 / 4.400 & 0.423 / 0.732 & 0.596 / 0.689 \\
\textbf{STNP (ours)} & \textbf{0.281 / 0.179} & 0.487 / 0.539 & \textbf{0.208 / 0.087} & \textbf{1.040 / 3.080} & \textbf{0.318 / 0.582} & \textbf{0.227 / 0.185} \\
\bottomrule
\end{tabular}}
\end{table}

\paragraph{Hyperparameters.}
Complete model- and dataset-specific hyperparameter configurations are
provided in the source code.

\paragraph{Relation to the canonical NP setting.}
Neural Processes are typically formulated as meta-learners over a distribution
of stochastic processes, where each task corresponds to a different function
realisation and prediction is performed via random context-target splits. In
contrast, single-series time-series forecasting provides only one observed
realisation of an underlying temporal process, and therefore does not fully
match the canonical Neural Process setting. Under a sliding-window
formulation, each look-back-prediction pair is more appropriately viewed as an
episode induced from the same underlying data-generating mechanism, rather
than as an independent sample from a different stochastic process. Moreover,
standard Neural Process evaluation is often closer to interpolation or
imputation, whereas look-back-then-predict forecasting requires predicting a
contiguous future block from past observations only. This makes the task
fundamentally more challenging, as it emphasises extrapolation rather than
set-based imputation. At the same time, it is particularly informative for
evaluating whether the models can adapt to temporally varying and potentially
non-stationary dynamics across forecasting episodes. A task that more naturally
matches the meta-learning interpretation is the climate data experiment in
Section~\ref{sec:climate-data}, where each June trajectory provides a separate
realisation and forecasting windows define context-target tasks within that
realisation.

\subsubsection{Climate data experiment details}
\label{app:chimet-details}

\paragraph{Data and training.}
We construct the Chimet\footnote{\url{https://www.chimet.co.uk/search.aspx}}
June benchmark from observations collected in June
between 2009 and 2025. We use five variables: wind speed (\textbf{WSPD}), air
temperature (\textbf{ATMP}), tidal depth (\textbf{DEPTH}), average wave height
(\textbf{AVWHT}), and barometric pressure (\textbf{BARO}). Each year's June
trajectory is treated as an independent realisation of the same seasonal
stochastic process,
\[
f_y : \mathcal{X}_{\mathrm{June}} \rightarrow \mathbb{R}^{5},
\qquad y \in \{2009,\ldots,2025\},
\]
where \(\mathcal{X}_{\mathrm{June}}\) denotes the shared within-June time
domain. Different years are therefore not concatenated into a single
continuous time series.

The raw measurements are recorded at approximately 5-minute resolution. We
first remove duplicated timestamps by averaging duplicated records. For a
timestamp \(t\) with duplicate observations \(\{z_t^{(m)}\}_{m=1}^{M_t}\), we
define
\[
\bar z_t = \frac{1}{M_t}\sum_{m=1}^{M_t} z_t^{(m)}.
\]
We then apply simple physical sanity checks to remove implausible values.
Invalid readings are treated as missing and are not used when constructing
hourly observations.

The cleaned 5-minute data are aggregated to hourly resolution. Let \(B_h\)
denote the set of 5-minute timestamps falling into hour \(h\). For variable
\(v\), the hourly value is computed as
\[
y_{h,v} =
\frac{1}{|B_{h,v}^{\mathrm{valid}}|}
\sum_{t \in B_{h,v}^{\mathrm{valid}}}
\bar z_{t,v},
\]
where \(B_{h,v}^{\mathrm{valid}}\subseteq B_h\) contains valid observations of
variable \(v\). An hourly observation is retained only if it contains
sufficiently many valid raw samples and all selected variables are observed:
\[
n_h = \min_{v} |B_{h,v}^{\mathrm{valid}}| \geq 10,
\qquad
y_{h,v} \neq \mathrm{NaN}
\quad \forall v \in \mathcal{V},
\]
where
\[
\mathcal{V} =
\{\mathrm{WSPD}, \mathrm{ATMP}, \mathrm{DEPTH}, \mathrm{AVWHT}, \mathrm{BARO}\}.
\]

We do not interpolate across missing periods. Instead, after hourly
aggregation, we identify maximal continuous valid hourly segments within each
June trajectory. Forecasting windows are generated only inside these
continuous segments. This prevents the model from being trained or evaluated
on artificial trajectories created by interpolating across sensor outages or
large observation gaps.

For model input, the time coordinate is defined as the hour within June:
\[
x_h = h,
\qquad
h = 0,1,\ldots,719,
\]
where \(h=0\) corresponds to June 1st 00:00. Thus, the same within-June hour
in different years shares the same input coordinate, while the corresponding
function values differ across years. This preserves the Neural Process
interpretation in which different years provide different function
realisations over a common input domain.

Each variable is standardised using statistics computed from the training
split only. For variable \(v\), we compute
\[
\mu_v =
\frac{1}{N_{\mathrm{train}}}
\sum_{(h,y)\in \mathcal{D}_{\mathrm{train}}}
y_{h,v},
\qquad
\sigma_v^2 =
\frac{1}{N_{\mathrm{train}}}
\sum_{(h,y)\in \mathcal{D}_{\mathrm{train}}}
(y_{h,v}-\mu_v)^2,
\]
and transform observations as
\[
\tilde y_{h,v} =
\frac{y_{h,v}-\mu_v}{\sigma_v}.
\]
All losses during training are computed in the standardised space, while
raw-scale metrics are obtained by transforming predictions back to physical
units:
\[
\hat y_{h,v} =
\sigma_v \hat{\tilde y}_{h,v} + \mu_v.
\]

The data are split chronologically by year. We use June trajectories from
2009--2020 for training, 2021--2022 for validation, and 2023--2025 for
testing. Years with severe missingness or missing key variables are excluded.
Forecasting windows are generated independently within each year and are not
allowed to cross missing gaps or year boundaries. Each task is defined by a
historical context window and a future target window:
\[
D_C =
\{(x_h,\tilde y_h): h \in [s, s+L)\},
\]
\[
D_T =
\{(x_h,\tilde y_h): h \in [s+L, s+L+T)\},
\]
where \(L=48\) hours is the look-back length and \(T=12\) hours is the forecast
horizon. Training, validation, and testing windows are generated with strides
of 6, 24, and 24 hours, respectively.

Forecasting models are trained with an MSE loss in the standardised space for
12 epochs, using 80 optimisation steps per epoch, with batch size 16 and
evaluation batch size 32. We use learning rate \(5\times 10^{-4}\), weight decay
\(10^{-4}\), and random seed 0. Model selection is based on validation
raw-scale RMSE. Table~\ref{tab:chimet-dataset-summary} summarises the resulting
benchmark protocol, and
Figures~\ref{fig:chimet-june-2025-five-variables-5min}
and~\ref{fig:chimet-june-2025-five-variables} show the June 2025 trajectory at
the raw 5-minute and resampled 1-hour resolutions, respectively.

\paragraph{Hyperparameters.}
For the other compared models, we follow the corresponding hyperparameter
settings used in the preceding real-world time-series experiments, while using
the shared Chimet data and training protocol described above.
Table~\ref{tab:chimet-parameter-counts} reports the learnable parameter counts
for the compared models, and Table~\ref{tab:chimet-ours-hyperparameters} lists
the Chimet-specific hyperparameters of our model.

\begin{table}[!htbp]
\centering
\caption{Summary of the Chimet climate forecasting dataset and training protocol.}
\label{tab:chimet-dataset-summary}
\scriptsize
\renewcommand{\arraystretch}{0.94}
\begin{tabular}{p{0.28\textwidth}p{0.62\textwidth}}
\toprule
Property & Value \\
\midrule
Location & Chichester Harbour, UK \\
Temporal coverage & June of each year from 2009 to 2025, excluding 2014 and
2019 \\
Raw resolution & Approximately 5 minutes \\
Raw observations & 144,474 \\
Resampled resolution & 1 hour \\
Variables & Wind speed (WSPD), air temperature (ATMP), tidal depth (DEPTH),
average wave height (AVWHT), barometric pressure (BARO) \\
Forecasting setup & \(L=48\) hours of context and \(T=12\) hours of prediction \\
Splits & Train: 2009--2020, excluding 2014 and 2019; validation: 2021--2022;
test: 2023--2025 \\
Window strides & Train: 6 hours; validation: 24 hours; test: 24 hours \\
Loss & MSE \\
Training epochs & 12 \\
Steps per epoch & 80 \\
Batch size & 16 \\
Evaluation batch size & 32 \\
Learning rate & \(5\times 10^{-4}\) \\
Weight decay & \(10^{-4}\) \\
Selection metric & Validation raw-scale RMSE \\
\bottomrule
\end{tabular}

\vspace{0.65em}

\caption{Learnable parameter counts for the Chimet climate forecasting
experiment. Values are reported in millions.}
\label{tab:chimet-parameter-counts}
\resizebox{\textwidth}{!}{%
\begin{tabular}{lccccccccccc}
\toprule
 & NP & CNP & ANP & CANP & BNP & BANP & ConvCNP & SConvCNP & TNP & TETNP & STNP \\
\midrule
Parameters (M) & 0.2 & 0.2 & 1.1 & 1.0 & 0.2 & 1.1 & 1.1 & 6.6 & 1.1 & 1.1 & 0.9 \\
\bottomrule
\end{tabular}}

\vspace{0.65em}

\caption{Chimet-specific hyperparameters of our model.}
\label{tab:chimet-ours-hyperparameters}
\scriptsize
\renewcommand{\arraystretch}{0.92}
\begin{tabular}{p{0.43\textwidth}p{0.43\textwidth}}
\toprule
Hyperparameter & Value \\
\midrule
\multicolumn{2}{l}{\textit{Spectral branch}} \\
Branch features & ATMP, DEPTH \\
Frequency-grid points & \([128,128]\) \\
Mixture components & \([2,2]\) \\
Samples per component & \([4,4]\) \\
Period range for ATMP & 18--30 hours \\
Period range for DEPTH & 10--16 hours \\
Frequency-grid spacing & Logarithmic \\
Phase shift & Disabled \\
Minimum standard deviation \(\sigma_{\min}\) & \(10^{-3}\) \\
\midrule
\multicolumn{2}{l}{\textit{Convolutional branch}} \\
Convolutional channels & 64 \\
Convolutional layers & 3 \\
Kernel size & 5 \\
Dilation growth & 1 \\
Convolutional dropout & 0.1 \\
Convolutional layer normalisation & Disabled \\
\midrule
\multicolumn{2}{l}{\textit{Transformer backbone}} \\
Input dimension \(d_x\) & 1 \\
Output dimension \(d_y\) & 5 \\
Model dimension \(d_{\mathrm{model}}\) & 128 \\
Embedding depth & 3 \\
Feed-forward dimension & 256 \\
Attention heads & 4 \\
Transformer layers & 6 \\
Dropout & 0.1 \\
Bounded output standard deviation & Enabled \\
Layer normalisation & Enabled \\
\bottomrule
\end{tabular}
\end{table}

\begin{figure}[!p]
\centering
\includegraphics[width=0.88\textwidth]{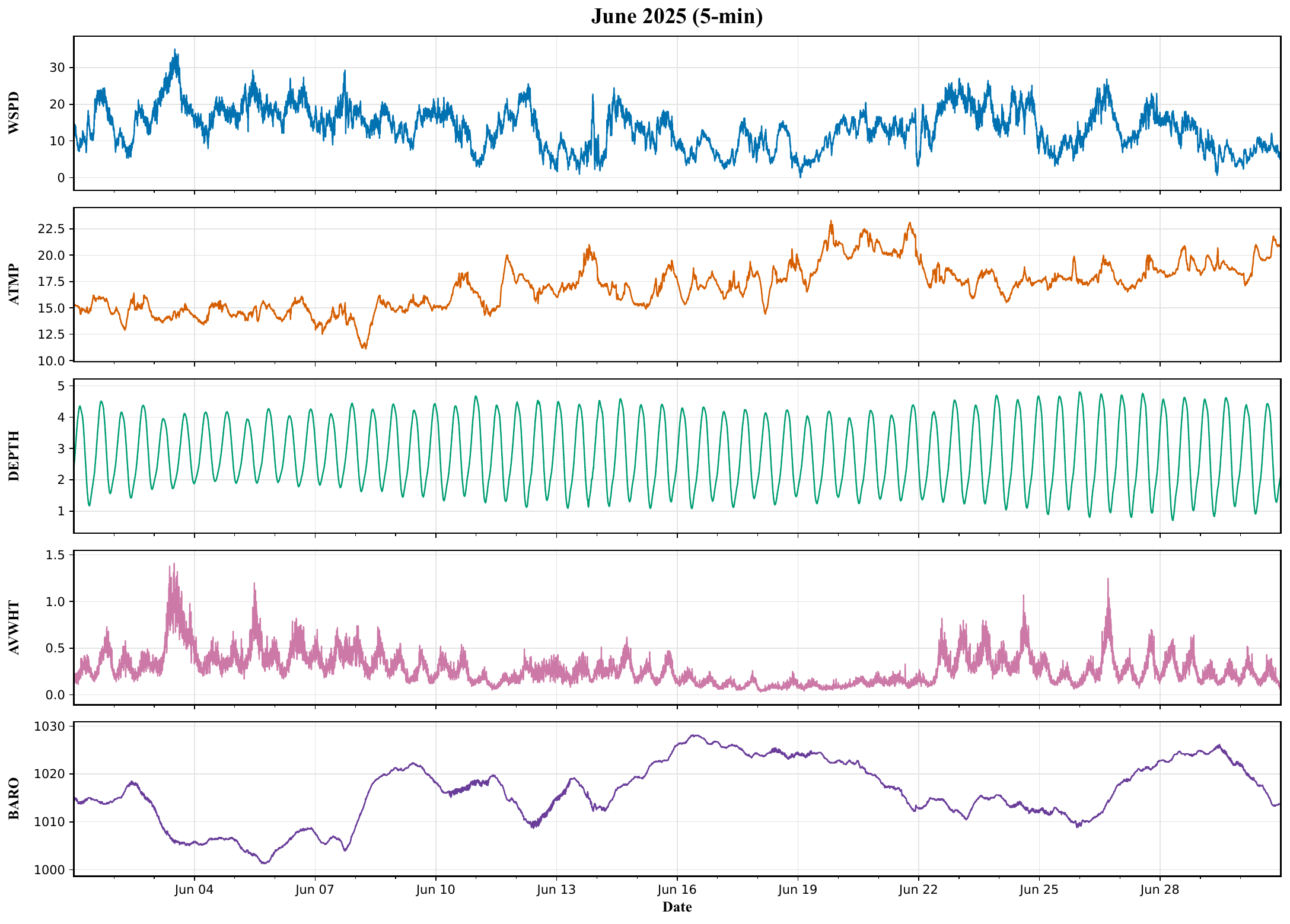}
\vspace{-0.4em}
\caption{Chimet climate observations for June 2025 at the raw 5-minute
resolution. The five panels show air temperature, wind speed, tidal depth,
average wave height, and barometric pressure.}
\label{fig:chimet-june-2025-five-variables-5min}
\vspace{0.4em}
\includegraphics[width=0.88\textwidth]{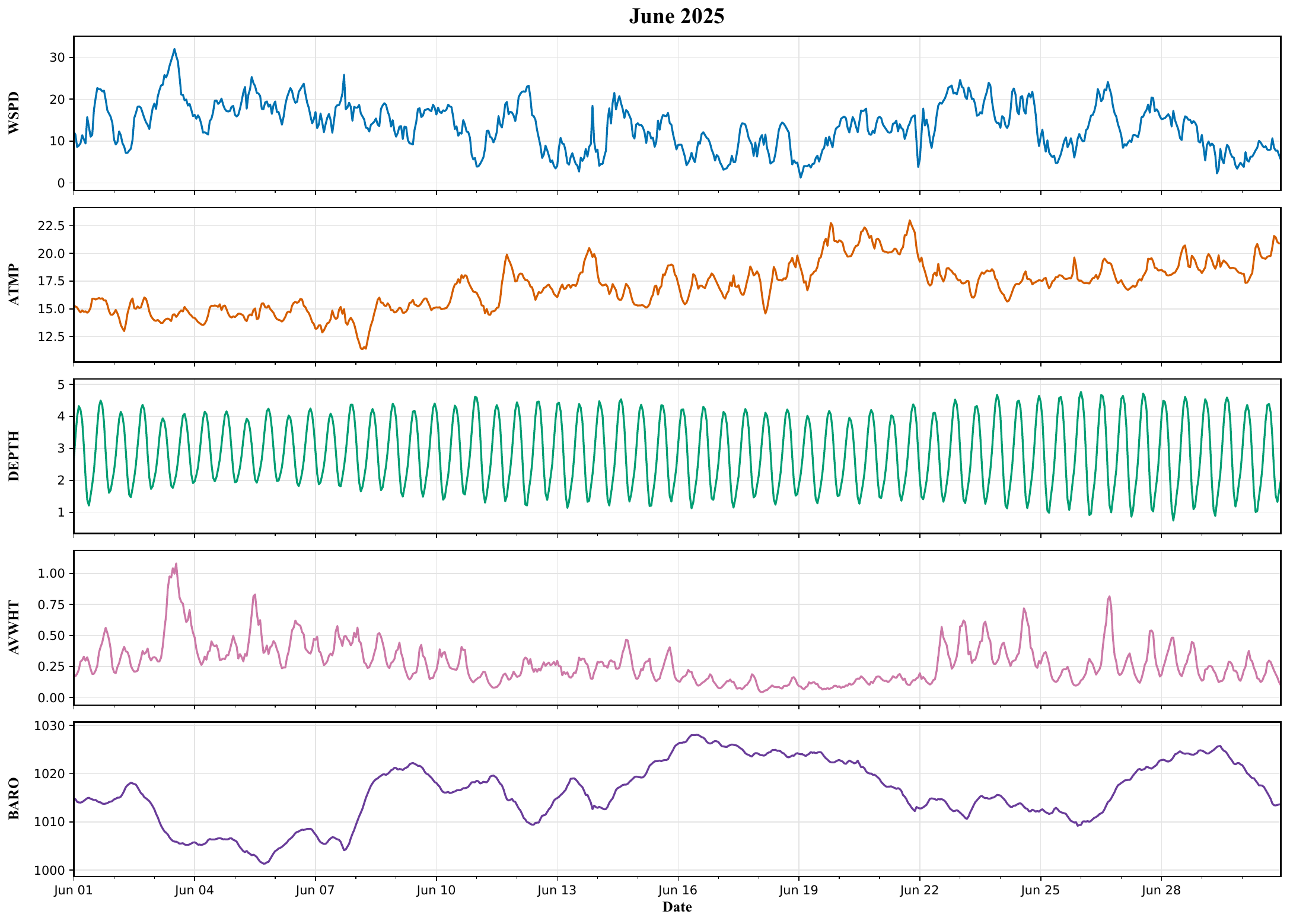}
\vspace{-0.4em}
\caption{Chimet climate observations for June 2025 after resampling to a
1-hour resolution. The five panels show air temperature, wind speed, tidal
depth, average wave height, and barometric pressure.}
\label{fig:chimet-june-2025-five-variables}
\end{figure}

\clearpage
\subsection{Pseudocode}

In Algorithms~1 and~2, the normalised frequency coordinate is used only as an
input feature for the responsibility network; the sampled frequencies used in
the final sinusoidal features remain in the original frequency scale. For
multi-output settings with selected channels, Algorithm~2 is applied either to
the shared selected-channel summary or independently to each selected channel.

\begin{algorithm}[H]
\caption{One-dimensional spectral embedding}
\footnotesize
\begin{algorithmic}[1]
\Require Context \(\mathcal{D}_C=\{(x_i,y_i)\}_{i=1}^{M}\), token inputs \(\{x_n\}_{n=1}^{N}\), frequency grid \(\{\tilde{\omega}_k\}_{k=1}^{K}\), components \(Q\), samples \(D_0\)
\Ensure Mixture parameters \(\theta_{\mathcal{D}_C}=\{(w_q,\mu_q,\sigma_q^2,\phi_q)\}_{q=1}^{Q}\) and features \(\{\phi_S(x_n)\}_{n=1}^{N}\)
\State Set \(\omega_{\max}\gets \max_{1\le j\le K}|\tilde{\omega}_j|\)
\For{$k=1$ to $K$}
    \State Compute \(a_k,b_k\) and \(E_k=a_k^2+b_k^2+\varepsilon\)
    \State Set \(\mathbf S_k=[\log E_k,a_k/\sqrt{E_k},b_k/\sqrt{E_k},\tilde{\omega}_k/\omega_{\max}]\) and \(p_k \propto E_k\)
\EndFor
\State Normalise \(p_k=\mathrm{softmax}_k(\log E_k)\)
\State Compute responsibility logits \(\ell_{qk}\) from \(\{\mathbf S_k\}_{k=1}^{K}\)
\State Set \(r_{qk}=\mathrm{softmax}_q(\ell_{qk})\)
\For{$q=1$ to $Q$}
    \State Compute \(w_q,\mu_q,\sigma_q^2\) from \(\{r_{qk},p_k,\tilde{\omega}_k\}_{k=1}^{K}\)
    \State Compute phase \(\phi_q\) from \(z_q=\sum_k \frac{r_{qk}p_k}{w_q}u_k\), where \(u_k=(a_k+\mathrm{i}b_k)/\sqrt{E_k}\)
    \For{$d=1$ to $D_0$}
        \State Sample \(\omega_{q,d} \sim \mathcal{N}(\mu_q,\sigma_q^2)\)
    \EndFor
\EndFor
\For{$n=1$ to $N$}
    \State Evaluate \(\phi_S(x_n)\) using \(\{w_q,\omega_{q,d},\phi_q\}_{q,d}\)
\EndFor
\State \Return \(\theta_{\mathcal{D}_C}\) and \(\{\phi_S(x_n)\}_{n=1}^{N}\)
\end{algorithmic}
\end{algorithm}

\begin{algorithm}[H]
\caption{Multi-dimensional spectral embedding}
\footnotesize
\begin{algorithmic}[1]
\Require Context \(\mathcal{D}_C=\{(\mathbf x_i,y_i)\}_{i=1}^{M}\), token inputs \(\{\mathbf x_n\}_{n=1}^{N}\), selected input coordinates \(\mathbf x_i^{\mathcal P}\in\mathbb R^{d_p}\), frequency grid \(\{\tilde{\boldsymbol\omega}_k\}_{k=1}^{K}\subset\mathbb R^{d_p}\), components \(Q\), samples \(D_0\)
\Ensure Mixture parameters \(\theta_{\mathcal{D}_C}=\{(w_q,\boldsymbol\mu_q,\mathbf\Sigma_q,\phi_q)\}_{q=1}^{Q}\) and features \(\{\phi_S(\mathbf x_n)\}_{n=1}^{N}\)
\State Set \(\omega_{\max}\gets \max_{1\le j\le K}\|\tilde{\boldsymbol\omega}_j\|_2\)
\State Centre the responses: \(y_i^{(s)} \gets y_i-\bar y\)
\For{$k=1$ to $K$}
    \State Compute \(a_k,b_k\) from projections onto \(\cos(\tilde{\boldsymbol\omega}_k^\top \mathbf x_i^{\mathcal P})\) and \(\sin(\tilde{\boldsymbol\omega}_k^\top \mathbf x_i^{\mathcal P})\)
    \State Set \(E_k \gets a_k^2+b_k^2+\varepsilon\), \(p_k \propto E_k\), and \(\mathbf S_k=[\log E_k,a_k/\sqrt{E_k},b_k/\sqrt{E_k},\tilde{\boldsymbol\omega}_k/\omega_{\max}]\)
\EndFor
\State Normalise \(p_k=\mathrm{softmax}_k(\log E_k)\)
\State Compute responsibility logits \(\ell_{qk}\) from \(\{\mathbf S_k\}_{k=1}^{K}\)
\State Set \(r_{qk}=\mathrm{softmax}_q(\ell_{qk})\)
\For{$q=1$ to $Q$}
    \State Compute
    \[
    w_q \gets \sum_{k=1}^{K} r_{qk}p_k,
    \qquad
    \boldsymbol\mu_q \gets \frac{\sum_{k=1}^{K} r_{qk}p_k\,\tilde{\boldsymbol\omega}_k}{w_q}
    \]
    \State Compute covariance \(\mathbf\Sigma_q\) from \(\{r_{qk},p_k,\tilde{\boldsymbol\omega}_k\}_{k=1}^{K}\)
    \State Compute phase \(\phi_q\) from \(z_q=\sum_k \frac{r_{qk}p_k}{w_q}u_k\), where \(u_k=(a_k+\mathrm{i}b_k)/\sqrt{E_k}\)
    \For{$d=1$ to $D_0$}
        \State Sample \(\boldsymbol\omega_{q,d} \sim \mathcal N(\boldsymbol\mu_q,\mathbf\Sigma_q)\)
    \EndFor
\EndFor
\For{$n=1$ to $N$}
    \State Evaluate \(\phi_S(\mathbf x_n)\) using \(\{w_q,\boldsymbol\omega_{q,d},\phi_q\}_{q,d}\) and \(\mathbf x_n^{\mathcal P}\)
\EndFor
\State \Return \(\theta_{\mathcal{D}_C}\) and \(\{\phi_S(\mathbf x_n)\}_{n=1}^{N}\)
\end{algorithmic}
\end{algorithm}

\subsection{Extended Related Work}
\label{app:extended-related-work}

\subsubsection{Neural Processes}
\label{app:extended-related-work-neural-processes}

\paragraph{Translation-equivariant Transformer Neural Processes.}
Translation-equivariant Transformer Neural Processes (TETNPs)
\cite{Ashman2024TranslationET} extend the TNP family by replacing standard
multi-head self-attention and cross-attention blocks with translation-equivariant
attention blocks. A standard TNP embeds context and target tokens and processes
them with a masked Transformer, where target tokens attend to context tokens but
not to other targets. TETNP keeps this encoder--decoder structure, but makes
attention scores depend on relative input offsets rather than absolute input
locations. For tokens associated with locations \(x_i\) and \(x_j\), TETNP
augments attention with pairwise differences
\[
\Delta x_{ij}=x_i-x_j .
\]
Thus, attention weights are functions of token features and relative offsets.
Shifting all inputs by a common \(\Delta\) leaves all \(\Delta x_{ij}\)
unchanged and shifts the predictive process consistently. For each target
\(i>M\), the final target representation is decoded as in TNP,
\[
(\mu_i,\sigma_i)=g_{\mathrm{pred}}(e_i^{(L)}),\qquad
p_\theta(y_i\mid x_i,\mathcal{D}_C)=
\mathcal N(\mu_i,\operatorname{diag}(\sigma_i^2)).
\]
This design gives a strong inductive bias for stationary or translationally
structured data, but it does not explicitly model frequency-domain structure or
task-specific periodic spectra.

\paragraph{Convolutional Conditional Neural Processes.}
Convolutional Conditional Neural Processes (ConvCNPs)
\cite{Gordon2019ConvolutionalCN} introduce translation equivariance into CNPs by
replacing finite-dimensional set aggregation~\cite{Zaheer2017DeepS} with a functional representation of
the context set. Given
\(\mathcal{D}_C=\{(x_i,y_i)\}_{i=1}^{M}\), a typical scalar-output construction
embeds the context into a function \(h\) over the input domain:
\[
h^{(0)}(z)=\sum_{i=1}^{M}\psi(z-x_i),
\qquad
h^{(1)}(z)=
\frac{\sum_{i=1}^{M}y_i\psi(z-x_i)}
{\sum_{i=1}^{M}\psi(z-x_i)+\epsilon},
\]
where \(\psi\) is usually a positive kernel and \(h^{(0)}\) is a density channel
indicating where context observations are available. The representation
\(h(z)=[h^{(0)}(z),h^{(1)}(z)]\) is evaluated on a regular grid, processed by a
CNN or U-Net, interpolated back to target locations, and decoded into
\((\mu_i,\sigma_i)\). Because both context embedding and interpolation depend on
relative offsets, and the intermediate representation is processed by
convolutional layers, ConvCNP is translation equivariant. Its main limitation is
that the convolutional backbone relies on local spatial kernels, so capturing
long-range or highly oscillatory structure may require large receptive fields or
deep CNNs.

\paragraph{Spectral Convolutional Conditional Neural Processes.}
Spectral Convolutional Conditional Neural Processes (SConvCNPs)
\cite{Mohseni2024SpectralCC} keep the ConvCNP pipeline of functional context
embedding, grid evaluation, and target interpolation, but replace the spatial
CNN/U-Net backbone with a Fourier-neural-operator-style spectral convolution
module. Let \(u^{(\ell)}\) denote the grid representation at layer \(\ell\). A
spectral convolution layer applies a Fourier transform on the grid, multiplies
the retained modes by learned complex-valued weights, and maps the result back
to the spatial domain:
\[
u^{(\ell+1)}
=
\eta\!\left(
W_\ell u^{(\ell)}
+
\mathcal F^{-1}\!\left(R_\ell \odot \mathcal F u^{(\ell)}\right)
\right),
\]
where \(R_\ell\) parameterises a finite number of retained Fourier modes and
\(\eta\) is a pointwise nonlinearity. This gives SConvCNP a global
convolutional receptive field while preserving the ConvCNP-style functional
representation and target interpolation. Compared with ConvCNP, SConvCNP is
better suited to global or long-range spatial interactions, but its spectral
filters are learned model parameters rather than task-adaptive spectral mixtures
inferred from \(\mathcal{D}_C\). Thus, unlike STNP, it does not explicitly
amortise a context-specific spectral mixture kernel.

\paragraph{Biased Scan Attention Transformer Neural Processes.}
Biased Scan Attention Transformer Neural Processes (BSA-TNPs)
\cite{Jenson2025ScalableSI} are TNP-style architectures designed for scalable
spatiotemporal neural processes. Given an episode
\((\mathcal{D}_C,\mathcal D_T)\), BSA-TNP constructs tokens using an observation
indicator \(o_i=\mathbf 1[i\leq M]\), the input \(x_i\), and the value
\(\tilde y_i=y_i\) for context points and \(\tilde y_i=0\) for target queries.
For spatiotemporal data, the input may be decomposed as
\(x_i=(u_i,s_i,t_i)\), where \(u_i\) denotes fixed covariates, \(s_i\) a spatial
location, and \(t_i\) a temporal index. The main computational unit is a kernel
regression block, where context tokens attend to context tokens and target
tokens attend only to context tokens. In attention head \(h\), the attention
logit can be written as
\[
\ell_{ij}^{(h)}
=
\frac{(q_i^{(h)})^\top k_j^{(h)}}{\sqrt{d_h}}
+
B_{ij}^{(h)},
\]
where \(B_{ij}^{(h)}\) is a group-invariant bias. For spatial or temporal
translation invariance, BSA-TNP uses RBF-style biases such as
\[
b(s_i,s_j)
=
\sum_{m=1}^{M_b}
\alpha_m \exp\!\left(-\beta_m \|s_i-s_j\|^2\right),
\]
with analogous terms for time. Since these biases depend only on relative
distances, shifting all spatial or temporal inputs leaves the bias unchanged.
This gives BSA-TNP a strong translation-invariant Gaussian-kernel attention
prior, but the bias is fixed by learned RBF components rather than inferred as a
task-adaptive spectral mixture from \(\mathcal{D}_C\).

\paragraph{Distance-informed Neural Processes.}
Distance-informed Neural Processes (DINPs) \cite{Venkataramanan2025DistanceinformedNP}
augment latent NPs with both a global latent path and a distance-aware local
latent path. The global path follows the usual latent NP formulation: the
context set is encoded into a mean-pooled summary \(s_C\), which parameterises a
Gaussian prior
\[
p_{\theta_G}(z_G\mid x_C,y_C)
=
\mathcal N\!\left(\mu_{\theta_G}(s_C),\Sigma_{\theta_G}(s_C)\right),
\]
where \(z_G\) captures task-level uncertainty. The local path introduces a
target-specific latent variable \(z_i\) for each target input \(x_i\). It first
maps inputs into a latent metric space using \(u_i=h(x_i)\), and regularises
\(h\) to be approximately bi-Lipschitz:
\[
L_1 d_X(x_i,x_j)
\leq
d_U(h(x_i),h(x_j))
\leq
L_2 d_X(x_i,x_j).
\]
The resulting distance-preserving embeddings define target-to-context attention
weights, for \(i>M\) and \(j\leq M\),
\[
\alpha_{ij}
=
\frac{\exp(-\|u_i-u_j\|/\sqrt{d_u})}
{\sum_{j'\leq M}\exp(-\|u_i-u_{j'}\|/\sqrt{d_u})},
\]
which parameterise a local latent prior through attention-weighted context
statistics. The decoder predicts from both latent variables,
\[
p_\theta(y_i\mid x_i,\mathcal{D}_C)
=
\iint
p_\theta(y_i\mid x_i,z_i,z_G)
p_{\theta_L}(z_i\mid x_i,\mathcal{D}_C)
p_{\theta_G}(z_G\mid \mathcal{D}_C)
\,dz_i\,dz_G .
\]
DINP mainly improves local similarity modelling and uncertainty calibration by
preserving input-space distances in the latent space. In contrast to STNP, it
does not explicitly extract frequency-domain structure from \(\mathcal{D}_C\) or
induce a task-adaptive spectral-mixture kernel.

\subsubsection{Transformers Implement In-Context Learning}
\label{app:extended-related-work-icl}

Large language models first popularised the phenomenon of in-context learning,
where a model adapts to a new task from demonstrations in the input sequence
without updating its parameters \cite{Brown2020LanguageMA}. In the
function-learning view, an in-context prompt can be written analogously to the
episodic notation used in Neural Processes: a context set
\(\mathcal{D}_C=\{(x_j,y_j)\}_{j=1}^{M}\) is followed by a query input \(x_i\),
and the Transformer produces a prediction for \(y_i\), corresponding to a
conditional predictor \(p_\theta(y_i\mid x_i,\mathcal{D}_C)\). This setting has
been studied in controlled function classes \cite{Garg2022WhatCT}, showing that
Transformers trained from scratch can in-context learn unseen linear functions,
sparse linear functions, two-layer neural networks, and decision trees. Linear
regression studies further show that Transformers can implement standard
learning algorithms, including gradient descent and closed-form ridge regression,
while trained in-context learners can resemble least-squares, ridge, or Bayesian
estimators depending on model and data regimes \cite{Akyrek2022WhatLA}.

A complementary line of work interprets self-attention as an implicit
optimisation procedure. A linear self-attention layer can implement a
transformation equivalent to one step of gradient descent on a regression loss,
linking autoregressive Transformer training to gradient-based meta-learning
\cite{Oswald2022TransformersLI}. GPT-style in-context learning has also been
interpreted as implicit fine-tuning or meta-optimisation, where attention
produces meta-gradients from demonstration examples \cite{Dai2022WhyCG}. Beyond
expressivity, subsequent analyses ask whether Transformers can learn such
algorithms from random problem instances, proving that the global minimum of a
one-layer linear Transformer implements one step of preconditioned gradient
descent and that multi-layer critical points can implement multiple such steps
\cite{Ahn2023TransformersLT}. This view extends to nonlinear function learning,
where nonlinear Transformers can implement functional gradient descent in context
\cite{Cheng2023TransformersIF}.

The algorithmic perspective has also been broadened to show that Transformers
can implement a range of statistical learning procedures in context, including
least squares, ridge regression, Lasso, generalised linear model learning, and
gradient descent for neural networks \cite{Bai2023TransformersAS}. More
importantly, a single Transformer can perform in-context algorithm selection,
adaptively choosing different procedures for different input sequences. These
results are closely related to Transformer Neural Processes, which also treat
prediction as sequence-conditioned inference from \(\mathcal{D}_C\) to
\(\mathcal D_T\). However, most theoretical work on in-context learning focuses
on generic regression, algorithm simulation, or statistical estimation, rather
than on designing structural inductive biases for specific data regimes.

\clearpage
\input{checklist}

\end{document}

%% file: checklist.tex
\section*{NeurIPS Paper Checklist}

\begin{enumerate}

\item {\bf Claims}
    \item[] Question: Do the main claims made in the abstract and introduction accurately reflect the paper's contributions and scope?
    \item[] Answer: \answerYes{} 
    \item[] Justification: The abstract and introduction clearly state that we introduce Spectral
Transformer Neural Processes (STNPs), explain the motivation for using spectral
methods to address limitations in TNPs, and accurately describe the scope of
our experimental validation. The claims match the method, experiments, and
ablation results presented in the paper.

\item {\bf Limitations}
    \item[] Question: Does the paper discuss the limitations of the work performed by the authors?
    \item[] Answer: \answerYes{} 
    \item[] Justification: The paper discusses the limitations of the conditionally independent
Gaussian output head, including residual dependencies and non-Gaussian
predictive uncertainty, and outlines extensions using correlated or
latent-variable predictive models and more expressive likelihoods.

\item {\bf Theory assumptions and proofs}
    \item[] Question: For each theoretical result, does the paper provide the full set of assumptions and a complete (and correct) proof?
    \item[] Answer: \answerYes{} 
    \item[] Justification: We state and prove the main properties of the spectral features and STNPs
in Appendix~A, including their equivariance, induced kernel structure, and
conditional stochastic-process validity.

    \item {\bf Experimental result reproducibility}
    \item[] Question: Does the paper fully disclose all the information needed to reproduce the main experimental results of the paper to the extent that it affects the main claims and/or conclusions of the paper (regardless of whether the code and data are provided or not)?
    \item[] Answer: \answerYes{} 
    \item[] Justification: The paper provides detailed descriptions of the model architecture, data
generation process, and experimental setup, ensuring that the main results can
be reproduced. In addition, all implementations are available in the uploaded
code zip file.

\item {\bf Open access to data and code}
    \item[] Question: Does the paper provide open access to the data and code, with sufficient instructions to faithfully reproduce the main experimental results, as described in supplemental material?
    \item[] Answer: \answerYes{} 
    \item[] Justification: All implementations are available in the supplemental material.

\item {\bf Experimental setting/details}
    \item[] Question: Does the paper specify all the training and test details (e.g., data splits, hyperparameters, how they were chosen, type of optimizer) necessary to understand the results?
    \item[] Answer: \answerYes{} 
    \item[] Justification: The paper provides detailed descriptions of the model architecture, data
generation process, and experimental setup.

\item {\bf Experiment statistical significance}
    \item[] Question: Does the paper report error bars suitably and correctly defined or other appropriate information about the statistical significance of the experiments?
    \item[] Answer: \answerYes{} 
    \item[] Justification: All experiments are run five times.

\item {\bf Experiments compute resources}
    \item[] Question: For each experiment, does the paper provide sufficient information on the computer resources (type of compute workers, memory, time of execution) needed to reproduce the experiments?
    \item[] Answer: \answerYes{} 
    \item[] Justification: The main image-completion experiments used eight NVIDIA A100 80GB GPUs.
Other experiments, including runtime and peak-memory profiling for both synthetic GP regression
and DTD image completion, used one NVIDA 2080 Ti GPU.

\item {\bf Code of ethics}
    \item[] Question: Does the research conducted in the paper conform, in every respect, with the NeurIPS Code of Ethics \url{https://neurips.cc/public/EthicsGuidelines}?
    \item[] Answer: \answerYes{} 
    \item[] Justification: The data used in this paper come from synthetic data and public datasets.
The research focuses on methodological advancements in neural network
architectures for probabilistic meta-learning.

\item {\bf Broader impacts}
    \item[] Question: Does the paper discuss both potential positive societal impacts and negative societal impacts of the work performed?
    \item[] Answer: \answerNA{} 
    \item[] Justification: The paper focuses on neural processes and includes no discussion of broader
societal impacts.

\item {\bf Safeguards}
    \item[] Question: Does the paper describe safeguards that have been put in place for responsible release of data or models that have a high risk for misuse (e.g., pre-trained language models, image generators, or scraped datasets)?
    \item[] Answer: \answerNA{} 
    \item[] Justification: Our work focuses on meta-learning for regression and forecasting tasks.
The models and data used do not pose significant risks for misuse.

\item {\bf Licenses for existing assets}
    \item[] Question: Are the creators or original owners of assets (e.g., code, data, models), used in the paper, properly credited and are the license and terms of use explicitly mentioned and properly respected?
    \item[] Answer: \answerYes{} 
    \item[] Justification: We have been meticulous in citing all prior work and assets used in our
experiments. For each external codebase, dataset, or model, we provide proper
attribution and ensure that the associated licences and terms of use are
respected.

\item {\bf New assets}
    \item[] Question: Are new assets introduced in the paper well documented and is the documentation provided alongside the assets?
    \item[] Answer: \answerYes{} 
    \item[] Justification: We provide the relevant code and documentation alongside the paper.

\item {\bf Crowdsourcing and research with human subjects}
    \item[] Question: For crowdsourcing experiments and research with human subjects, does the paper include the full text of instructions given to participants and screenshots, if applicable, as well as details about compensation (if any)? 
    \item[] Answer: \answerNA{} 
    \item[] Justification: Our paper does not involve crowdsourcing or research with human subjects.

\pagebreak
\item {\bf Institutional review board (IRB) approvals or equivalent for research with human subjects}
    \item[] Question: Does the paper describe potential risks incurred by study participants, whether such risks were disclosed to the subjects, and whether Institutional Review Board (IRB) approvals (or an equivalent approval/review based on the requirements of your country or institution) were obtained?
    \item[] Answer: \answerNA{} 
    \item[] Justification: The paper does not involve crowdsourcing or research with human subjects.

\item {\bf Declaration of LLM usage}
    \item[] Question: Does the paper describe the usage of LLMs if it is an important, original, or non-standard component of the core methods in this research? Note that if the LLM is used only for writing, editing, or formatting purposes and does \emph{not} impact the core methodology, scientific rigor, or originality of the research, declaration is not required.
    \item[] Answer: \answerNA{} 
    \item[] Justification: The core method development in this research does not use LLMs as important,
original, or non-standard components.

\end{enumerate}